\documentclass{article}

\usepackage[utf8]{inputenc}
\usepackage[T1]{fontenc}
\usepackage{amsmath}
\usepackage{amsfonts}
\usepackage{amssymb}
\usepackage{amsthm}
\usepackage{graphicx}
\usepackage[pdftex]{hyperref}
\usepackage{fullpage}
\usepackage{color}
\usepackage{comment}
\usepackage[sort&compress,numbers]{natbib}
\usepackage{todonotes}
\usepackage{booktabs}
\usepackage{algorithm}
\usepackage{algpseudocode}
\usepackage{enumitem}
\usepackage{tikz}

\usepackage{tabularx}     
\usepackage{caption}      
\usepackage{subcaption}   
\usepackage{etoolbox}     

\theoremstyle{definition}

\newtheorem{corollary}{Corollary}[section]

\newtheorem{assumption}{Assumption}[section]
\newtheorem{proposition}{Proposition}[section]
\newtheorem{remark}{Remark}[section]

\title{Anti-Collapse Dynamics and the Emergence of Multi-Time-Scale Learning in Recurrent Neural Networks}

\author{
Lorenzo Livi\thanks{
OPIT -- Open Institute of Technology, Msida, MSD 9021, Malta.
\href{mailto:lorenz.livi@gmail.com}{lorenz.livi@gmail.com}.
ORCID: \href{https://orcid.org/0000-0001-6384-4743}{0000-0001-6384-4743}.
\href{https://scholar.google.com/citations?user=hAL9amAAAAAJ&hl=en}{Google Scholar profile}.
}}
\date{\today}

\begin{document}
\maketitle

\begin{abstract}
Long-range learning is hard for recurrent networks trained with stochastic gradient descent, because the influence of a past input fades with the lag $\ell$, and if it fades too fast the dependence cannot be learned from finite data.
This fade is captured by an envelope $f(\ell)$.
An exponential fade makes the data needed to learn a lag-$\ell$ dependence grow exponentially, putting long horizons out of reach; a power-law fade keeps the cost polynomial.
We show that the asymptotic decay behavior of $f(\ell)$ is not fixed by the architecture.
Instead, it emerges from the coupling between the state dynamics and parameter dynamics, settling into either a collapsed regime (fast, exponential forgetting) or an extended, anti-collapsed regime (slow, power-law forgetting).
The intuition is a competition within these coupled dynamics.
The systematic part of the training updates pulls long effective time scales back toward the network's typical range, while rare, heavy-tailed fluctuations can push a few far beyond it.
When these pushes are strong enough to balance the pull, the extended regime survives.
We make this mathematically precise with a coarse-grained stochastic process and derive the explicit threshold above which this balance holds.
A single exponent, the spectral exponent~$\beta$, then governs both the spread of time scales and how slowly the network forgets.
Realizing the regime in practice needs one more ingredient: the joint action of the architecture and the optimizer must be able to hold such a broad spread.
A network whose capacity to generate broad time-scale spectra is severely constrained still collapses, even when supplied with strong heavy-tailed forcing.
Heavy-tailed fluctuations thus act not as noise to be suppressed, but as the mechanism that sustains long-range learning.
\end{abstract}

\tableofcontents

\section{Introduction}

Temporal learning in recurrent neural networks (RNNs) is shaped by how gradient
information decays across time steps, and this decay geometry sets how far back a
network can effectively learn. When the decay is exponential, distant
dependencies quickly fall out of statistical reach; when it follows a power-law,
the same dependencies remain accessible across far longer horizons. Gated RNNs
trained with stochastic gradient descent (SGD) can land in either regime: some
converge to a narrow temporal organization, while others develop a heterogeneous
mixture of slow and fast time scales that sustains long-range learning.
These outcomes are not rigidly determined by architectural and optimization
design but emerge dynamically during training. What mechanism governs which
regime a trained network occupies?

The effective learning rates introduced by \citet{livi2025timescale} provide a natural framework for this question.
These mesoscopic variables capture the strength of the coupling between state dynamics and parameter updates: for each neuron $q$, the quantity $\mu^{(q)}_{t,\ell}$ measures how much gradient information originating at time $t$ survives $\ell$ recurrent transitions.
Under general SGD-like optimizers, this coupling factorizes as $\mu^{(q)}_{t,\ell} = \Lambda^{(q)}_{r,\ell}\,\Gamma^{(q)}_{t,\ell}$, separating an optimizer-dependent amplitude~$\Lambda$ from a transport factor~$\Gamma$ determined by the recurrent dynamics.
At the network level, temporal learning is governed by the macroscopic envelope $f(\ell)$, whose decay geometry determines the effective temporal reach of learning.
The statistical cost of this reach depends critically on envelope shape\footnote{%
We adopt \emph{power-law} as the canonical term for envelope
decay of the form $f(\ell)\sim \ell^{-\beta}$ and the
corresponding tail of the time-scale distribution
$p_\infty(\tau)\sim \tau^{-1-\beta}$. The same relation
is commonly called \emph{polynomial} when referring to the
asymptotic scaling regime or the growth of derived quantities, and is also known as
\emph{algebraic decay} in parts of the literature.
In this paper, however, we reserve \emph{power-law} for the decay form and
use \emph{polynomial} for the associated scaling class.
This convention follows the companion
paper~\citep{livi2026learnability}.
For consistency with the canonical scaling forms used in this
paper, we write $\sim$ for tail and envelope decay laws in the
loose ``scales as'' sense, with regime-of-validity statements
made in surrounding prose where they matter. The companion
paper~\citep{livi2026learnability} adopts the notation
$\asymp$, which explicitly absorbs multiplicative constants in
two-sided sandwich bounds. We use $\sim$ here to keep the
notation light and uniform across the envelope classes and
tail models without losing precision in places where
multiplicative constants matter operationally.}.
As shown by \citet{livi2026learnability}, an exponential envelope $f(\ell)\sim a^\ell$ with $0<a<1$ requires $N(\ell)\sim a^{-\kappa_\alpha\ell}$ samples
(with $\kappa_\alpha=\alpha/(\alpha-1)$ the concentration exponent and $1<\alpha\leq2$ the stability index of the underlying symmetric $\alpha$-stable model) to detect a
dependency at lag~$\ell$, while a power-law envelope
$f(\ell)\sim\ell^{-\beta}$ requires only
$N(\ell)\sim\ell^{\kappa_\alpha\beta}$ samples.

This exponential-versus-polynomial gap places functional pressure on the learning dynamics.
Whether the envelope decays exponentially or as a power law depends on the dynamical regime the architecture--optimizer pair sustains during the later stages of training.
We call this regime \emph{collapsed} when training concentrates the effective time scales on a narrow spectrum\footnote{We use \emph{spectrum} and \emph{density} interchangeably throughout: the spectrum of effective time scales is the time-scale density, and equivalently for the density of the associated log-effective decay rates, the two being related by a reciprocal change of variables. The two names reflect two readings of one object.
\emph{Spectrum} emphasizes which time scales are realized and how broadly they are spread; \emph{density} (or \emph{distribution}) is used more formally to name the mathematical object that appears in the stationary equations and in the asymptotic analysis. Neither should be confused with the eigenvalue or singular-value spectrum of a matrix.} and the envelope is non-power-law and decays effectively exponentially over its characteristic-scale range, and \emph{anti-collapsed} when the dynamics instead sustain a broad, power-law spectrum and the envelope decays polynomially.
Anti-collapse is thus a property of the coupled state and training dynamics, not a fixed attribute of the architecture.
We use anti-collapsed for this family of dynamical regimes and for the broad time-scale spectra they sustain.

Empirically, the pattern is asymmetric: architecture--optimizer pairs that sustain long-range learning consistently realize anti-collapsed spectra together with persistent heavy-tailed forcing, whereas collapsed spectra tend to co-occur with lighter-tailed statistics, though not exclusively, since collapse can arise under both near-Gaussian and heavy-tailed forcing.
These observations motivate heavy-tailed forcing as a possible driver of anti-collapse, but forcing alone is insufficient.
Its effect depends on the \emph{capacity} of the architecture--optimizer pair to generate, populate, and maintain a broad spectrum of effective time scales.
The dynamical mechanism by which such forcing sustains this spectrum during training has remained unclear.

In this work, we develop a coarse-grained stochastic framework that resolves this gap.
We model the collective evolution of log-effective decay rates as a population-level process in which restoring drift competes with aggregated heavy-tailed stochastic forcing.
This places the model inside a well-developed mathematical framework for L\'evy-driven stochastic processes while retaining a direct interpretation in terms of recurrent training dynamics.
Tempered stable generators form a class of stochastic models with a long track record in several domains, including finance, hydrology, and statistical physics of anomalous diffusion~\citep{carr2002fine,cont2004financial,meerschaert2008tempered,stanislavsky2008diffusion}.
We present this as one principled and analytically tractable route to anti-collapse rather than the only possible one: other mechanisms could, in principle, reach the same regime.

\paragraph{Contributions.}

\begin{itemize}

\item \emph{A formal route to anti-collapse, with explicit existence conditions.}
The proposed jump-driven route to anti-collapse combines heavy-tailed effective forcing above an explicit jump-amplitude threshold with a saturating restoring drift on the populated far-left tail.
Within the proposed stochastic model, these conditions are sufficient for anti-collapse: the stationary far-tail balance then selects a spectral exponent~$\beta$ that depends implicitly on the underlying control parameters.
Below the threshold, the time-scale tail remains a power law up to a logarithmic factor, with an exponent fixed by the tempering of the jumps instead of this balance.
The access route is derived from a tempered stable L\'evy generator and its stationary nonlocal Fokker--Planck equation.
Whether a particular architecture--optimizer pair actually realizes the anti-collapsed regime in a finite training experiment depends additionally on its capacity to populate and sustain the resulting broad time-scale spectrum.

\item \emph{A sharp phase structure with~$\beta$ as an observable order parameter.}
We develop a mathematically precise Laplace--Tauberian correspondence between the tail of the stationary time-scale density $p_\infty(\tau)$ and the asymptotic decay of the envelope $f(\ell)$.
Therefore, the realized regime determines the envelope class rather than simply providing a phenomenological fit. The canonical OU representative yields an operational exponential decay, while power-law and logarithmic envelopes follow asymptotically from their corresponding tail classes.
The analysis reveals a sharp phase structure controlled by~$\beta$, acting as an observable order parameter.
Along the jump-driven route above the threshold, anti-collapse is concentrated for $\beta>1$ and broad for $\beta<1$.
For pure power-law time-scale tails, $\beta=1$ marks a qualitative change in the time-scale distribution: its mean is finite for $\beta>1$ and divergent for $\beta\le1$.

\item \emph{Experimental tests of the observable consequences.}
We test the theoretical predictions on gated recurrent networks of variable complexity through the observables it constrains.
A structural negative control shows that heavy-tailed forcing alone is not sufficient. A frozen-gate network, in which trainable time scales are removed, remains collapsed when such forcing is injected into the updates of its slow neurons, with an exponential envelope and no restoring drift on the far-left tail.
In the opposite direction, a network with one trainable gate per neuron shows together the signatures that the theory associates with jump forcing: large non-Gaussian excursions of the slow log-effective decay rates, a positive restoring drift on the far-left tail, a broad time-scale spectrum, and a slowly decaying envelope.

\end{itemize}

These contributions recast heavy-tailed stochastic fluctuations not as noise to be suppressed but as a structural mechanism that counteracts spectral collapse and sustains the temporal reach of learning.
Figure~\ref{fig:regime_phenomenology} provides an informal sketch of the regimes predicted by the theory.

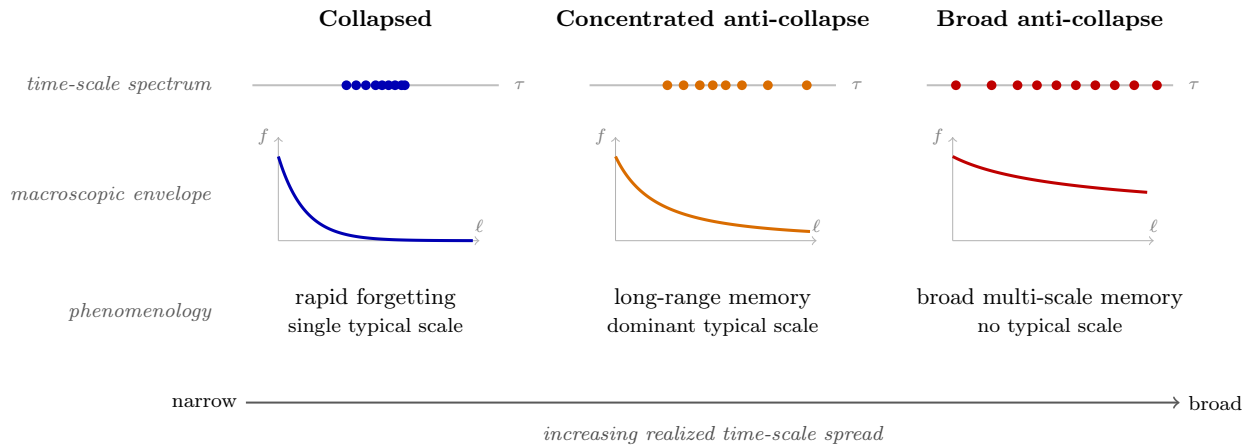
\begin{figure}[t]
\centering
\resizebox{\textwidth}{!}{%
\begin{tikzpicture}[scale=0.9]

\def\colA{0}
\def\colB{5.2}
\def\colC{10.4}
\def\specy{5.0}
\def\envytop{4.0}
\def\envybot{2.6}
\def\labely{1.5}

\node[font=\bfseries\small] at (\colA, 6.0) {Collapsed};
\node[font=\bfseries\small] at (\colB, 6.0) {Concentrated anti-collapse};
\node[font=\bfseries\small] at (\colC, 6.0) {Broad anti-collapse};

\draw[thick, gray!50] (\colA-1.9, \specy) -- (\colA+1.9, \specy);
\node[gray, font=\scriptsize, anchor=west] at (\colA+2.0, \specy) {$\tau$};
\foreach \x in {-0.45, -0.30, -0.15, 0.00, 0.10, 0.20, 0.30, 0.40, 0.45} {
    \fill[blue!70!black] (\colA+\x, \specy) circle (0.075);
}
\draw[thick, gray!50] (\colB-1.9, \specy) -- (\colB+1.9, \specy);
\node[gray, font=\scriptsize, anchor=west] at (\colB+2.0, \specy) {$\tau$};
\foreach \x in {-0.70, -0.45, -0.20, 0.00, 0.20, 0.45, 0.85, 1.45} {
    \fill[orange!85!black] (\colB+\x, \specy) circle (0.075);
}
\draw[thick, gray!50] (\colC-1.9, \specy) -- (\colC+1.9, \specy);
\node[gray, font=\scriptsize, anchor=west] at (\colC+2.0, \specy) {$\tau$};
\foreach \x in {-1.45, -0.90, -0.50, -0.20, 0.10, 0.40, 0.70, 1.00, 1.30, 1.65} {
    \fill[red!75!black] (\colC+\x, \specy) circle (0.075);
}

\node[anchor=east, font=\itshape\footnotesize, gray!70!black]
    at (\colA-2.4, \specy) {time-scale spectrum};

\foreach \cx in {\colA, \colB, \colC} {
    \draw[->, gray!60] (\cx-1.5, \envybot) -- (\cx-1.5, \envytop+0.2)
        node[anchor=east, font=\scriptsize, gray] {$f(\ell)$};
    \draw[->, gray!60] (\cx-1.5, \envybot) -- (\cx+1.6, \envybot)
        node[anchor=south, font=\scriptsize, gray] {$\ell$};
}

\draw[very thick, blue!70!black] plot[domain=0:3, samples=60, smooth]
    (\colA-1.5+\x, {\envybot + 1.3*exp(-2.5*\x)});

\draw[very thick, orange!85!black] plot[domain=0:3, samples=60, smooth]
    (\colB-1.5+\x, {\envybot + 1.3/((1+\x)^1.6)});

\draw[very thick, red!75!black] plot[domain=0:3, samples=60, smooth]
    (\colC-1.5+\x, {\envybot + 1.3/((1+\x)^0.4)});

\node[anchor=east, font=\itshape\footnotesize, gray!70!black]
    at (\colA-2.4, 3.3) {macroscopic envelope};

\node[align=center, font=\small] at (\colA, \labely) {%
    rapid forgetting\\
    \footnotesize single typical scale};
\node[align=center, font=\small] at (\colB, \labely) {%
    long-range memory\\
    \footnotesize dominant typical scale};
\node[align=center, font=\small] at (\colC, \labely) {%
    broad multi-scale memory\\
    \footnotesize no typical scale};

\node[anchor=east, font=\itshape\footnotesize, gray!70!black]
    at (\colA-2.4, \labely) {phenomenology};

\draw[->, thick, gray!70!black] (\colA-2.0, 0.1) -- (\colC+2.0, 0.1);
\node[font=\footnotesize, anchor=east] at (\colA-2.0, 0.1) {narrow};
\node[font=\footnotesize, anchor=west] at (\colC+2.0, 0.1) {broad};
\node[font=\footnotesize\itshape, gray!60!black] at (\colB, -0.4)
    {increasing realized time-scale spread};

\end{tikzpicture}%
}
\caption{Phenomenological summary of the regimes predicted by the theory.
\emph{Top row}. Schematic time-scale spectrum over~$\tau$
of the trained network: a tight cluster in the collapsed
regime, a broadened distribution with a heavy right tail in
concentrated anti-collapse, a very broad distribution with
a heavy right tail in broad anti-collapse.
\emph{Middle row}. Shape of the (macroscopic) envelope $f(\ell)$ as a function of lag~$\ell$: exponential in the collapsed regime, steep power-law in concentrated anti-collapse, shallow power-law in broad anti-collapse.
\emph{Bottom row}. The corresponding learned-temporal phenomenology.
The horizontal axis is the realized spread of the trained time-scale spectrum, not a parameter axis.
Along the proposed jump-driven route, realizing this broadening requires both the theoretical route conditions and the architecture--optimizer pair's capacity to populate and sustain the time-scale spectrum.}
\label{fig:regime_phenomenology}
\end{figure}

\paragraph{Paper structure.}
Section~\ref{sec:related_work} reviews related literature.
Section~\ref{sec:conceptual_framework} introduces the effective learning rate formalism and the population-level description of time-scale spectra.
Section~\ref{sec:stochastic_process_evolution_logeffectiveLR} presents the coarse-grained stochastic model and derives the phase structure of the stationary spectrum.
Section~\ref{sec:envelope_laws_from_tails} derives the canonical envelope decay laws from the tails of the time-scale spectra.
Section~\ref{sec:empirical} provides empirical validation of the proposed theoretical framework.
Finally, Section~\ref{sec:discussion} discusses implications, limitations, and future directions.
Appendices~\ref{app:SGD}--\ref{app:code} collect detailed derivations, architectural specifications, numerical validations, and asymptotic calculations supporting the main text.

\section{Related Work}
\label{sec:related_work}

The state-space dynamics of RNNs have long been analyzed through the lens of dynamical systems theory~\citep{ceni2020echo,mastrogiuseppe2018linking,sussillo2013opening,asabuki2025taming}.
Early work on the exploding/vanishing gradient problem \citep{pascanu2013difficulty} established conditions under which long-range temporal dependencies become numerically unfeasible.
Gated architectures such as LSTMs \citep{greff2017lstm} and GRUs \citep{cho2014gru} partially mitigated the problem by introducing gating schemes to control information flow in state space and develop heterogeneous temporal behavior during training.
Random-matrix and mean-field analyses \citep{tankut2020gating,chen2018dynamical} explain how this structure arises at initialization and affects trainability, mapping specific gating schemes with dynamical aspects, such as accumulation of slow modes and effective spectral properties.
Continuous-time models show that gating produces line attractors and can induce transitions to chaos \citep{krishnamurthy2022theory}.
A complementary line of work derives gating from general principles such as quasi-invariance to time warping \citep{tallec2018can}, adaptive time scales in neural ODEs \citep{kim2023gnode}, and time-delay feedback from delay-differential equations \citep{erichson2025delay}, reinforcing the view that gating is a principled dynamical mechanism rather than a heuristic to stabilize training.
\citet{zucchet2024vanishing} show that network input--output sensitivity grows with memory length, even when the gradient norm remains stable.

A parallel line of research studies SGD-based optimization itself as a stochastic process.
\citet{mandt2017sgd} showed that constant-step SGD can be approximated by an Ornstein--Uhlenbeck (OU) process with a Gaussian invariant measure, and \citet{yaida2019fluctuation} established fluctuation--dissipation relations constraining any stationary state.
At the level of full wide-network solutions, \citet{avidan2025unified} unify the neural tangent kernel and neural network Gaussian process descriptions through gradient descent with additive noise, deriving a time-dependent neural dynamical kernel and identifying two learning phases: a fast gradient-driven phase, whose time scale is set by the initialization variance, followed by a slow diffusive phase in which the parameters sample a Gibbs equilibrium at a rate set by the noise and the prior variance.
Recent work~\citep{simsekli2019tail} reveals that gradient noise in trained networks is often heavy-tailed. Modeling stochastic optimization as discrete random recurrence relations, \citet{hodgkinson2021multiplicative} show that multiplicative noise from variance in local convergence rates produces heavy-tailed stationary parameter distributions across a broad class of optimizers and non-convex models. \citet{nguyen2019first} analyze the first-exit time of the discretized SGD recursion under heavy-tailed gradient noise, showing that exit-time scaling is governed by the noise tail index. \citet{barsbey2021heavy} link the heavy-tailed dynamics of SGD to the compressibility of overparametrized networks, giving a theoretical basis for the empirical effectiveness of simple pruning. \citet{ziyinueda2025stationary} show that SGD minibatch noise regularizes toward a noise-balanced solution under rescaling parameter symmetries, and derive a stationary distribution for diagonal linear networks exhibiting phase transitions, broken ergodicity, and fluctuation inversion.
A matrix-valued stochastic model has been proposed by \citet{olsen2025sgd}, who derive continuous-time stochastic differential equations (SDEs) for the singular values of weight matrices under SGD, identify the squared-singular-value dynamics with Dyson Brownian motion, and show that the stationary spectral density is a gamma-type law with a power-law tail, providing an SDE-level account of the empirically observed trained weight spectra.
Complementary statistical-physics analyses of high-dimensional non-convex landscapes reveal aging, slow relaxation, and trapping near marginally stable threshold states \citep{Bonnaire_2024}, while related work shows that implicit regularization biases optimization toward flat minima \citep{xie2021diffusion,smith2021origin}.

A well-established research line exploits the long-standing hypothesis that neural systems benefit from operating near criticality~\citep{hesse2014self}.
Recurrent models have been studied at the edge of chaos \citep{engelken2023lyapunov,bertschinger2004edge,livi2018edge,kadmon2015nonlinear} as a regime enabling rich computation in state space, and subsequent work linked phase-transition phenomena to signal propagation and generalization in deep networks \citep{poole2016exponential}; more recent works study the concept of computation on the edge focusing on the dynamics of the optimizer~\citep{dynamicalSGDchemnitz2025}, i.e. the dynamics underlying parameter updates. Mean-field analyses have since been extended to recurrent and gated architectures~\citep{massar2013mean}, with \citet{chen2018dynamical} showing that gating enables dynamical isometry and that the ordered-to-chaotic transition is modulated by the gate bias.

\section{Conceptual framework}
\label{sec:conceptual_framework}

This section introduces the effective learning rates, the population-level description of their spectrum, and the foundational assumptions that underpin the stochastic model of Section~\ref{sec:stochastic_process_evolution_logeffectiveLR}.
We use \emph{coarse-grained} and \emph{mesoscopic} synonymously:
both refer to the level of description in which individual SGD
updates are aggregated into slowly evolving effective variables, while
the microscopic parameter updates and sequence-time dynamics remain
implicit.

\subsection{Effective learning rates and their spectrum}
\label{sec:asymptotic_decay_log_effective_LR}

As originally introduced in~\citep{livi2025timescale,livi2026learnability}, under a first-order diagonal approximation of the recurrent Jacobian product, the lag-$\ell$ contribution to the BPTT parameter gradient decomposes into neuron-wise terms, each modulated by a scalar effective learning rate $\mu^{(q)}_{t,\ell}$.
These are the mesoscopic coupling variables between state-space
dynamics and parameter-space dynamics: each $\mu^{(q)}_{t,\ell}$
encodes how the forward evolution of gated hidden states in neuron~$q$
modulates the backward transport of gradient information across $\ell$
recurrent steps, thereby governing how effectively the optimizer can
assign credit to temporally distant events.

\paragraph{Generalized effective learning rate (GELR) factorization.}
Under general SGD-like optimizers, the effective learning rate admits the factorization
\begin{equation}
\label{eq:GELR_factorization}
\mu^{(q)}_{t,\ell}
\;=\;
\Lambda^{(q)}_{r,\ell}\,\Gamma^{(q)}_{t,\ell},
\end{equation}
where $\Gamma^{(q)}_{t,\ell}$ is the \emph{transport factor}, aggregating the zeroth- and first-order diagonal contributions of the recurrent Jacobian product (Appendix~\ref{app:gated_rnns}), and $\Lambda^{(q)}_{r,\ell}>0$ is the \emph{adaptive base rate}, obtained by projecting the optimizer's diagonal preconditioner onto the parameter-space direction associated with neuron~$q$ at lag~$\ell$ using the Rayleigh quotient~\citep{livi2026learnability}.
Under plain SGD, $\Lambda^{(q)}_{r,\ell}=\mu$ for all $q$ and $\ell$, and the original formulation of~\citet{livi2025timescale} is recovered.

The Rayleigh quotient satisfies $\lambda_{\min,r}\le\Lambda^{(q)}_{r,\ell}\le\lambda_{\max,r}$, where $\lambda_{\min,r}$ and $\lambda_{\max,r}$ are the extremal entries of the optimizer preconditioner at iteration~$r$~\citep{livi2026learnability}.
In particular, $\Lambda^{(q)}_{r,\ell}$ is bounded and strictly positive, so it modulates the amplitude of neuron-wise transport without altering the dominant decay rate, which is determined entirely by~$\Gamma^{(q)}_{t,\ell}$.

\paragraph{Asymptotic decay rate.}
The lag-dependent effective learning rate $\mu^{(q)}_{t,\ell}$ is not itself a single-number summary of neuron $q$.
We therefore characterize each neuron $q$ by an asymptotic per-step decay rate that is intensive in the lag, removes finite-lag transients, and connects directly to the time-scale interpretation $\tau_q = 1/\bar{\mu}_q$ used in the rest of the paper:
\begin{equation}
\label{eq:mu_asymptotic_decay_rate}
\bar{\mu}_q
=
-\lim_{\ell\to\infty}
\frac{1}{\ell}\,
\mathbb{E}
\!\left[
\log\big|\mu^{(q)}_{t,\ell}\big|
\right],
\end{equation}
where the expectation averages over training dynamics in the late-training regime (Section~\ref{sec:sgd_coarse_time}).
Here $\ell$ is read as the temporal displacement (lag) over which gradient information decays, so that the factor $1/\ell$
extracts the corresponding per-step rate; the same symbol
$\ell$ may appear elsewhere in the paper as a summation index
in the BPTT chain rule~\eqref{eq:bptt_one_step_app}.

Under the GELR factorization~\eqref{eq:GELR_factorization}, the
limit is determined by the zeroth-order gate-product transport,
with the optimizer amplitude $\Lambda^{(q)}_{r,\ell}$ contributing
negligibly. It is well-defined under a finite stationary log-moment
condition on the architecture-specific gate-derived retention
factor, and strictly positive whenever the gate process induces a
non-degenerate average contraction; the full derivation, including
the subexponential bound on the first-order corrections, is in
Appendix~\ref{app:product_structure}.

This construction has the mathematical shape of a Lyapunov
exponent~\citep{viana2014lectures}: products of bounded transport
factors become additive after taking logarithms, and the per-lag time-average selects an asymptotic rate.
It is, however, conceptually distinct. Lyapunov exponents quantify the sensitivity of state trajectories to perturbations, whereas $\bar{\mu}_q$
quantifies the asymptotic decay of the coupling strength between states and parameter updates.

\paragraph{Empirical spectrum of log-effective decay rates.}
We introduce the log-effective decay rate
\begin{equation}
\label{eq:log_effective_LR}
\zeta_q = \log \bar{\mu}_q \in \mathbb{R},
\end{equation}
and define the empirical spectrum of a network with $H$ neurons
as the normalized sum of point masses
\begin{equation}
\label{eq:log-effective_empirical_distribution}
\hat{\rho}_H(\zeta,t)
=
\frac{1}{H}
\sum_{q=1}^{H}
\delta(\zeta-\zeta_q(t)).
\end{equation}
This is a random, finite-dimensional object that depends on the particular network realization, training run, and training time.
It describes the instantaneous distribution of log-decay rates across the $H$ neurons at training time~$t$.
The logarithmic variable $\zeta$ is the natural coordinate for the analysis to follow.
Most importantly, the log transform maps the small-$\bar{\mu}$ regime, where the slowest neurons that govern long-range learning reside, to a left-tail problem on the corresponding spectrum, which is the natural setting for the stochastic model developed in the rest of the manuscript.
Two further properties support the choice: stochastic variability that acts
multiplicatively on $\mu^{(q)}_{t,\ell}$ becomes additive on
$\zeta$, so fluctuations of the effective rate translate into
tractable additive noise on the log-rate; and $\zeta$ provides a
scale-invariant coordinate for comparing decay rates that may
span several orders of magnitude across neurons.

\paragraph{Three coordinates.}
The same underlying quantity is described throughout the paper in
three equivalent coordinates, each adapted to a different aspect
of the analysis:
\begin{equation}
\begin{aligned}
\label{eq:tau_def}
\bar{\mu}_q &\in (0,\infty)
\quad\text{(asymptotic decay rate)}, \\
\zeta_q &= \log\bar{\mu}_q \in \mathbb{R}
\quad\text{(log-effective decay rate)}, \\
\tau_q &= \bar{\mu}_q^{-1} \;=\; e^{-\zeta_q} \in (0,\infty)
\quad\text{(effective time scale)}.
\end{aligned}
\end{equation}
The effective time scale $\tau_q$ is the natural coordinate for
the envelope and memory interpretation (Section~\ref{sec:envelope_laws_from_tails}): it is the characteristic memory horizon of neuron~$q$.
We will move freely between the three coordinates, using $\zeta$
when formulating the stochastic dynamics and $\tau$ when
discussing envelopes and tail behavior, and we refer to
distributions over $\tau$ as time-scale spectra.

\subsection{Training time and coarse-graining}
\label{sec:sgd_coarse_time}

The quantities $\bar{\mu}_q$, $\zeta_q$, and the empirical spectrum $\hat{\rho}_H$
are induced functionals of the network parameters: they depend on long
products of recurrent Jacobians and gating factors, and they change as
SGD updates those parameters.
While individual gradient steps can produce substantial short-time
fluctuations, the collective evolution of these quantities is slow
when observed over many updates.
The underlying gradient dynamics are summarized in
Appendix~\ref{app:SGD} and will not be repeated here.
Instead, we clarify the notion of \emph{training time} used in the
main text and the conditions under which a coarse-grained description
of these slow variables becomes appropriate.

Optimization unfolds over discrete parameter updates, whereas
the decay rates $\bar{\mu}_q$ and the spectrum $\hat{\rho}_H$
depend on averages over long temporal horizons and many updates.
In the same spirit as previous work modeling SGD as a continuous-time
process in parameter space~\citep{mandt2017sgd,li2017stochastic},
we model the slow evolution of the log-decay
rates~$\zeta_q$~\eqref{eq:log_effective_LR} as a continuous-time
stochastic process on a \emph{coarse training time} variable~$t$.
The goal is to provide a tractable stochastic model whose stationary
distribution captures the late-training spectrum $\hat{\rho}_H$ and
the tail structure that governs envelope decay.
Conceptually, one unit of coarse-grained training time corresponds to a block of SGD updates, so that $\bar{\mu}_q$ and $\hat{\rho}_H$ evolve only slightly.
All time derivatives, stochastic differential equations, and population densities appearing in subsequent sections are understood with respect to this coarse-grained training time, which is distinct from the sequence-time index that appears in $\mu^{(q)}_{t,\ell}$ and in the BPTT notation of Appendix~\ref{app:SGD}.

\subsection{Foundational assumptions}
\label{sec:foundational_assumptions}

The empirical
spectrum~$\hat{\rho}_H$~\eqref{eq:log-effective_empirical_distribution}
is a random, finite-dimensional measure that changes from run to run.
To develop a tractable theory, we require two structural hypotheses
that together justify replacing $\hat{\rho}_H$ by a deterministic,
quasi-stationary population density $\rho_\infty(\zeta)$.

\begin{assumption}[Quasi-stationary spectrum]
\label{ass:quasi_stationarity}
In the late-training regime, the empirical
spectrum~$\hat{\rho}_H(\zeta,t)$ evolves sufficiently slowly on the
coarse training-time scale that it admits an approximately
time-independent description over the time horizons relevant for
gradient transport.
\end{assumption}

This assumption restricts the analysis to the late-training regime, in which optimization trajectories evolve slowly within broad regions of parameter space and parameter updates produce only gradual changes in macroscopic observables.
Recent analyses of gradient-based optimization in trained deep networks provide support for this picture at the level of the underlying dynamics.
For example, \citet{Bonnaire_2024} show that gradient-descent dynamics in high-dimensional non-convex landscapes can exhibit aging, with relaxation times that increase with the age of the dynamics, and can remain trapped near marginally stable threshold states.
Via a mean-first-passage-time analysis on multifractal landscapes, \citet{ly2025multifractal} similarly find that the optimizer preferentially settles in smoother basins containing flatter minima, with escape times that lengthen as the basin smooths.
While these results are formulated at the level of parameter or collective-variable dynamics, they support the expectation that macroscopic observables depending smoothly on the parameters, such as the empirical spectrum $\hat{\rho}_H(\zeta,t)$, also evolve on slow time scales in late training.
Complementarily, \citet{yaida2019fluctuation} established fluctuation--dissipation relations characterizing approximate stationary states of SGD, showing that, in late training phases, noise and curvature balance to produce stable stochastic fluctuations.
In the same spirit, the analytic wide-network theory of \citet{avidan2025unified} separates training into a fast gradient-driven phase and a slow diffusive phase in which the parameters sample an equilibrium distribution, so that late training is a slowly evolving sampling stage rather than continued rapid descent.

The quasi-stationary assumption breaks down during early training phases, under severe nonstationarities, or in regimes of rapid loss-landscape reorganization.
However, our analysis explicitly targets the late-training regime and focuses on the tail structure of time scales reached at that stage.

\begin{assumption}[Population-level concentration]
\label{ass:mean_field}
In the late-training regime, the empirical
spectrum~$\hat{\rho}_H(\zeta,t)$ defined in
\eqref{eq:log-effective_empirical_distribution} concentrates as
$H\to\infty$ in the sense that the empirical probability measure converges weakly.
That is,
\begin{equation}
\label{eq:population_density_limit}
\hat{\rho}_H(\zeta,t)\,d\zeta
\;\Rightarrow\;
\rho(\zeta,t)\,d\zeta,
\qquad H\to\infty,
\end{equation}
in probability, in the sense of weak convergence of probability
measures~\citep{durrett2019probability}. Equivalently, for every
bounded continuous test function~$\varphi$,
\begin{equation}
\int_{\mathbb R} \varphi(\zeta)\hat{\rho}_H(\zeta,t)\,d\zeta
\;\xrightarrow{\;P\;}\;
\int_{\mathbb R} \varphi(\zeta)\rho(\zeta,t)\,d\zeta .
\end{equation}
Here $\rho(\zeta,t)$ is a deterministic population density that is no
longer random and no longer a sum of delta functions.
From this point onward, $\rho(\zeta,t)$ always refers to this
deterministic large-width limit, not to the finite-width empirical
measure~$\hat{\rho}_H$.
\end{assumption}

The concentration assumption requires each neuron's log-decay rate
$\zeta_q$ to be a well-defined scalar.
This is established by the first-order diagonal expansion of gradient
transport
products~\citep{livi2025timescale,livi2026learnability}: the expansion
retains the dominant zeroth-order gate-product contribution together
with first-order diagonal corrections, while discarding off-diagonal
cross-neuron terms. Numerical simulations confirm that the perturbative regime is satisfied in the architectures considered here~\citep{livi2025timescale}.
Under the GELR factorization~\eqref{eq:GELR_factorization}, $\zeta_q$
depends on $\Gamma^{(q)}_{t,\ell}$, which is primarily determined by
each neuron's own gate products, while the bounded adaptive base
rate $\Lambda^{(q)}_{r,\ell}$ acts neuron-wise and introduces no
cross-neuron coupling.
These observations justify treating the $\{\zeta_q\}$ as a population
of weakly coupled scalar degrees of freedom.

This is a mean-field-style modeling hypothesis. As width grows, the
finite empirical spectrum is replaced by a deterministic population measure.
Classical mean-field results for wide networks \citep{mei2018mean,sirignano2020mean,nguyen2020rigorous} provide the closest precedent for such deterministic large-width descriptions.
Here, the object of interest is the derived spectrum of log-effective decay rates obtained from the GELR construction, rather than the parameter distribution or network output itself, so we state the concentration as an explicit assumption.

Assumption~\ref{ass:quasi_stationarity} is supported indirectly by late-training results from the literature and directly by the empirical drift-closure validation of Section~\ref{sec:exp_access_route}, which shows that the estimated restoring drift in the populated far-left tail is stable enough over the sampled late-training window to admit an approximately time-homogeneous description.
Assumption~\ref{ass:mean_field}, by contrast, is supported primarily
by the width-scaling diagnostics of
Appendix~\ref{app:mixture_and_large_width_validation}, which show
concentration of the empirical time-scale distribution and
stabilization of the intensive envelope as width increases.

Under Assumptions~\ref{ass:quasi_stationarity} and~\ref{ass:mean_field}, the deterministic population density
$\rho(\zeta,t)$ admits an approximately time-independent limit
$\rho_\infty(\zeta)$, i.e. the stationary spectrum of log-effective decay rates.
The remainder of the paper is devoted to modeling the dynamics that
generate $\rho_\infty$, characterizing its tail properties as a
function of control parameters, deriving the associated time-scale spectrum, and deriving the resulting implications in terms of envelope scaling laws.

\section{Coarse-grained stochastic model and phase structure}
\label{sec:stochastic_process_evolution_logeffectiveLR}

We now formalize the central hypothesis by introducing a coarse-grained stochastic model for the evolution of the log-effective decay rates~\eqref{eq:log_effective_LR}.
The technical underpinnings of this section are developed in two complementary appendices. Appendix~\ref{app:levy_generator} specifies the L\'evy generator used throughout and discusses its modeling scope.
Appendix~\ref{app:nonlocal_balance} derives the nonlocal stationary balance and the resulting characteristic equation. It gives the jump-amplitude threshold $\eta_J^*$ above which the equation has a unique interior root. This root determines the spectral exponent~$\beta$.

\subsection{Stochastic model for the log-effective decay rates}
\label{sec:stochastic_modeling}

Rather than modeling microscopic SGD parameter updates, we work directly with the emergent population of log-effective decay rates $\zeta_q(t)$, whose empirical spectrum~$\hat{\rho}_H$~\eqref{eq:log-effective_empirical_distribution} encodes the macroscopic temporal organization relevant for learnability.
Considering Assumptions~\ref{ass:quasi_stationarity} and~\ref{ass:mean_field}, it suffices to model a single representative variable $\zeta(t)$ whose law generates $\rho(\zeta,t)$ through the associated forward equation; in the quasi-stationary regime this yields the stationary spectrum $\rho_\infty(\zeta)$.

We model the evolution of $\zeta(t)$ by a one-dimensional SDE with deterministic drift, Gaussian diffusion, and L\'evy jump forcing~\citep{applebaum2009levy,sato1999levy}
\begin{equation}
\label{eq:zeta_SDE_mixed}
d\zeta(t)
=
F(\zeta(t))\,dt
+
\sqrt{2\eta_G}\,dW_t
+
dJ_t,
\end{equation}
where $W_t$ is standard Brownian motion and $J_t$ is a pure-jump L\'evy process with the tempered-stable L\'evy measure specified below.
Such drift--diffusion--jump SDEs are standard in the L\'evy-driven SDE literature, including models combining Brownian diffusion with pure-jump L\'evy forcing, and potential-driven systems perturbed by stable L\'evy processes~\citep{gloter2018jump,imkeller2006first}.

The three terms of the SDE~\eqref{eq:zeta_SDE_mixed} encode distinct physical mechanisms.
The deterministic drift $F(\zeta)$ models the net restoring tendency imposed by architecture and training: it pushes log-decay rates back toward a typical range, preventing the population from drifting indefinitely toward extreme values.
The Brownian term $\sqrt{2\eta_G}\,dW_t$, with diffusion intensity
$\eta_G>0$, captures the cumulative effect of many small stochastic
perturbations~\citep{mandt2017sgd,li2017stochastic}.
The jump term $dJ_t$ models intermittent large-scale deviations
in $\zeta$ that the Gaussian background cannot account for.
Because the model is mesoscopic, the heavy-tailed forcing it summarizes, parametrized by $(\eta_J,\alpha_{\mathrm{jump}})$, abstracts from the microscopic origin of these deviations.
It may arise from minibatch gradient fluctuations propagated through recurrent or forward dynamics~\citep{simsekli2019tail,gurbuzbalaban2021heavy,hodgkinson2021multiplicative}, from weight matrices with heavy-tailed spectra~\citep{martin2021implicit}, from the distribution of parameter updates~\citep{zhang2025heavytailed}, from data fluctuations or from any other microscopic contribution whose net effect on the $\zeta$ increments is heavy-tailed. All such contributions enter the model only through their aggregated statistics.

Concretely, $J_t$ is modeled as a symmetric tempered $\alpha$-stable L\'evy process: its jumps follow a symmetric $\alpha$-stable distribution with stability index $\alpha_{\mathrm{jump}}\in(0,2)$, modified by an exponential tempering factor that suppresses arbitrarily large excursions~\citep{rosinski2007tempering,kuchler2013tempered}.
The increments of $J_t$ are characterized by the symmetric tempered L\'evy measure:
\begin{equation}
\label{eq:tempered_levy_measure}
\nu_{\alpha_{\mathrm{jump}},\lambda}(dy)
=
\eta_J\,c_{\alpha_{\mathrm{jump}}}
\frac{e^{-\lambda|y|}}{|y|^{1+\alpha_{\mathrm{jump}}}}\,dy,
\qquad
\alpha_{\mathrm{jump}}\in(0,2),\ \lambda>0.
\end{equation}
In this expression, $\eta_J\ge 0$ is the jump amplitude that controls the overall intensity of the jump component:
when $\eta_J=0$ the jump term is absent and only Gaussian diffusion remains.
The stability index $\alpha_{\mathrm{jump}}\in(0,2)$ governs the intermediate-scale heaviness of the jump-size distribution; smaller values correspond to more frequent large jumps.
The tempering parameter $\lambda>0$ exponentially suppresses
jumps beyond the scale $1/\lambda$, reflecting the physical
constraint that bounded activations and finite hidden-state
dimension prevent arbitrarily large shifts in a neuron's
log-decay rate. Mathematically, this suppression gives the
tempered L\'evy measure finite exponential moments
\citep{kuchler2013tempered,rosinski2007tempering}. The constant
$c_{\alpha_{\mathrm{jump}}}$ fixes the normalization convention
for this stable-like jump measure.

All parameters in the stochastic model are \emph{effective} quantities that emerge from the specific experimental realization of the architecture--optimizer pair and should not be confused with hyperparameters.
The modeling details justifying the choice of the mixed drift--diffusion--jump SDE~\eqref{eq:zeta_SDE_mixed} are discussed in Appendix~\ref{app:levy_generator}.

\paragraph{Restoring drift.}
As part of the coarse-grained closure, we assume that the effective
drift $F(\zeta)$ is restoring in the late-training regime, so that
the log-effective decay rates do not drift indefinitely and a
stationary population description remains admissible. Two distinct
features of~$F$ will enter the stationary analysis, each controlling
a different asymptotic regime of the stationary density and therefore
a different phase of the model: a near-equilibrium linearization
around the bulk, and a structural closure in the far-left tail.

\emph{Near the bulk.} We assume that there exists a
typical log-decay rate~$\zeta^*$ (an interior equilibrium of
the deterministic part of the dynamics) at which $F(\zeta^*)=0$, with
$F(\zeta)<0$ for $\zeta>\zeta^*$ and $F(\zeta)>0$ for $\zeta<\zeta^*$.
The value~$\zeta^*$ is an emergent quantity of the coarse-grained
model, the log-decay rate around which the bulk of the population
concentrates in steady state, and appears as the center of the
Gaussian bulk in the collapsed-regime analysis; we only use the fact that such
a point exists, not a specific expression for it.
Taylor-expanding $F$ to first order around~$\zeta^*$ yields
\begin{equation}
\label{eq:drift_linearization}
F(\zeta)\;\approx\;-\gamma\,(\zeta-\zeta^*),
\qquad \gamma = -F'(\zeta^*)>0,
\end{equation}
so that $\zeta(t)$ behaves locally as a mean-reverting
OU process with restoring rate~$\gamma$ and
equilibrium level~$\zeta^*$, see e.g.~\citep{Bonnaire_2024,mandt2017sgd} for similar assumptions.
Positivity of~$\gamma$ is exactly the statement that $\zeta^*$ is a stable fixed point of the deterministic drift.
This linear-drift regime controls the collapsed-regime analysis in Section~\ref{sec:collapsed_regime}.

\emph{In the far-left tail}, corresponding to large
effective time scales $\tau$~\eqref{eq:tau_def}, we
instead adopt the structural asymptotic closure
\begin{equation}
\label{eq:drift_saturation}
F(\zeta)\;=\;\kappa + o(1),
\qquad\kappa>0,
\quad\text{as }\zeta\to-\infty,
\end{equation}
where~$\kappa$ plays the role of the asymptotic restoring
strength acting on neurons with extremely long effective
time scales. Equation~\eqref{eq:drift_saturation} should be
read as a modeling assumption on the far-tail behavior of
the coarse-grained drift.

The two closures rest on different parts of~$F$ and are
therefore logically independent. The collapsed-regime
analysis depends only on the local restoring geometry
of~$F$ near~$\zeta^*$ through the
linearization~\eqref{eq:drift_linearization}, and does not
invoke the far-tail closure~\eqref{eq:drift_saturation} at
all. The anti-collapsed analysis depends on~$\kappa$ but not on the local rate~$\gamma$ at~$\zeta^*$. A single drift profile~$F$ can carry both features simultaneously: a stable zero with local slope~$\gamma$ near~$\zeta^*$ and an asymptotic plateau~$\kappa$ as $\zeta\to-\infty$.

A detailed derivation of the near-equilibrium linearization is given in Appendix~\ref{app:drift_structure}, and a targeted empirical validation of the far-tail closure~\eqref{eq:drift_saturation} is given in Section~\ref{sec:exp_access_route}.
The Gaussian-confining null in Appendix~\ref{app:gaussian_null} provides a complementary comparison. Gaussian-only forcing combined with a confining potential that grows faster than linearly excludes a regularly varying stationary time-scale tail and the corresponding power-law envelope.

\paragraph{Infinitesimal generator and its role.}
The SDE~\eqref{eq:zeta_SDE_mixed} specifies how a representative log-decay rate moves in time, but the phase theory requires a stationary equation whose tail can be analyzed.
We therefore introduce the infinitesimal generator $\mathcal{L}_{\omega}$: it packages the drift, diffusion, and jump mechanisms into the operator that acts on smooth compactly supported
test functions~$\varphi$; its adjoint governs the evolution of the
population density (Appendix~\ref{app:sde_to_generator}).
For the process~\eqref{eq:zeta_SDE_mixed}, the generator takes the form
\begin{equation}
\label{eq:generator_definition}
\mathcal{L}_{\omega}\varphi(\zeta)
=
F(\zeta)\,\partial_\zeta \varphi(\zeta)
+
\eta_G\,\partial_{\zeta\zeta}\varphi(\zeta)
+
\mathcal{J}_{\alpha_{\mathrm{jump}},\lambda}[\varphi](\zeta).
\end{equation}
Each term has a direct interpretation.
The first-order derivative
$F(\zeta)\,\partial_\zeta\varphi$ captures the effect of the
deterministic drift, which pushes the process toward its typical range.
The second-order derivative
$\eta_G\,\partial_{\zeta\zeta}\varphi$ captures how Gaussian diffusion
spreads the distribution locally (this is the standard diffusion
operator with intensity~$\eta_G$).
The nonlocal operator
$\mathcal{J}_{\alpha_{\mathrm{jump}},\lambda}[\varphi]$ is the
compensated L\'evy jump operator.
Its explicit form and L\'evy--Khintchine representation are given in Appendix~\ref{app:levy_generator}.

At this level of generality, the generator~\eqref{eq:generator_definition}
acts on compactly supported test functions~$\varphi$ and may appear
removed from the concrete quantity of interest, i.e. $\rho_\infty(\zeta)$.
Its analytical payoff comes after passing to the adjoint forward
equation for the density: in the stationary far-left tail,
substituting the density ansatz $\rho_\infty(\zeta)\sim c\,e^{\beta\zeta}$ reduces the nonlocal
balance to an algebraic characteristic equation for the spectral
exponent~$\beta$ (Section~\ref{sec:anticollapsed_regime}; Appendix~\ref{app:nonlocal_balance}).
The five control parameters of the generator are collected into the
control-parameter vector
\begin{equation}
\label{eq:control_parameter_vector}
\omega=(\kappa,\eta_G,\eta_J,\alpha_{\mathrm{jump}},\lambda),
\end{equation}
comprising the asymptotic far-left-tail drift strength~$\kappa$
from~\eqref{eq:drift_saturation}, the Gaussian diffusion
intensity~$\eta_G$, the jump amplitude~$\eta_J$, the
intermediate heavy-tail index~$\alpha_{\mathrm{jump}}$, and the
tempering parameter~$\lambda$.
Here $\omega$ collects the effective parameters relevant for the
far-left-tail phase structure of the model; the collapsed-regime Gaussian-bulk specialization (Section~\ref{sec:collapsed_regime}; Appendix~\ref{app:exponential_envelope_full_details}) additionally depends on the
local linearization rate~$\gamma=-F'(\zeta^*)$ at the
bulk equilibrium~$\zeta^*$ introduced
in~\eqref{eq:drift_linearization}.
The root-selected exponent~$\beta$ depends on the drift~$F$
only through its far-left limit~$\kappa$. The Gaussian approximation
near the bulk equilibrium instead uses the local slope~$\gamma$.
We therefore keep~$\omega$ as the phase-structure control vector
and treat~$\gamma$ as a separate local bulk parameter rather than
adding it to~$\omega$.

The term \emph{control parameter} is used in the dynamical-systems
sense: $\omega$ parametrizes the family of stationary densities
$\rho_\infty(\zeta;\omega)$ in the far-tail regime and
thereby controls the phase structure.
It is not a set of practitioner-chosen hyperparameters.
Each component of~$\omega$ is an emergent effective quantity
that summarizes the net effect of architecture, optimizer,
and data distribution on the mesoscopic dynamics of the
log-effective decay rates.
For instance, the far-left-tail drift strength~$\kappa$ reflects
how strongly the architecture and training objective restore
log-decay rates that have drifted into the large-time-scale tail, and the
jump parameters $(\eta_J, \alpha_{\mathrm{jump}}, \lambda)$
encode the aggregate heavy-tailed fluctuations as filtered through the coupled dynamics.
In particular, $\alpha_{\mathrm{jump}}$ summarizes the heaviness of coarse increments and
does not need to coincide with the microscopic fluctuations.

\paragraph{Population-level forward equation.}
The generator~\eqref{eq:generator_definition} describes how smooth observables evolve under the stochastic dynamics.
To obtain an equation for the population density $\rho(\zeta,t)$, we pass to the adjoint description of the same evolution.
The density $\rho$ evolves according to the nonlocal Fokker--Planck equation~\citep{applebaum2009levy}
\begin{equation}
\label{eq:general_forward_equation}
\partial_t \rho(\zeta,t)
=
\mathcal{L}_{\omega}^{*}\rho(\zeta,t)
=
-\partial_\zeta(F(\zeta)\rho(\zeta,t))
+
\eta_G\,\partial_{\zeta\zeta}\rho(\zeta,t)
+
\mathcal{I}^*_{\alpha_{\mathrm{jump}},\lambda}[\rho](\zeta,t).
\end{equation}
The derivation of~\eqref{eq:general_forward_equation} from the generator~\eqref{eq:generator_definition} is given in Appendix~\ref{app:characteristic_derivation}.
The three terms describe how probability mass is pushed by the restoring drift, spread locally by Gaussian diffusion, and redistributed nonlocally by jumps through the adjoint integral operator $\mathcal{I}^*_{\alpha_{\mathrm{jump}},\lambda}$.
In the late-training regime, the density changes slowly and $\partial_t\rho\approx 0$.
The stationary spectrum $\rho_\infty(\zeta)$ is the density at which these three mechanisms balance exactly, yielding $\mathcal{L}_{\omega}^{*}\rho_\infty = 0$.
In the jump-driven regime, this stationary density should be understood as a non-equilibrium steady state: it is a stationarity condition for the population density, not a detailed-balance assumption.
Probability mass can circulate through nonlocal jump redistribution and restoring drift, while the density remains time-independent.
Under the constant-drift closure~\eqref{eq:drift_saturation}, an exponential tail ansatz reduces the nonlocal Fokker--Planck equation \eqref{eq:general_forward_equation} to an algebraic characteristic equation.
The endpoint condition at $\beta=\lambda$ gives the explicit jump-amplitude threshold $\eta_J^*$.
Above this threshold, the unique interior root determines the spectral exponent~$\beta$.

\subsection{Phase structure of the stationary time-scale spectrum}
\label{sec:phase_structure}

The stationary time-scale spectrum $p_\infty(\tau)$, obtained from $\rho_\infty(\zeta)$ through the change of variables~\eqref{eq:tau_def}, changes qualitatively as the control parameters $\omega$ vary.
We first describe the canonical collapsed representative and then develop the primary jump-driven route to anti-collapse (Figure~\ref{fig:phase_diagram}).
On this branch, the stationary time-scale distribution has a power-law tail whose exponent $\beta(\omega)$ is selected by the characteristic equation.

We write $\beta(\omega)$ when the parameter dependence matters and abbreviate it to $\beta$ otherwise.
The exponent is defined implicitly by the characteristic equation (Appendix~\ref{app:nonlocal_balance}).
The existence threshold $\eta_J^*$ marks the onset of this root-selected branch; it depends on the remaining control parameters and vanishes when $\lambda\ge\kappa/\eta_G$.
We discuss the diffusion-only limit and the sub-threshold regime with jumps separately in Section~\ref{sec:existence_threshold}.

Within the root-selected branch, the boundary $\beta=1$ separates concentrated anti-collapse ($\beta>1$, finite mean time scale) from broad anti-collapse ($\beta<1$, divergent mean time scale).
\begin{figure}[tp!]
\centering
\begin{tikzpicture}[scale=1.0]
  \draw[->] (0,0) -- (7.5,0) node[below] {$\eta_J$ \small{(jump amplitude)}};
  \draw[->] (0,0) -- (0,6.5) node[left] {$\kappa$ \small{(drift strength)}};

  \draw[very thick, red]
    plot[smooth, domain=0.8:7.0, samples=50]
    (\x, {0.75 + 3.2/\x})
    node[above left, black] {\small $\beta=1$};

  \node[align=center] at (2.2, 5.4)
    {\textbf{Concentrated}\\[2pt]
     \small $\beta>1$\\[1pt]
     \small finite $\mathbb{E}[\tau]$\\[1pt]
     \small steep envelope};

  \node[align=center] at (3.0, 0.8)
    {\textbf{Broad}\\[1pt]
     \small $\beta<1$, \;$\mathbb{E}[\tau]=\infty$\\[1pt]
     \small shallow envelope};

  \draw[<-, thick, gray, dashed] (1.5, 2.0) -- (5.5, 2.0);
  \node[gray, anchor=south] at (3.7, 2.0) {\small $\beta$ decreases};

  \draw[thick, blue, dotted] (-0.06, 0.3) -- (-0.06, 6.2);
  \node[blue, rotate=90, anchor=south] at (-0.07, 3.2){\small Collapsed ($\eta_J\!=\!0$)};

  \node[red, anchor=west] at (2.0, 2.4) {\small $\mathcal{M}_c$};

  \node[anchor=south] at (3.8, 6.6)
    {\small \emph{Anti-collapsed regime} ($\eta_J>\eta_J^*$)};
\end{tikzpicture}
\caption{Schematic phase diagram in the $(\eta_J,\kappa)$ plane for $\lambda\ge\kappa/\eta_G$, where the jump-amplitude threshold vanishes and the anti-collapsed regime occupies the entire interior $\eta_J>0$.
The time-scale distribution has a power-law tail $p_\infty(\tau)\sim c\,\tau^{-1-\beta}$.
The red curve is the critical manifold $\mathcal{M}_c$, separating concentrated anti-collapse ($\beta>1$, finite mean, steep power-law envelope) from broad anti-collapse ($\beta<1$, divergent mean, shallow power-law envelope).
To keep the schematic focused on the primary jump-driven route, the $\eta_J=0$ axis (blue dotted) displays the canonical Gaussian-confining pure-diffusion representative of the collapsed regime. For the diffusion-only power-law solution with a saturated far-left drift, see Section~\ref{sec:existence_threshold} and Appendix~\ref{app:limiting_cases_diffusion_only}.
Increasing $\eta_J$ monotonically decreases $\beta$ (Appendix~\ref{app:beta_smoothness}).
The positive-threshold case $\lambda<\kappa/\eta_G$, including its sub-threshold behavior and marginal boundary, is discussed separately in Section~\ref{sec:existence_threshold} and is not shown here.}
\label{fig:phase_diagram}
\end{figure}
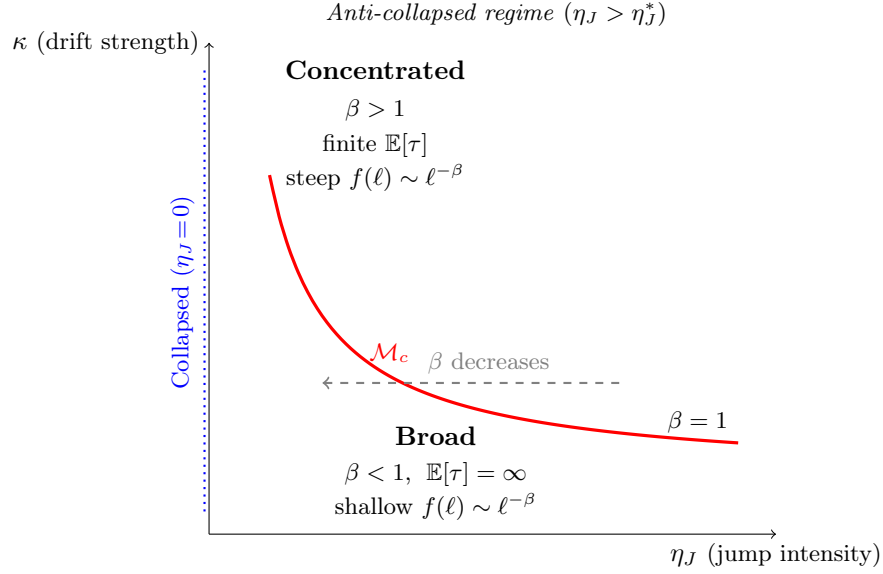

The remainder of this section treats each regime in turn.
Section~\ref{sec:collapsed_regime} describes the collapsed regime and its canonical pure-diffusion representative.
Section~\ref{sec:anticollapsed_regime} gives the characteristic equation for the root-selected branch of anti-collapse ($\eta_J>\eta_J^*$) and describes its concentrated and broad sub-regimes.
The full derivation and existence/uniqueness proof for the spectral exponent are given in Appendix~\ref{app:nonlocal_balance}.
Finally, Section~\ref{sec:log_regular_regime} introduces the log-regular regime, i.e. the limiting regime corresponding to $\beta\downarrow0$.

We use \emph{collapsed regime} and \emph{anti-collapsed regime} as collective names for the corresponding phases of the late-training dynamics.
Within either phase, individual regimes can have different time-scale spectra while sharing the same qualitative tail behavior.
The term \emph{log-regular regime} likewise denotes a family of spectra, included here as a limiting class.

\subsubsection{Collapsed regime}
\label{sec:collapsed_regime}

In the collapsed regime, effective time scales remain concentrated around a characteristic scale rather than developing a power-law tail.
The pure-diffusion boundary $\eta_J=0$ is not classified by the absence of jumps alone: its stationary tail also depends on the far-left geometry of the drift.
The canonical collapsed representative considered here extends the near-bulk linear restoring behavior~\eqref{eq:drift_linearization} globally, giving a pure-diffusion OU process with a Gaussian stationary log-rate density.
More generally, the Gaussian-confining null of Proposition~\ref{prop:gaussian_null} excludes regularly varying time-scale tails whenever the confining potential grows faster than linearly\footnote{Informally, a regularly varying tail is a power law up to a slowly varying correction: it behaves like $x^{-\rho}$ times a factor that changes more slowly than any power of $x$ (for instance a constant or a logarithmic term), with the pure power law as the base case~\cite{feller1971introduction,bingham1989regular,nair2022fundamentals}.}.

If the drift retains the saturated far-left closure~\eqref{eq:drift_saturation} when $\eta_J=0$, pure diffusion gives a power-law solution with exponent $\beta_{\mathrm{diff}}=\kappa/\eta_G$ (Appendix~\ref{app:limiting_cases_diffusion_only}).
Thus, the near-bulk OU linearization and the far-left saturation describe different regions of the drift and may coexist; it is the realized far-left geometry that determines which diffusion-only behavior applies.
For the canonical OU representative, the explicit stationary density and the associated envelope-decay derivation are given in Appendix~\ref{app:exponential_envelope_full_details}.

The collapsed regime is not pathological.
Many well-trained recurrent networks operate in this regime and achieve useful temporal learning over windows set by $\tau_\ast$.
The limitation is that the temporal reach is bounded by the characteristic scale, and the envelope decays too rapidly to support extended multi-scale structure and thus long-range learning.

\subsubsection{Root-selected branch of anti-collapse}
\label{sec:anticollapsed_regime}

We now consider the primary jump-driven route to anti-collapse.
Under the saturated far-left drift closure~\eqref{eq:drift_saturation}, the characteristic equation selects a positive spectral exponent $\beta$ when the jump amplitude $\eta_J$ exceeds the threshold $\eta_J^*$ (Proposition~\ref{lem:Phi_properties}).
The corresponding exponential far-left tail of $\rho_\infty(\zeta)$ produces a power-law right tail in the time-scale distribution.
The threshold vanishes when $\lambda\ge\kappa/\eta_G$, in which case every $\eta_J>0$ admits this branch.

\paragraph{Characteristic equation and spectral exponent.}

The tail structure is determined by a nonlocal balance
between constant restoring drift and stochastic forcing
in the stationary forward equation~\eqref{eq:general_forward_equation}.
In the far-left-tail regime ($\zeta\to-\infty$), substituting
an exponential ansatz $\rho_\infty(\zeta)\sim c\,e^{\beta\zeta}$
into the stationary Fokker--Planck equation~\eqref{eq:general_forward_equation}
reduces the
integro-differential problem to the algebraic characteristic function
\begin{equation}
\label{eq:spectral_characteristic_function}
\Phi(\beta;\omega)
=
\underbrace{\eta_G\beta^2}_{\text{diffusion}}
+
\underbrace{\Psi^*(\beta;\omega)}_{\text{jumps}}
-
\underbrace{\kappa\beta}_{\text{drift}},
\end{equation}
where $\Psi^*$ is the adjoint jump symbol defined in
Eq.~\eqref{eq:appendix_adjoint_jump_symbol}.
Each term encodes one of the three mechanisms competing
in the far tail of the stationary density:
Gaussian diffusion, L\'evy jumps, and deterministic
restoring drift.
Appendix~\ref{app:nonlocal_balance} establishes that $\Phi(\cdot\,;\omega)$
is strictly convex on $[0,\lambda)$ with
$\Phi(0;\omega)=0$,
$\Phi'(0;\omega)<0$,
and $\Phi(\lambda;\omega)$ finite
(for the symmetric tempered L\'evy measure).
A unique strictly positive solution $\beta\in(0,\lambda)$ of the characteristic equation $\Phi(\beta; \omega)=0$ exists if and only if $\Phi(\lambda;\omega)>0$, which gives an explicit threshold $\eta_J^*$ for the jump amplitude (Eq.~\eqref{eq:appendix_threshold_etaJ}).

\paragraph{Emergent power-law tail and time-scale distribution.}

The positive root $\beta$ selected by the characteristic equation is therefore interpreted through the same exponential far-left-tail ansatz used to derive~\eqref{eq:spectral_characteristic_function}.
Within the far-left-tail closure~\eqref{eq:drift_saturation}, that ansatz gives the admissible stationary tail
\begin{equation}
\label{eq:stationary_spectrum_left_tail}
\rho_\infty(\zeta)
\sim
c\,e^{\beta\,\zeta},
\qquad
\zeta\to -\infty.
\end{equation}

Passing from the log-effective decay rate to the effective time scale via $\tau=\bar{\mu}^{-1}=e^{-\zeta}$~\eqref{eq:tau_def}, the density transformation
$p_\infty(\tau) = \rho_\infty(\zeta)/|d\tau/d\zeta|$, combined with~\eqref{eq:stationary_spectrum_left_tail},
converts the exponential left tail in~$\zeta$ into a
power-law right tail~\footnote{Explicitly:
$|d\tau/d\zeta|=|{-e^{-\zeta}}|=\tau$, so
$p_\infty(\tau)=\rho_\infty(-\log\tau)/\tau$.
Substituting~\eqref{eq:stationary_spectrum_left_tail}
gives $\rho_\infty(-\log\tau)\sim c\,e^{-\beta\log\tau}
= c\,\tau^{-\beta}$ as $\tau\to\infty$, hence
$p_\infty(\tau)\sim c\,\tau^{-\beta}/\tau
= c\,\tau^{-1-\beta}$.} for the time-scale distribution:
\begin{equation}
\label{eq:stationary_timescale_power_law}
p_\infty(\tau)
\sim
c\,\tau^{-1-\beta},
\qquad
\tau\to\infty,
\end{equation}
where $c>0$ is the asymptotic tail amplitude.
The restriction $\beta>0$ is essential. At $\beta=0$,
\eqref{eq:stationary_timescale_power_law} would reduce to
$p_\infty(\tau)\propto \tau^{-1}$, which is not integrable on
$[1,\infty)$ and therefore cannot define a stationary time-scale density.

Thus the interior root $\beta$ determines the right-tail exponent of the stationary time-scale density on this branch.

\begin{remark}[Observational scaling window]
\label{rem:scaling_window}
The power-law tail~\eqref{eq:stationary_timescale_power_law} extends to infinity in the model, but in finite networks the observable range of the power-law decay is bounded.
The maximum observable effective time scale is limited by the experimental horizon, so in practice \(\tau_{\max}\) cannot exceed the sequence length scale.
It is also affected by the tempering parameter $\lambda$, which controls how strongly large log-effective decay rate excursions are suppressed in the stochastic generator.
The minimum time scale is set by architectural and numerical constraints.
The power-law tail of $p_\infty(\tau)$ is therefore observable only over a finite window $\tau_{\min}\leq \tau \leq \tau_{\max}$; beyond this window, finite-size corrections dominate.
The implications of this truncation for the macroscopic envelope are analyzed in Section~\ref{sec:envelope_laws_from_tails}.
\end{remark}

\paragraph{Concentrated and broad sub-regimes.}

On this branch, the magnitude of $\beta$ governs the degree of spectral broadening.
Two sub-regimes arise:
\begin{itemize}
\item \textbf{Concentrated anti-collapse ($\beta>1$).}
The mean time scale $\mathbb{E}[\tau]=\int \tau\,p_\infty(\tau)\,d\tau$ is finite. More generally, the moment $\mathbb{E}[\tau^r]$ is finite for $r<\beta$ and diverges for $r\ge\beta$.
The spectrum concentrates around a well-defined typical scale, and the power-law tail, while present, decays steeply and carries relatively little mass.
The power-law tail region is narrow when $\beta\gg 1$ and widens as $\beta$ decreases toward~$1$.

\item \textbf{Broad anti-collapse ($\beta<1$).}
The mean time scale diverges.
The power-law tail is shallow enough that probability mass extends over a broad range of time scales, with no characteristic scale dominating.
The spectrum supports genuinely multi-scale temporal structure.
\end{itemize}

\paragraph{Critical manifold ($\beta=1$).}

On the root-selected branch, the set
\begin{equation}
\label{eq:critical_manifold_definition}
\mathcal{M}_c
=
\{\omega : \beta(\omega) = 1\}
\end{equation}
separates concentrated ($\beta>1$) from broad ($\beta<1$) anti-collapse in control-parameter space.
When nonempty, it is a smooth hypersurface, a submanifold of codimension one (Appendix~\ref{app:beta_smoothness}).
We call it a critical manifold because, for the pure power-law tail considered here, the mean time scale is finite for $\beta>1$ and infinite for $\beta\le1$, including on $\mathcal{M}_c$ itself.
The power-law tail class persists across the boundary, and its exponent varies smoothly through $\beta=1$.
Here, ``critical'' refers to the divergence of a distributional moment, not to a bifurcation of the dynamics.

\subsubsection{Log-regular regime ($\beta\downarrow 0$ limit)}
\label{sec:log_regular_regime}

Beyond the critical manifold $\beta=1$, the anti-collapsed taxonomy has a second boundary at $\beta\downarrow 0$.
By Proposition~\ref{lem:Phi_properties}, whenever the spectral exponent exists ($\eta_J>\eta_J^*$) it satisfies $\beta\in(0,\lambda)$ strictly, so $\beta=0$ is never reached at any finite point in control-parameter space.
The boundary is approached but not attained within the symmetric tempered L\'evy framework studied here.

We refer to the limiting behavior at this boundary as the log-regular regime. The explicit tail model that characterizes this regime and its associated envelope decay are introduced in Section~\ref{sec:envelope_laws_from_tails}.
Because this regime sits at a limit not realized by the stochastic model, we include it for taxonomic completeness only: its tail and envelope companions complete the spectrum--envelope correspondence, but no stable phase of the present generator populates it.

\subsection{Jump-amplitude threshold and regime accessibility}
\label{sec:existence_threshold}

\paragraph{Existence threshold for the spectral exponent.}
An interior root $\beta\in(0,\lambda)$ exists when the jump amplitude $\eta_J$ exceeds the threshold $\eta_J^*$ given in Proposition~\ref{lem:Phi_properties}.
This threshold is zero when $\lambda\ge\kappa/\eta_G$ and positive when $\lambda<\kappa/\eta_G$.

At $\eta_J=0$, pure diffusion with a saturated far-left drift gives a power-law tail with exponent $\beta_{\mathrm{diff}}=\kappa/\eta_G$ set by the drift--diffusion balance alone.
The tempering rate $\lambda$ belongs to the jump component of the mixed model. When $\lambda\ge\kappa/\eta_G$, the threshold vanishes and the interior root of the mixed model approaches $\beta_{\mathrm{diff}}$ as $\eta_J\downarrow0$.
Further details are given in Appendix~\ref{app:limiting_cases_diffusion_only}.

When $\lambda<\kappa/\eta_G$, the threshold is strictly positive, $\eta_J^*>0$.
For $0<\eta_J<\eta_J^*$, the characteristic equation has no positive root in $(0,\lambda]$.
The nonlocal jump operator carries probability from the populated bulk directly into the far-left tail, at a rate set by the tempering instead of a characteristic root.
In this \emph{tempering-controlled sub-threshold regime}, the stationary time-scale density $p_\infty(\tau)$ scales as $\tau^{-1-\lambda}(\log\tau)^{-1-\alpha_{\mathrm{jump}}}$. The corresponding fraction of time scales exceeding $T$ scales as $T^{-\lambda}(\log T)^{-1-\alpha_{\mathrm{jump}}}$ (Corollary~\ref{cor:subthreshold_density} and Proposition~\ref{prop:subthreshold_survival}, Appendix~\ref{app:subthreshold_tail}).
The root-selected and tempering-controlled branches thus belong to the same anti-collapsed phase. On both branches the stationary time-scale density is regularly varying, with index $-1-\lambda$ strictly below the positive threshold and $-1-\beta$ above it.
In the sub-threshold regime the mean time scale is finite for $\lambda\ge1$ and diverges for $0<\lambda<1$. At $\lambda=1$, the logarithmic correction makes the tail integral for $\mathbb{E}[\tau]$ converge.

\paragraph{Tempering and regime accessibility.}
On the root-selected branch, the spectral exponent satisfies $\beta\in(0,\lambda)$.
Consequently, concentrated anti-collapse ($\beta>1$) and the critical manifold ($\beta=1$) require $\lambda>1$.
When $\lambda\le1$, every exponent on this branch satisfies $\beta<1$, so its time-scale spectrum has a divergent mean and belongs to broad anti-collapse.

\subsection{Boundary structure and monotonicity}
\label{sec:boundary_monotonicity}

\paragraph{Marginal existence boundary.}
When $\eta_J^*>0$, the root-selected exponent approaches $\lambda$ as $\eta_J\downarrow\eta_J^*$ from above. At equality, $\beta=\lambda$ is the unique positive root of the characteristic equation.
At this endpoint, the candidate stationary log-rate tail and the contribution from jumps arriving directly from the bulk have the same exponential decay rate. Our analysis therefore leaves open whether the stationary time-scale tail is regularly varying at the threshold and how the corresponding envelope decays.

In the opposite limit $\eta_J\to\infty$, the root moves toward zero, $\beta\downarrow0$, and the tail becomes shallower.
Empirically, the boundary may be difficult to resolve when $\lambda$ is large, because the corresponding near-boundary power laws are very steep over finite observational windows.

\paragraph{Monotonicity in the jump amplitude.}
For $\eta_J>\eta_J^*$, as $\eta_J$ increases (with other control parameters fixed), $\Psi^*(\beta;\omega)$ increases pointwise for each $0<\beta<\lambda$. By the strict convexity of $\Phi$ established in Proposition~\ref{lem:Phi_properties}, the root $\beta$ therefore decreases monotonically with $\eta_J$ (Appendix~\ref{app:beta_smoothness} gives the derivative). A stronger jump forcing (larger $\eta_J$) thus lowers $\beta$ and broadens the time-scale spectrum.

\section{Canonical envelope decay laws from tail spectra}
\label{sec:envelope_laws_from_tails}

The phase taxonomy of Section~\ref{sec:phase_structure} characterizes the stationary time-scale spectrum through the tail behavior of $p_\infty(\tau)$.
The macroscopic envelope $f(\ell)$ introduced in Section~\ref{sec:micro-macro} aggregates the contribution of all time scales and therefore inherits its asymptotic decay from $p_\infty(\tau)$.
The goal of this section is to make this correspondence mathematically explicit and derive the resulting envelope scaling laws from specific tail models of $p_\infty(\tau)$.
In this framework, the unweighted envelope is regularly varying on both anti-collapsed branches, whereas collapsed spectra remain outside that scaling class. The canonical collapsed representative is effectively exponential over its median-dominated range and has a non-power-law log-normal asymptotic at larger lags; the log-regular ansatz produces a logarithmic envelope.
These canonical forms correspond to the canonical envelope decay laws identified in the learnability-window analysis~\citep{livi2026learnability}.

\subsection{From effective learning rates to the Laplace representation}
\label{sec:micro-macro}

\subsubsection{Mixture-of-exponentials representation and macroscopic envelope}

The macroscopic envelope is defined as
\begin{equation}
f(\ell) = \|\mu_{t,\ell}\|_1
=
\sum_{q=1}^H
\left|\mu^{(q)}_{t,\ell}\right|
=
\sum_{q=1}^H
\Lambda^{(q)}_{r,\ell} \left|\Gamma^{(q)}_{t,\ell}\right|.
\end{equation}
In the late quasi-stationary regime and for the class of gated
RNNs considered here (Appendix~\ref{app:gated_rnns}),
the transport factor of each individual neuron decays
exponentially at long lags with a neuron-specific effective
time scale $\tau_q$~\eqref{eq:tau_def}.
More precisely, for sufficiently large $\ell$ the per-neuron
contribution satisfies
\begin{equation}
\label{eq:mixture_of_exponentials_prefactor}
|\mu^{(q)}_{t,\ell}|
\approx
\Lambda^{(q)}_{r,\ell}\,
\exp(-\ell/\tau_q)
\times \tilde{L}_q(\ell),
\end{equation}
where $\tilde{L}_q(\ell)$ captures sub-exponential corrections
originating from gate sensitivities and higher-order mixing
terms, and $\Lambda^{(q)}_{r,\ell}>0$ is the bounded adaptive base rate.
The base rate modulates the amplitude of each neuron's
contribution but does not modify the dominant exponential
decay rate $\tau_q^{-1}$, which is determined entirely by the
transport factor~$\Gamma^{(q)}_{t,\ell}$.

Retaining only the dominant exponential component of each
per-neuron contribution and summing over the population,
$f(\ell)$ is approximated as follows:
\begin{equation}
\label{eq:mixture_of_exponentials}
f(\ell) = \|\mu_{t,\ell}\|_1
\;\approx\;
\sum_{q=1}^H
\Lambda^{(q)}_{r,\ell}\,
\exp(-\ell/\tau_q).
\end{equation}
Hence, the macroscopic envelope $f(\ell)$ is a mixture of exponentials, with
architecture-dependent time scales $\{\tau_q\}$ and optimizer-dependent amplitudes $\{\Lambda^{(q)}_{r,\ell}\}$.
Empirical results supporting this representation are provided in Appendix~\ref{app:mixture_exp_validation}.

\subsubsection{Large-network-width representation}

To remove the trivial linear scaling with network width from $f(\ell)$, we consider the \emph{intensive} macroscopic envelope
\begin{equation}
\label{eq:intensive_envelope}
f(\ell)
\;=\;
\frac{1}{H}\,\|\mu_{t,\ell}\|_1.
\end{equation}
From now on, the notation $f(\ell)$ refers to Eq.~\ref{eq:intensive_envelope} with the approximation in Eq.~\ref{eq:mixture_of_exponentials}.

\paragraph{From finite mixture to continuous integral.}
Taking the large-width limit ($H\rightarrow\infty$) for $f(\ell)$ \eqref{eq:intensive_envelope}, the population-level concentration (Assumption~\ref{ass:mean_field}) applies to the joint empirical distribution $(\tau_q,\Lambda^{(q)}_{r,\ell})$.
Hence, this limit gives
\begin{equation*}
f(\ell)
\;=\;
\mathbb{E}\!\left[\Lambda^{(q)}_{r,\ell}\,e^{-\ell/\tau_q}\right].
\end{equation*}
Because the exponential kernel $e^{-\ell/\tau}$ depends only on the time scale $\tau_q$, the adaptive base rate enters only through its conditional mean at each time scale, while the $\tau$-marginal $p_\infty(\tau)$ controls the geometry of the envelope in the time-scale coordinate.
The resulting \emph{$\Lambda$-weighted Laplace representation} is
\begin{equation}
\label{eq:envelope_integral_weighted}
f(\ell)
\;=\;
\int_0^\infty
\bar{\Lambda}(\tau)\,
e^{-\ell/\tau}\,
p_\infty(\tau)\,d\tau,
\end{equation}
where
\begin{equation}
\label{eq:lambda_bar_def}
\bar{\Lambda}(\tau)
=
\mathbb{E}\!\left[\Lambda^{(q)}_{r,\ell}\,\big|\,\tau_q=\tau\right]
\end{equation}
is the conditional mean adaptive base rate at time
scale~$\tau$.
Intuitively, any variation of the optimizer-induced
preconditioner across neurons sharing a common time scale
is invisible to the envelope: the Laplace kernel cannot resolve
$\Lambda$-fluctuations at fixed~$\tau$, so only their
conditional mean~$\bar{\Lambda}(\tau)$ survives the
averaging.
In practice, $\bar{\Lambda}(\tau)$ need not be computed explicitly: because the Rayleigh projection is uniformly
bounded, $\lambda_{\min}\le\Lambda^{(q)}_{r,\ell}\le\lambda_{\max}$ for all $q$ and $\ell$~\citep{livi2026learnability},
the function $\bar{\Lambda}(\tau)$ inherits the same bounds (hence the $\ell$-dependence is suppressed), and the analysis proceeds via the unweighted envelope $f_0(\ell)$ introduced below.

\paragraph{Sandwich bound and optimizer independence.}
The uniform bound~\citep{livi2026learnability} on $\bar{\Lambda}(\tau)$ implies that the weighted envelope~\eqref{eq:envelope_integral_weighted} is sandwiched between two multiples of the \emph{unweighted envelope}
\begin{equation}
\label{eq:envelope_integral}
f_0(\ell)
=
\int_0^\infty
e^{-\ell/\tau}\,
p_\infty(\tau)\,d\tau,
\end{equation}
namely,
\begin{equation}
\label{eq:sandwich}
\lambda_{\min}\,f_0(\ell)
\;\le\;
f(\ell)
\;\le\;
\lambda_{\max}\,f_0(\ell).
\end{equation}
The sandwich~\eqref{eq:sandwich} ensures that $f(\ell)$ and $f_0(\ell)$ share the same asymptotic scaling class: if $f_0(\ell)\sim c\,g(\ell)$ then $f(\ell)$ is bounded above and below by positive multiples of $g(\ell)$ for all large~$\ell$.
The scaling class is therefore optimizer-independent; only the prefactor absorbs the bounded optimizer-dependent factor.

This passage from the finite mixture to the continuous integral is checked numerically in Appendix~\ref{app:mixture_and_large_width_validation}, where both the per-neuron exponential approximation (Appendix~\ref{app:mixture_exp_validation}) and the population-level concentration (Appendix~\ref{app:width_validation}) are validated.

\subsubsection{Standard Laplace form}

The asymptotic results below are stated for the unweighted envelope $f_0(\ell)$.
To apply classical Tauberian theorems~\citep{feller1971introduction,bingham1989regular,nair2022fundamentals}, we rewrite~\eqref{eq:envelope_integral} in the standard Laplace form.
The natural variable is the asymptotic decay rate~$\bar{\mu}=\tau^{-1}$, given that $e^{-\ell/\tau}=e^{-\ell\bar{\mu}}$ is already a Laplace kernel in the rate.
Rates and time scales are dual descriptions of the same quantity~\eqref{eq:tau_def}; the analysis moves to time scales~$\tau$ when interpreting the spectrum, and returns to rates~$\bar{\mu}$ when the Laplace structure is needed.
Since $|d\tau/d\bar{\mu}|=\bar{\mu}^{-2}$, conservation of probability mass gives the rate density
\begin{equation}
\label{eq:change_tau_u}
p_{\bar{\mu}}(\bar{\mu})=\bar{\mu}^{-2}\,p_\infty(1/\bar{\mu}),
\end{equation}
and~\eqref{eq:envelope_integral} becomes
\begin{equation}
\label{eq:envelope_Laplace_u}
f_0(\ell)=\int_0^\infty e^{-\ell\bar{\mu}}\,p_{\bar{\mu}}(\bar{\mu})\,d\bar{\mu},
\end{equation}
which is the Laplace transform of $p_{\bar{\mu}}$.
Consequently, the asymptotic decay of $f_0(\ell)$ for $\ell\to\infty$ is determined by the behavior of $p_{\bar{\mu}}(\bar{\mu})$ near $\bar{\mu}=0$; equivalently via~\eqref{eq:tau_def} by the behavior of $p_\infty(\tau)$ as $\tau\to\infty$.
The classical Tauberian theorem for Laplace transforms \citep{feller1971introduction} makes this correspondence precise.
Appendix~\ref{app:envelope_tail_calculations} contains the full derivations for all envelope decay laws discussed below, together with the tail-class definitions (slowly varying, regularly varying, log-regularly varying) used in the derivations.

\subsection{Exponential envelope from the canonical collapsed representative}
\label{sec:exponential_envelope}

The canonical collapsed representative does not have the regularly varying right tail required for a power-law envelope.
The clean representative considered here is the pure-diffusion OU case, whose envelope is effectively exponential throughout the median-dominated lag range and crosses to a non-power-law, log-normal-type asymptotic at larger lags (Remark~\ref{rem:collapsed_range}).
In this representative regime, the stationary density $\rho_\infty(\zeta)$ is Gaussian (centered at~$\zeta^*$ with variance $\eta_G/\gamma$, where $\gamma$ is the local restoring rate of the drift near the bulk equilibrium~\eqref{eq:drift_linearization}), and therefore the corresponding time-scale distribution $p_\infty(\tau)$ is log-normal. It is concentrated around the typical scale $\tau_\ast=e^{-\mathbb{E}[\zeta]}$, has all moments finite, and has a right tail that decays faster than any polynomial.

The collapsed configurations examined in Section~\ref{sec:empirical} exhibit concentrated spectra and effectively exponential envelopes over the measured lag range.
This supports the canonical OU representative as a useful operational description of the collapsed behavior observed in finite-horizon experiments.

When the log-normal variance $\sigma^2=\eta_G/\gamma$ is small, the integral representation~\eqref{eq:envelope_integral} is dominated by values of $\tau$ in a neighborhood of the median $\tau_\ast$ for all lags in the median-dominated regime $\ell\sigma^2/\tau_\ast\ll 1$.
Expanding $1/\tau$ about $\tau_\ast$ then gives
\begin{equation}
\label{eq:collapsed_decay}
f_0(\ell)\sim e^{-\ell/\tau_\ast}.
\end{equation}
This approximation covers an operationally wide range of lags when the local restoring rate is strong relative to diffusive noise, i.e. when $\sigma^2 = \eta_G/\gamma$ is small (Remark~\ref{rem:collapsed_range}). Thus, the envelope decays effectively as a single exponential throughout the median-dominated lag range $\ell\sigma^2/\tau_\ast\ll 1$. This is an operational approximation, not a strict $\ell\to\infty$ equivalence for the log-normal mixture (see Remark~\ref{rem:collapsed_range} for the crossover behavior at longer lags).
For integer lags, writing $a=e^{-1/\tau_\ast}\in(0,1)$ gives
$e^{-\ell/\tau_\ast}=a^\ell$. Thus, over the median-dominated lag
range, the collapsed envelope has the geometric form
$f_0(\ell)\sim a^\ell$.\footnote{The learnability analysis
\citep{livi2026learnability} denotes this geometric base by $\lambda$.
Here we use $a$ to distinguish it from the tempering rate $\lambda$.}

\subsection{Power-law envelopes in the anti-collapsed phase}
\label{sec:algebraic_envelope}

When the jump amplitude exceeds $\eta_J^*$, the stationary time-scale distribution has a regularly varying right tail
\begin{equation}
\label{eq:power_tail_tau}
p_\infty(\tau)
\sim
c\,\tau^{-1-\beta},
\qquad
\beta>0,
\
\tau\to\infty,
\end{equation}
as already established in Eq.~\ref{eq:stationary_timescale_power_law}, Section~\ref{sec:anticollapsed_regime}.
We retain $c$ explicitly because the Tauberian correspondence preserves it.
Using the change of variable \eqref{eq:change_tau_u} gives the density of the small-rate behavior:
\begin{equation}
\label{eq:small_u_density}
p_{\bar{\mu}}(\bar{\mu})
\sim
c\,\bar{\mu}^{\beta-1},
\qquad
\bar{\mu}\to0 .
\end{equation}
Since the Laplace transform~\eqref{eq:envelope_Laplace_u} concentrates near $\bar{\mu}=0$ when $\ell\to\infty$, the asymptotic decay of $f_0(\ell)$ is determined by the small-$\bar{\mu}$ behavior of $p_{\bar{\mu}}$.
Applying case~(i) of Proposition~\ref{prop:tauberian} to the asymptotic form~\eqref{eq:small_u_density} yields
\begin{equation}
\label{eq:power_decay}
f_0(\ell)
\sim
c\,\Gamma(\beta)\,\ell^{-\beta},
\qquad \beta>0, \ \ell\to\infty,
\end{equation}
where $\Gamma(\beta)$ denotes the Gamma function~\eqref{eq:app_gamma_integral}.
Since $\Gamma(\beta)$ is a finite positive constant for each
$\beta>0$, it enters as a prefactor that does not affect the
polynomial scaling class.
On the root-selected branch, the power $\tau^{-1-\beta}$ in the stationary time-scale density thus gives the power $\ell^{-\beta}$ in the unweighted envelope.

For $0<\eta_J<\eta_J^*$, tempering instead sets the decay of the unweighted envelope $f_0(\ell)$, which scales as $\ell^{-\lambda}(\log\ell)^{-1-\alpha_{\mathrm{jump}}}$ (Corollary~\ref{cor:subthreshold_envelope}).

\subsection{Logarithmic envelope from log-regular regime (ansatz tail, $\beta\downarrow 0$)}
\label{sec:logarithmic_envelope}

A log-regularly varying tail of the form
\begin{equation}
\label{eq:app_logregular_tail}
p_\infty(\tau)
\sim
\frac{1}{\tau(\log\tau)^{1+\vartheta}},
\qquad \tau\to\infty,\ \vartheta>0,
\end{equation}
is the borderline integrable regularly varying tail: the logarithmic factor makes the otherwise nonnormalizable $\tau^{-1}$ density admissible and produces a sub-polynomial Laplace envelope~\citep{feller1971introduction}.

Case~(ii) of Proposition~\ref{prop:tauberian}, applied to the log-regularly varying tail~\eqref{eq:app_logregular_tail}, then gives (see Appendix~\ref{app:log_envelope_derivation_full} for the full calculation)
\begin{equation}
\label{eq:critical_decay}
f_0(\ell)\sim \frac{1}{\vartheta}\,(\log \ell)^{-\vartheta},
\qquad \ell\to\infty.
\end{equation}

Unlike the exponential and power-law cases, the log-regularly varying tail~\eqref{eq:app_logregular_tail} is not derived from the proposed stochastic model.
It is the ansatz at the $\beta\downarrow 0$ boundary, motivated by the observation that, as $\beta\downarrow 0$, the power-law exponent vanishes and the tail approaches $\tau^{-1}$ modulated by a slowly varying factor; the form~\eqref{eq:app_logregular_tail} is the canonical representative of the log-regularly varying class~\citep{feller1971introduction}.
Within the symmetric tempered L\'evy framework, whenever the spectral exponent exists ($\eta_J>\eta_J^*$) it satisfies $\beta\in(0,\lambda)$ strictly (Proposition~\ref{lem:Phi_properties}), so the model never reaches this class.

\subsection{Observable envelopes, phase boundaries, and the limits of spectral broadening}
\label{sec:observable_envelopes}

The phase structure, stationary tails, and canonical envelope behaviors are collected in Table~\ref{tab:envelope_classes}.
\begin{table}[tp]
\centering
\begin{tabular}{@{}llll@{}}
\toprule
\textbf{Regime} & \textbf{Condition} & \textbf{Tail $p_\infty(\tau)$} & \textbf{Canonical envelope $f_0(\ell)$} \\
\midrule
\shortstack[l]{Anti-collapsed\\(root-selected)} & $\eta_J>\eta_J^*$ & $\tau^{-1-\beta}$ & $\ell^{-\beta}$ \\
\quad Broad & $\beta<1$ & $\mathbb{E}[\tau]=\infty$ & \\
\quad Concentrated & $\beta>1$ & finite $\mathbb{E}[\tau]$ & \\[3pt]
\shortstack[l]{Anti-collapsed\\(sub-threshold)} & $0<\eta_J<\eta_J^*$ & $\tau^{-1-\lambda}(\log\tau)^{-1-\alpha_{\mathrm{jump}}}$ & $\ell^{-\lambda}(\log\ell)^{-1-\alpha_{\mathrm{jump}}}$ \\[3pt]
Collapsed (OU) & \shortstack[l]{$\eta_J=0$\\globally linear drift} & log-normal & $e^{-\ell/\tau_\ast}$ \\[3pt]
\cmidrule(lr){1-4}
Log-regular (ansatz) & $\eta_J>\eta_J^*$, $\beta\downarrow 0$ & $ \tau^{-1}(\log\tau)^{-1-\vartheta} $ & $1/\vartheta(\log\ell)^{-\vartheta}$ \\
\bottomrule
\end{tabular}
\caption{Above the jump-amplitude threshold, the interior root $\beta$ sets the decay exponent.
On this branch, the mean time scale is finite for $\beta>1$ and infinite for $\beta\le1$. The envelope exponent varies continuously through $\beta=1$, where the unweighted envelope decays in proportion to $1/\ell$.
For $0<\eta_J<\eta_J^*$, the spectrum is still anti-collapsed. The tempering rate $\lambda$ sets the leading power, and the density and unweighted envelope carry the same logarithmic correction (Corollaries~\ref{cor:subthreshold_density} and~\ref{cor:subthreshold_envelope}).
The collapsed row is the canonical OU representative. It assumes globally linear restoring drift, and its displayed exponential envelope applies over the median-dominated lag range, not as a strict large-lag asymptotic.
At $\eta_J=\eta_J^*>0$, it remains open whether the stationary time-scale tail is regularly varying and how the unweighted envelope decays.
The log-regular regime (bottom) is included via an ansatz tail at the $\beta\downarrow 0$ limit and is not attainable within the proposed stochastic model.}
\label{tab:envelope_classes}
\end{table}

\paragraph{Finite-network truncation.}
The canonical envelope laws derived above describe idealized asymptotic behavior.
In a finite network, the observable time-scale spectrum occupies a finite window $[\tau_{\min},\tau_{\max}]$. The lower endpoint is set by architectural and numerical constraints, while the upper endpoint is limited by sequence length, hidden dimension, and numerical precision.
On the root-selected branch ($\eta_J>\eta_J^*$), we therefore model the finite-window spectrum by the truncated power-law density
\begin{equation}
\label{eq:truncated_timescale_density}
p_\infty^{\mathrm{trunc}}(\tau)
=
c\,\tau^{-1-\beta}\,
\mathbf 1_{\{\tau_{\min}\le\tau\le\tau_{\max}\}},
\end{equation}
where $c>0$ normalizes the density on this interval.
The corresponding unweighted envelope is
\begin{equation}
\label{eq:truncated_envelope}
f_0(\ell)
=
\int_{\tau_{\min}}^{\tau_{\max}}
e^{-\ell/\tau}\,c\,\tau^{-1-\beta}\,d\tau .
\end{equation}
This integral exhibits two asymptotic regimes.

\emph{Power-law regime ($\tau_{\min}\ll\ell\ll\tau_{\max}$).}
Extending the integral in~\eqref{eq:truncated_envelope} to $(0,\infty)$ gives
\begin{equation}
\int_0^\infty
e^{-\ell/\tau}\,c\,\tau^{-1-\beta}\,d\tau
=
c\,\Gamma(\beta)\,\ell^{-\beta}.
\end{equation}
The added contribution below $\tau_{\min}$ is exponentially small when $\ell\gg\tau_{\min}$, with order $\tau_{\min}^{1-\beta}\ell^{-1}e^{-\ell/\tau_{\min}}$.
The added contribution above $\tau_{\max}$ satisfies
\begin{equation}
\int_{\tau_{\max}}^\infty
e^{-\ell/\tau}\,c\,\tau^{-1-\beta}\,d\tau
\le
\frac{c}{\beta\,\tau_{\max}^{\beta}},
\end{equation}
whose size relative to $\ell^{-\beta}$ is of order $(\ell/\tau_{\max})^\beta$.
Both endpoint corrections are therefore negligible when $\tau_{\min}\ll\ell\ll\tau_{\max}$, and
\begin{equation}
f_0(\ell)
\sim
c\,\Gamma(\beta)\,\ell^{-\beta},
\qquad
\tau_{\min}\ll\ell\ll\tau_{\max}.
\end{equation}

\emph{Cutoff regime ($\ell\gg\tau_{\max}$).}
For large~$\ell$, the kernel $e^{-\ell/\tau}$ suppresses time scales below the upper endpoint, so the integral is dominated by $\tau=\tau_{\max}$.
Applying the boundary-maximum form of Laplace's method~\citep{de2010asymptotic} to~\eqref{eq:truncated_envelope}, with $h(\tau)=-1/\tau$ maximized at~$\tau_{\max}$, gives
\begin{equation}
f_0(\ell)
\sim
\frac{c\,\tau_{\max}^{-1-\beta}}
     {\ell\,h'(\tau_{\max})}
e^{\ell h(\tau_{\max})}
=
c\,\tau_{\max}^{1-\beta}\,
\ell^{-1}e^{-\ell/\tau_{\max}},
\qquad
h'(\tau_{\max})=\tau_{\max}^{-2}.
\end{equation}
Thus, beyond the largest supported time scale, the envelope is exponentially suppressed, with an algebraic prefactor.

The crossover occurs at a lag of order $\ell^\ast\approx\tau_{\max}$.
For $\tau_{\min}\ll\ell\ll\tau_{\max}$, the finite endpoints do not alter the power-law scaling.
For $\ell\gg\tau_{\max}$, the kernel concentrates near $\tau_{\max}$ and the envelope crosses to exponential cutoff behavior.

A convenient phenomenological fitting form that captures this crossover is
\begin{equation}
\label{eq:tempered_interpolant}
f_{\mathrm{int}}(\ell) = A\,\ell^{-\beta}e^{-\ell/\tau_{\max}},
\end{equation}
where $A$ absorbs the tail amplitude and the bounded optimizer-dependent factor from~\eqref{eq:sandwich}.
This interpolant is not derived from \eqref{eq:truncated_envelope}: it reproduces the power-law regime and the exponential cutoff rate, but not the exact algebraic prefactor $\ell^{-1}$ of the endpoint asymptotic.

\paragraph{Multimodality and phase classification.}
In finite networks, the stationary spectrum $p_\infty(\tau)$ may be
multimodal due to neurons specializing into different temporal roles.
However, since $f(\ell)$ is a Laplace--type transform of $p_\infty(\tau)$,
its large-lag asymptotics are determined solely by the
heaviest tail of the distribution.
Multimodality therefore affects only short-lag prefactors
and does not change the asymptotic envelope class.

\paragraph{Why sub-polynomial envelopes might be unreachable.}
The logarithmic envelope $f_0(\ell)\sim(\log\ell)^{-\vartheta}$ (Table~\ref{tab:envelope_classes}) would require $\beta\downarrow 0$, but whenever the spectral exponent exists ($\eta_J>\eta_J^*$), Proposition~\ref{lem:Phi_properties} constrains $\beta\in(0,\lambda)$ strictly.

Approaching $\beta=0$ on the root-selected branch requires a limiting choice of the control parameters, such as $\kappa\to0$ or $\eta_J\to\infty$.
Finite networks impose an additional restriction.
Their observable time-scale spectrum is supported only on $[\tau_{\min},\tau_{\max}]$, and as $\beta\downarrow0$ the normalized truncated power law approaches $p_\infty(\tau)\propto1/\tau$ on this fixed interval.
This corresponds to approximately uniform probability mass per logarithmic time-scale interval, not to an asymptotic tail extending to arbitrarily large~$\tau$.
Moreover, the endpoint calculation above shows that for $\ell\gg\tau_{\max}$ the envelope necessarily crosses to exponential cutoff behavior.
A finite system may therefore display an increasingly shallow power law over its observable scaling window, but it cannot realize a logarithmic envelope as a strict large-lag asymptotic.

A further difficulty is that, as $\beta$ approaches zero, the route to a stable log-regular phase becomes increasingly poorly conditioned: ordinary finite changes in the effective control parameters (results not shown) produce little reliable leverage on the exponent itself, while small changes in the far-tail balance can visibly reshape the stationary spectrum. Thus, the deep broad-anti-collapse boundary is not simply a matter of tuning $\beta$ smaller: within the present stochastic model, it is an operationally fragile limit, hard to reach reproducibly and hard to stabilize once reached.
We offer this as a physics intuition for why the $\beta\downarrow 0$ edge may be dynamically inaccessible within the present stochastic model, not as an impossibility result for the log-regular regime as such.

These observations suggest that the root-selected branch robustly produces a power-law envelope with exponent $\beta>0$ and that the log-regular regime is unlikely to emerge as a stable dynamical phase here.
Whether other admissible mesoscopic models, with modified drift saturation, different jump structure, or other ingredients, could yield genuinely accessible log-regular phases is an open question we do not address here.

\section{Empirical validation}
\label{sec:empirical}

In the proposed stochastic model, the primary jump-driven route to anti-collapse combines two conditions: a sufficiently heavy-tailed stochastic forcing channel and a positive, asymptotically flat far-left restoring drift, both realizable by the finite architecture--optimizer pair.
Two experiments probe this claim from opposite directions by considering the architectures described in Appendix~\ref{app:gated_rnns}.

The first experiment in Section~\ref{sec:exp_negative_control_main} is the falsification arm.
It uses a structurally constrained pair (ConstGate, whose gates are frozen at initialization) to test whether anti-collapse still fails when the heavy-tailed forcing channel is supplied externally.
This is the intended dissociation process in which forcing is artificially present but the capacity is structurally absent: forcing alone should not be sufficient if the architecture--optimizer pair cannot access the regime.
The second experiment in Section~\ref{sec:exp_access_route} is the verification arm.
Using the minimal trainable-gate architecture that can represent per-unit time scales (DiagGate), with SharedGate kept as a reference, it tests whether the proposed access route is actually experimentally realized. More precisely, the experiment tests whether heavy-tailed forcing co-occurs with a populated far-left log-rate spectrum, leading to a broad time-scale spectrum, a positive finite-window saturating restoring drift, and a power-law envelope (with an expected exponential cut-off).

Finally, we note that the experimental configurations below are the largest systematic multi-seed, multi-checkpoint runs we could carry out with the available compute.
The code is available in Appendix~\ref{app:code} for full replication.

\subsection{Task and experimental setup}
\label{sec:empirical_shared_setup}

\paragraph{Task.}
All experiments train recurrent networks on a synthetic long-memory
regression task with Gaussian input $x_t\sim\mathcal{N}(0,I_D)$ and a
target built from delayed one-dimensional projections,
\begin{equation}
\label{eq:long_memory_task}
y_t = \sum_{k=1}^{K} c_k\,u^\top x_{t-\ell_k} + \varepsilon_t,
\qquad \varepsilon_t\sim\mathcal{N}(0,\sigma_\varepsilon^2),
\end{equation}
with $u$ a fixed unit direction and a geometric coefficient schedule $c_k=c_0 r^{k}$.
The reported experiments use the heavy-tailed-lag variant, in which the $K$ target lags are drawn once from a truncated Pareto law $p(\ell)\;\propto\;\ell^{-(\alpha_{\mathrm{data}}+1)}, \ell\in[\ell_{\min},\ell_{\max}]$, (integer-rounded and de-duplicated) and held fixed across all sequences, batches, and seeds, so smaller $\alpha_{\mathrm{data}}$ places more mass at long lags.
The full-scale configuration on which all empirical claims rest is $H=512$, $T=1280$, $D=16$, $\alpha_{\mathrm{data}}=0.6$, $K=16$, $[\ell_{\min},\ell_{\max}]=[8,640]$, $c_0=0.6$, $r=0.85$, and $\sigma_\varepsilon=0.3$.

\paragraph{Training.}
Training uses AdamW~\citep{loshchilov2019decoupled} (learning rate $10^{-3}$, weight decay $10^{-4}$, default moments; global $\ell_2$ gradient clipping at $1.0$). Each experiment uses the ten random seeds $\{47,83,12,69,31,104,218,337,451,592\}$, a training set of $N_{\mathrm{train}}=8000$ and a diagnostic set of $N_{\mathrm{diag}}=6000$ sequences independently drawn from the same task instance and fixed across seeds, mini-batch size used for training of $380$ and for diagnostic $256$, and dense checkpointing every $40$ epochs.

\paragraph{Forcing in the two experiments.}
Large log-rate fluctuations can have several sources during training. To model their combined effect, the jump term in~\eqref{eq:zeta_SDE_mixed} uses the L\'evy measure of a symmetric $\alpha$-stable process with an exponential cutoff on large jumps~\eqref{eq:tempered_levy_measure}.
This modeling choice motivates using symmetric $\alpha$-stable laws as the experimental reference for the frequency of large fluctuations over the observed range.
In the frozen-gate intervention, we draw perturbations from this family, smoothly taper the largest draws, and add them to the parameter updates of slow neurons.
For naturally occurring fluctuations, we use this family to calibrate the effective tail index and compare the observed extremes with a Gaussian reference matched in sample size.
Other models of heavy-tailed forcing could likewise populate long time scales through rare jumps, although the stationary tail and the existence of a comparable threshold would have to be established for each jump law.

\subsection{Structural negative control via a frozen-gate architecture}
\label{sec:exp_negative_control_main}

Whether a trained architecture--optimizer pair realizes the anti-collapsed regime
depends not only on the route conditions established in this paper, but also on the
pair's capacity to populate and maintain a broad time-scale spectrum.
ConstGate lacks that capacity by construction: its gates are frozen at
initialization, so trainable time scales are removed and the achievable
spectrum is anchored to the fixed-gate skeleton, with only first- and
higher-order learning-rate corrections available to create spectral heterogeneity.
The experiment asks whether this severely constrained pair fails to enter (and sustain) the anti-collapsed regime even when heavy-tailed forcing is supplied externally in the relevant update channel.

We run two matched paths on the task.
\emph{Path A} is the intervention variant: after each optimizer step, soft-tapered symmetric $\alpha$-stable forcing (configured here with $\alpha_{\mathrm{inj}}=1.5$) is added to the post-Adam update of the slow neurons (those affecting the far-left tail).
The forward pass and the loss gradient are left untouched.
The perturbation enters only the training-time update of the slow-mode parameters: the rows, indexed by the slow neurons, of the trainable weight matrices whose leading dimension runs over hidden neurons.
The slow neurons are the hidden neurons in the far-left log-rate slice $\zeta_q(t)\le\zeta_{q_{\mathrm{low}}}$ (equivalently, the neurons with the largest time scales), and these rows are the channel through which the effective-rate process $\zeta_q$ evolves.
We certify the intervention by applying the calibration-anchored tail estimator to the realized post-update increments of these slow rows.
The primary statistic is the moment estimator of the extreme-value index of \citet{dekkers1989moment}, which extends Hill's estimator to the full real line and therefore remains valid for light-tailed as well as heavy-tailed residuals.
We bias-correct it by matched-sample calibration against synthetic symmetric $\alpha$-stable laws, and report the resulting in-range effective tail index $\hat\alpha_{\mathrm{eff}}$ (for a regularly varying tail the extreme-value index is positive and the tail index is its reciprocal). The Hill estimator, defined only for heavy tails, is retained as a reliability cross-check~\citep{hill1975simple}.
The statistical certification is rejection of the matched-$n$ Gaussian light-tail null, while the continuous effect size is reported as the calibrated effective tail index $\hat\alpha_{\mathrm{eff}}$.
Because ConstGate's gates are frozen, the injection enters its trainable input/recurrent weights only, and we do not require ConstGate to exhibit heavy realized log-rate motion $\Delta\zeta_q$. If the frozen gate prevents a certified heavy update driver from becoming heavy motion in the log-rate process, that suppression is itself the capacity barrier.
\emph{Path B} is the matched spontaneous baseline: ConstGate is trained with no intervention, and the forcing index, far-left drift plateau, spectrum, and envelope class are read directly from the realized trajectory.

\paragraph{Far-left drift diagnostic.}
For consecutive late-training checkpoints, we write
$\Delta\zeta_q(t)=\zeta_q(t+\Delta t)-\zeta_q(t)$ for the one-checkpoint
increment of unit~$q$. For a lower-tail fraction $q_{\mathrm{low}}$, let
$\zeta_{q_{\mathrm{low}}}$ denote the empirical
$q_{\mathrm{low}}$-quantile of the pooled late-window values of
$\zeta_q(t)$. We collect all one-checkpoint increments
$d_i=\Delta\zeta_q(t)$ whose starting value satisfies
$\zeta_q(t)\le\zeta_{q_{\mathrm{low}}}$, and write their order
statistics as $d_{(1)}\le\cdots\le d_{(n)}$.
The reported plateau estimate is the trimmed mean $\hat\kappa_{\mathrm{tail}}(q_{\mathrm{low}}) = \frac{1}{n-2m} \sum_{i=m+1}^{n-m} d_{(i)},\ m=\lfloor \rho n\rfloor,\ \rho=0.1$.
Positive values correspond to inward motion (i.e. a restoring effect) of the slow neurons back toward the bulk of the log-rate spectrum.
We report the primary cut $q_{\mathrm{low}}=0.10$ and a robustness sweep $q_{\mathrm{low}}\in\{0.03,0.05,0.10,0.15,0.20\}$.
Uncertainty is estimated by a block bootstrap over checkpoint-transition blocks, with the tail quantile re-estimated inside each bootstrap sample.

\paragraph{Results.}
The result is a clean negative control across ten seeds
(Figure~\ref{fig:exp1_summary}). The intervention does what it is designed to
do in the measured coordinate: the matched-$n$ Gaussian null is rejected in all
ten Path~A seeds (Gaussian-boundary $p\le 0.003$), with calibrated index
$\hat\alpha_{\mathrm{eff}}=1.749\pm0.009$, against
$\hat\alpha_{\mathrm{eff}}=1.997\pm0.001$ for the spontaneous baseline
(null rejected in only $3/10$ seeds, all with tail index estimates very close to the Gaussian boundary).
The two paths reach the same final training loss ($\approx0.322$), so the result is not a training-failure artifact.
Both paths remain collapsed in all ten seeds (Figure~\ref{fig:exp1_summary_spectrum}). The log-rate spectrum tracks the Gaussian reference, indicating a light-tailed distribution in~$\zeta$; accordingly, the time-scale spectra remain concentrated and closely follow the log-normal reference, with no resolved power-law window in their complementary cumulative distribution functions (CCDFs).
Consistently with theory, the envelope decay is exponential over the measured lag range (Figure~\ref{fig:exp1_summary_envelope}).

The drift diagnostic confirms the result (Figure~\ref{fig:exp1_summary_drift}).
The far-left plateau $\hat\kappa_{\mathrm{tail}}$ is negative across the entire $q_{\mathrm{low}}$ sweep in both paths.
At the primary cut $q_{\mathrm{low}}=0.10$, Path~A gives
$\hat\kappa_{\mathrm{tail}}=-7.98\times10^{-3}$ with 90\% CI
$[-8.70,-7.22]\times10^{-3}$, while Path~B gives
$\hat\kappa_{\mathrm{tail}}=-6.79\times10^{-3}$ with 90\% CI
$[-7.51,-6.10]\times10^{-3}$.
At this primary cut, the tail slice is flat and compatible with a constant model in both paths, but with the wrong, outward sign; the positivity check fails at every cut.
Supplying heavy forcing therefore does not move the frozen-gate pair toward a positive restoring plateau.

The negative control experiment clearly shows that ConstGate does not convert the artificial heavy-tailed forcing into the far-left drift and spectral geometry signatures of the anti-collapsed regime.

\begin{figure}[thp!]
\centering

\begin{subfigure}[t]{0.96\linewidth}
\centering
\includegraphics[width=\linewidth]{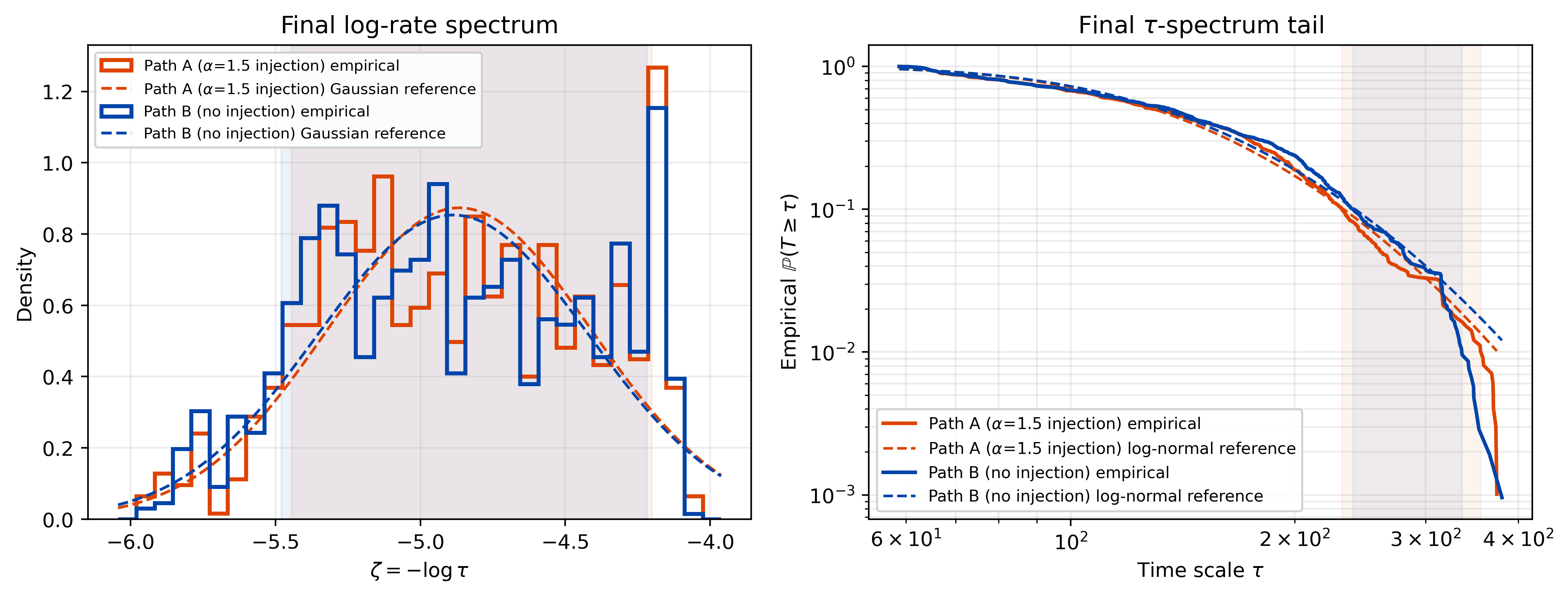}
\caption{Final log-rate spectrum and $\tau$-CCDF: concentrated with no resolved power-law window.}
\label{fig:exp1_summary_spectrum}
\end{subfigure}

\vspace{0.9em}

\begin{subfigure}[t]{0.48\linewidth}
\centering
\includegraphics[width=\linewidth]{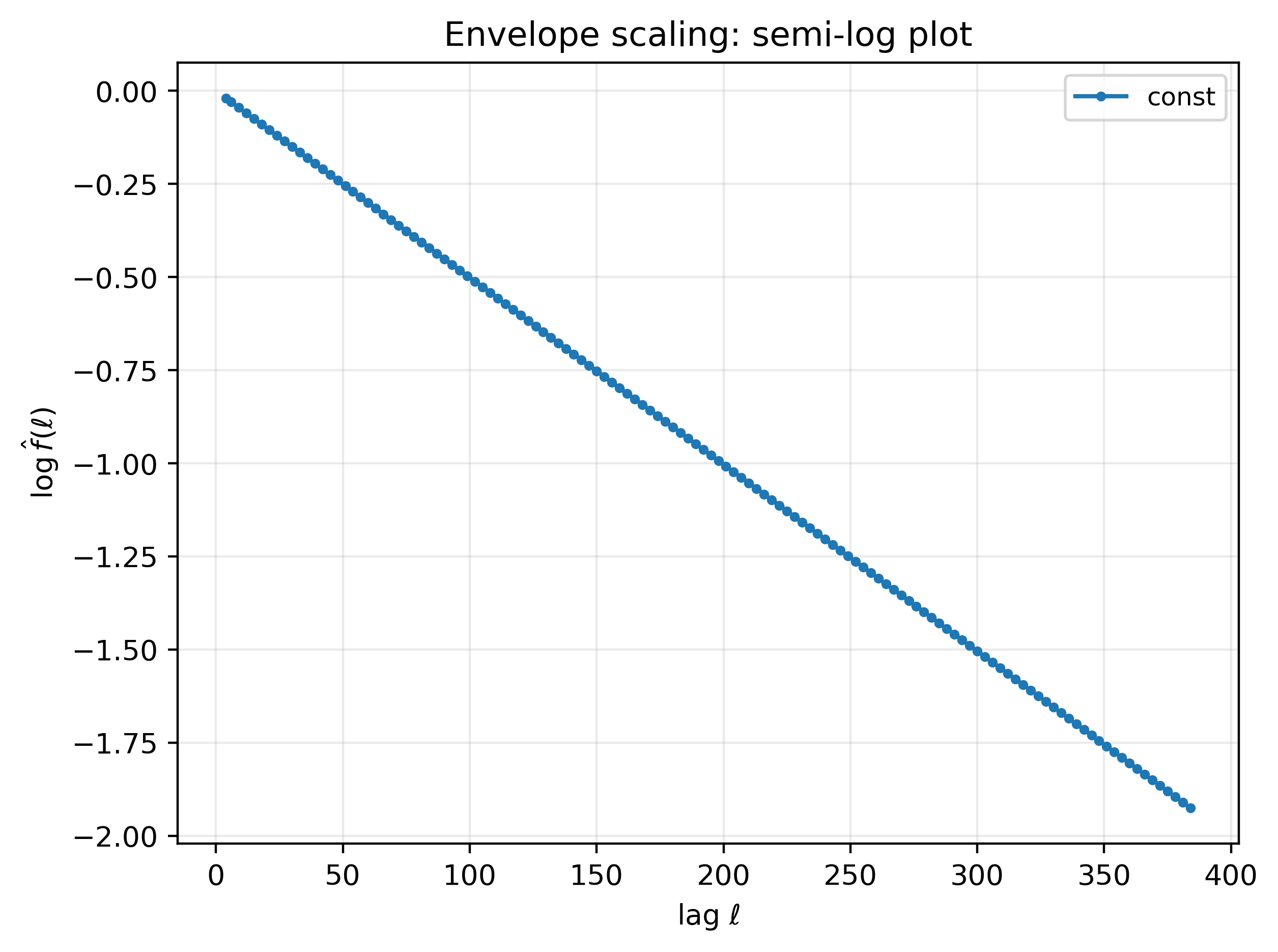}
\caption{Envelope $f(\ell)$ at convergence.}
\label{fig:exp1_summary_envelope}
\end{subfigure}
\hfill
\begin{subfigure}[t]{0.48\linewidth}
\centering
\includegraphics[width=\linewidth]{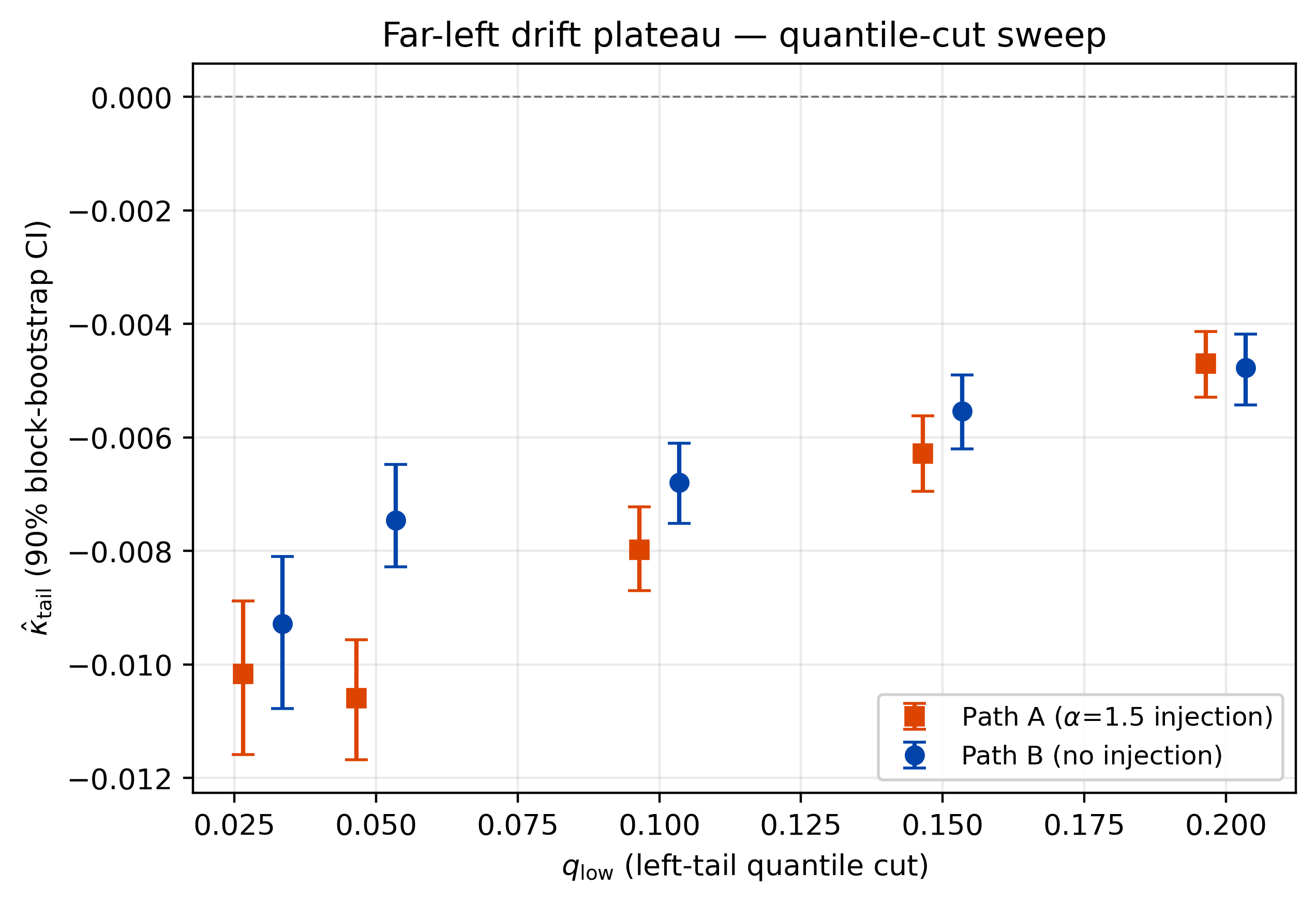}
\caption{Far-left drift plateau $\hat\kappa_{\mathrm{tail}}$ across the
$q_{\mathrm{low}}$ sweep: negative at every cut in both paths.}
\label{fig:exp1_summary_drift}
\end{subfigure}

\caption{Structural negative control on ConstGate.
Even when artificially injected heavy-tailed forcing is added into the slow-mode channel (Path~A), the frozen-gate model stays collapsed: the spectrum remains concentrated, the envelope stays exponential, and the far-left drift stays negative.}
\label{fig:exp1_summary}
\end{figure}

\subsection{Access route through trainable diagonal gates}
\label{sec:exp_access_route}

This experiment asks whether the access route that ConstGate fails to
realize becomes available once the model has per-unit spectral control.
We train SharedGate and DiagGate with AdamW on the same heavy-tailed-lag
task. SharedGate has one trainable gate shared across all recurrent
neurons: it can adapt a global decay scale, but cannot assign distinct
time scales to different neurons. DiagGate has one trainable gate per unit
and is the minimal architecture in this family with explicit per-unit
spectral degrees of freedom. The comparison is not a general capacity
ranking; it tests whether the route ingredients identified by the theory
appear together when heterogeneous spectral control is available.

\paragraph{Route diagnostics.}
We read the access route through four measurements.
First, the far-left drift diagnostic tests for a positive,
approximately saturating restoring plateau
$\hat\kappa_{\mathrm{tail}}(q_{\mathrm{low}})$.
We also display the associated conditional drift $\widehat F(\zeta) = \mathbb{E}[\Delta\zeta_q(t)\mid \zeta_q(t)=\zeta]$, estimated by binning late-window log-rate values and taking a trimmed mean of the corresponding increments.
The scalar $\hat\kappa_{\mathrm{tail}}(q_{\mathrm{low}})$ is the same drift observable averaged over the lower-tail slice $\zeta_q(t)\le \zeta_{q_{\mathrm{low}}}$.
Then, the slow log-rate increments are tested against a matched-$n$ Gaussian light-tail null, with effect size reported as the calibrated tail index $\hat\alpha_{\mathrm{eff}}$ using the estimator described in Section~\ref{sec:exp_negative_control_main}.
The estimated log-rate density should populate the far-left tail, and the time-scale CCDF should exhibit a resolved scaling window.
Finally, the macroscopic envelope should depart from the exponential class in a way consistent with the time-scale tail.

We do not estimate $\eta_J$ or $\eta_J^\ast$ directly.
Instead, the experiment tests whether heavy-tailed forcing signatures and a positive, approximately saturating far-left drift co-occur with a populated long-time-scale tail and compatible envelope decay.
These observations support the proposed drift--forcing mechanism, but do not by themselves establish that $\eta_J>\eta_J^\ast$. The diagnostics are compatible with both the root-selected and the tempering-controlled branches of anti-collapse.

\paragraph{Results.}

The drift diagnostics give the first route-level check.
Across the lower-tail sweep, DiagGate exhibits a positive restoring drift
on the far-left slice, with
$\hat\kappa_{\mathrm{tail}}(0.10)=3.3\times10^{-3}$
and a 90\% confidence interval (CI) $[2.7,4.0]\times10^{-3}$ (Figure~\ref{fig:exp2_drift_plateau}). SharedGate also has a small
positive estimate, but the magnitude is substantially weaker.
Thus, the drift diagnostic distinguishes the two models quantitatively: DiagGate
produces a much stronger restoring force in the slow tail, while
SharedGate remains close to a narrow, weak-drift regime.
The detailed DiagGate drift profile clarifies what the plateau summary compresses.
The conditional drift $\widehat F(\zeta)$ is strongly inward at the far-left edge
and decreases toward the bulk rather than forming a perfectly flat
constant over the entire displayed range
(Figure~\ref{fig:exp2_diag_conditional_drift}). We therefore read the
diagnostic as evidence for a positive finite-window saturating restoring
geometry, not as an exact test of an asymptotic constant.
The late-window moment traces are approximately stable (Figure~\ref{fig:exp2_diag_late_moments}), supporting the interpretation that the drift is being measured in a late-training regime rather than during a transient expansion of the spectrum.


The forcing diagnostic uses the same instrument as the negative control
(Section~\ref{sec:exp_negative_control_main}): a matched-$n$ Gaussian null for
significance, the calibrated effective tail index $\hat\alpha_{\mathrm{eff}}$ for
the continuous effect size, and the Hill estimate $\hat\alpha_{\mathrm{Hill}}$ as an extreme-tail-sensitive cross-check.
We lead with the coordinate the SDE forcing acts on directly: the drift-subtracted far-left log-effective decay rate increments.
For each model, we form the one-checkpoint far-left log-rate increments $\Delta\zeta_q$ and subtract the estimated conditional drift, $r=\Delta\zeta_q-\widehat F(\zeta_q)$, then compare these residuals against a matched Gaussian reference.
For DiagGate, the residual survival lies far above the matched Gaussian over the displayed tail, the quantile--quantile (QQ) plot shows standardized excursions more than an order of magnitude beyond the Gaussian range, and the extreme-tail cross-check is heavy, $\hat\alpha_{\mathrm{Hill}}\approx1.5$ (Figures~\ref{fig:exp2_diag_zeta_residual_logsurvival} and~\ref{fig:exp2_diag_zeta_residual_qq}).
SharedGate shows no such signature: its residuals stay within the Gaussian range and the cross-check saturates at the light-tail boundary, $\hat\alpha_{\mathrm{Hill}}\approx2$
(Figures~\ref{fig:exp2_shared_zeta_residual_logsurvival} and~\ref{fig:exp2_shared_zeta_residual_qq}).
Therefore, the residuals obtained from the drift statistics separate the two architectures cleanly: a heavy-tailed forcing component for DiagGate and a near-Gaussian residual for SharedGate.
The trajectory-level view confirms this from the opposite side (Figure~\ref{fig:exp2_forcing_trajectory}).
The calibrated index $\hat\alpha_{\mathrm{eff}}$ stays near the Gaussian boundary
($\alpha=2$) for both models throughout training, ending at $1.982\pm0.004$ for
DiagGate and $1.989\pm0.004$ for SharedGate, so its magnitude does not separate them.
The matched-$n$ Gaussian null is nonetheless rejected, in all random seeds for DiagGate (median $p\approx0.003$) and in $5$ seeds for SharedGate.
At this sample size, even a small non-Gaussian component is detectable, so the rejection confirms its presence without measuring its size.
We therefore read $\hat\alpha_{\mathrm{eff}}$ and the trajectory as a consistency audit.

Finally, the spectrum and envelope diagnostics show whether the route ingredients produce the stationary consequence predicted by the theory.
SharedGate remains concentrated near a single log-rate scale: its final log-rate density is narrow, and the corresponding time-scale CCDF decays rapidly rather than forming a scaling window (Figure~\ref{fig:exp2_time_scale_spectrum}).
DiagGate, by contrast, populates a broad far-left log-rate tail and produces a slowly decaying time-scale CCDF over the resolved range.
The dashed guide in Figure~\ref{fig:exp2_time_scale_spectrum} fits the pre-cutoff shoulder of the DiagGate CCDF; the subsequent downward bend is the expected finite-size/finite-context cutoff.

We now take into account the related envelope scaling behaviour. For DiagGate, the estimated envelope $\hat f(\ell)$ is well described on log--log axes by a finite-window power law, with the expected cutoff close to the sequence length (Figure~\ref{fig:exp2_envelope_loglog_comparison}).
SharedGate does not show the same behavior and the corresponding envelope is visibly exponential, confirming the results for the time-scale spectrum.

Quantitatively, the fitted average time-scale tail exponent for DiagGate is $\hat\beta_{\mathrm{tail}}\approx0.42$, with a seed-level 95\% CI $[0.27,0.57]$.
The corresponding envelope-side estimate is $\hat\beta_{\mathrm{env}}\approx0.29$, with seed-level 95\% CI $[0.28,0.30]$ (Figure~\ref{fig:exp2_beta_env_trajectory}).
Both estimates are below one, consistent with a broad anti-collapsed spectrum over the observed range.
On either anti-collapsed branch, the stationary time-scale survival function and the envelope share a leading decay exponent. This exponent is $\beta$ on the root-selected branch and $\lambda$ on the sub-threshold branch.
Figure~\ref{fig:exp2_beta_env_trajectory} compares the two finite-window estimates along training and focuses the statistical comparison on the final checkpoint, where the late-window diagnostics indicate quasi-stationarity.
At the final checkpoint, the seed-level CIs overlap and paired tests across seeds (two-sided paired $t$-test, with Wilcoxon signed-rank as a nonparametric check) do not reject equality of the two estimates at the 5\% level. We therefore interpret the agreement as a finite-window consistency check between the time-scale spectrum and the envelope class.

\begin{figure}[tph!]
  \centering

  \begin{subfigure}{0.92\textwidth}
    \centering
    \includegraphics[width=\linewidth]{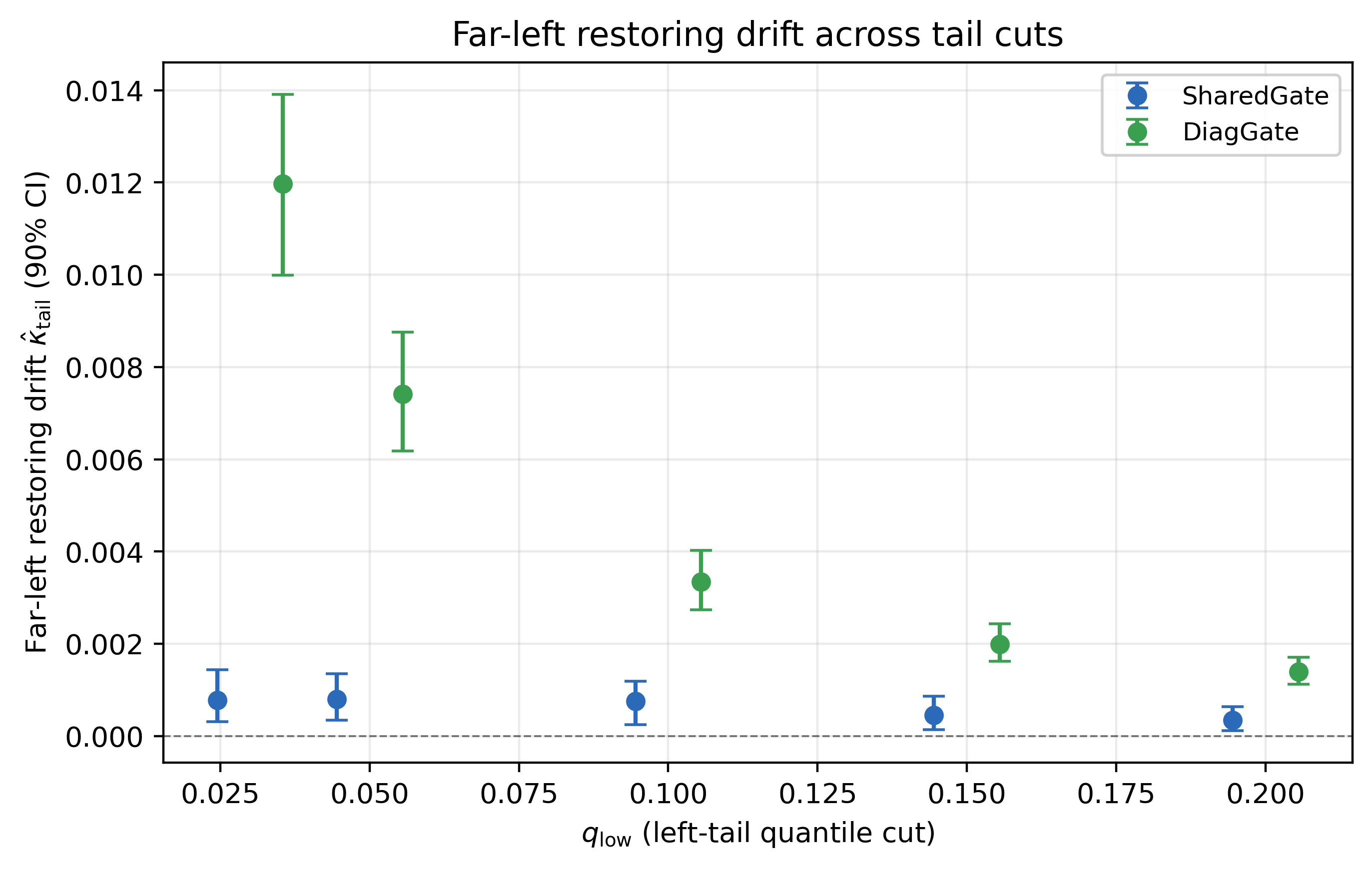}
    \caption{Far-left restoring drift across lower-tail cuts.}
    \label{fig:exp2_drift_plateau}
  \end{subfigure}

  \vspace{0.8em}

  \begin{subfigure}{0.48\textwidth}
    \centering
    \includegraphics[width=\linewidth]{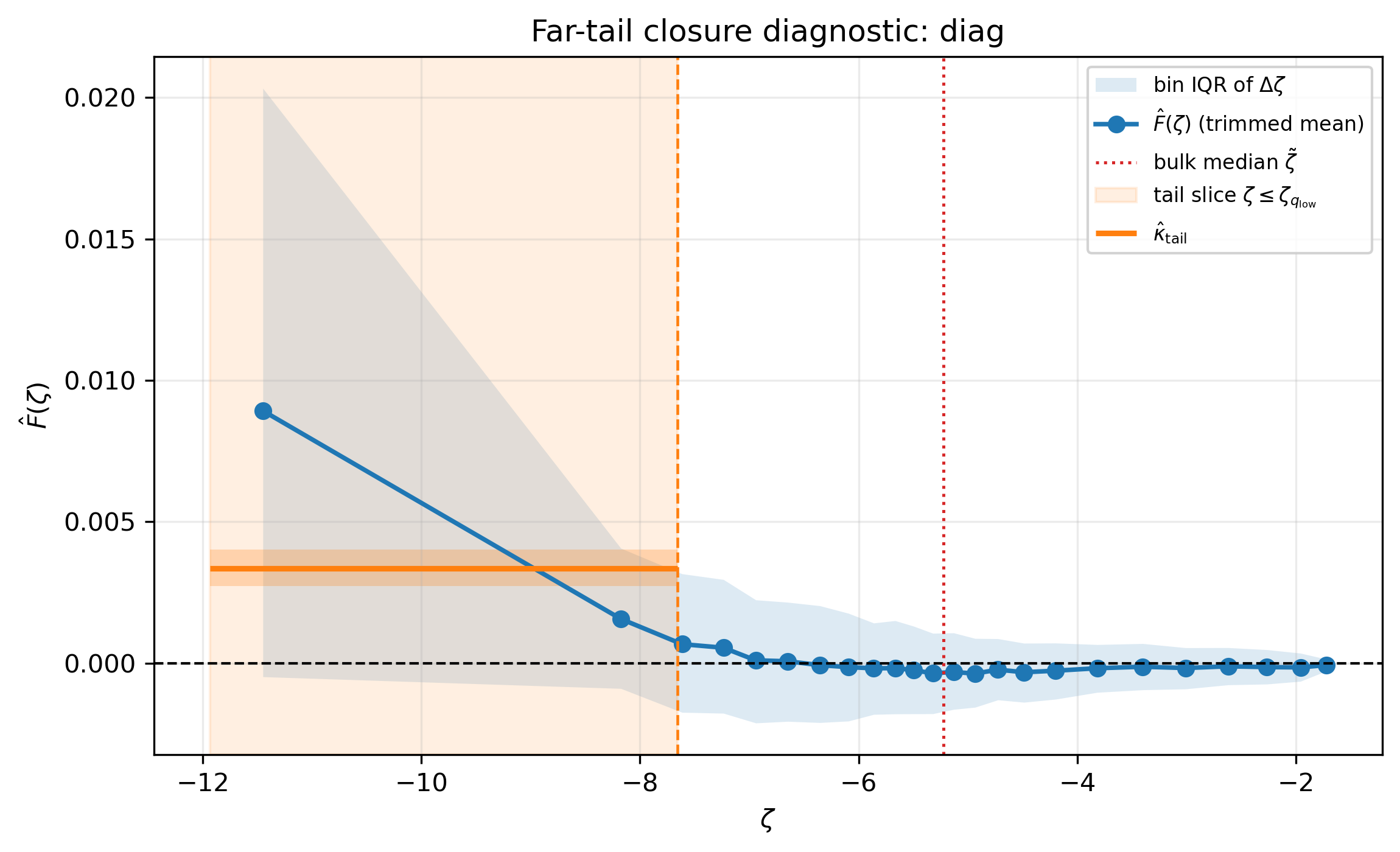}
    \caption{Conditional drift profile for DiagGate.}
    \label{fig:exp2_diag_conditional_drift}
  \end{subfigure}
  \hfill
  \begin{subfigure}{0.48\textwidth}
    \centering
    \includegraphics[width=\linewidth]{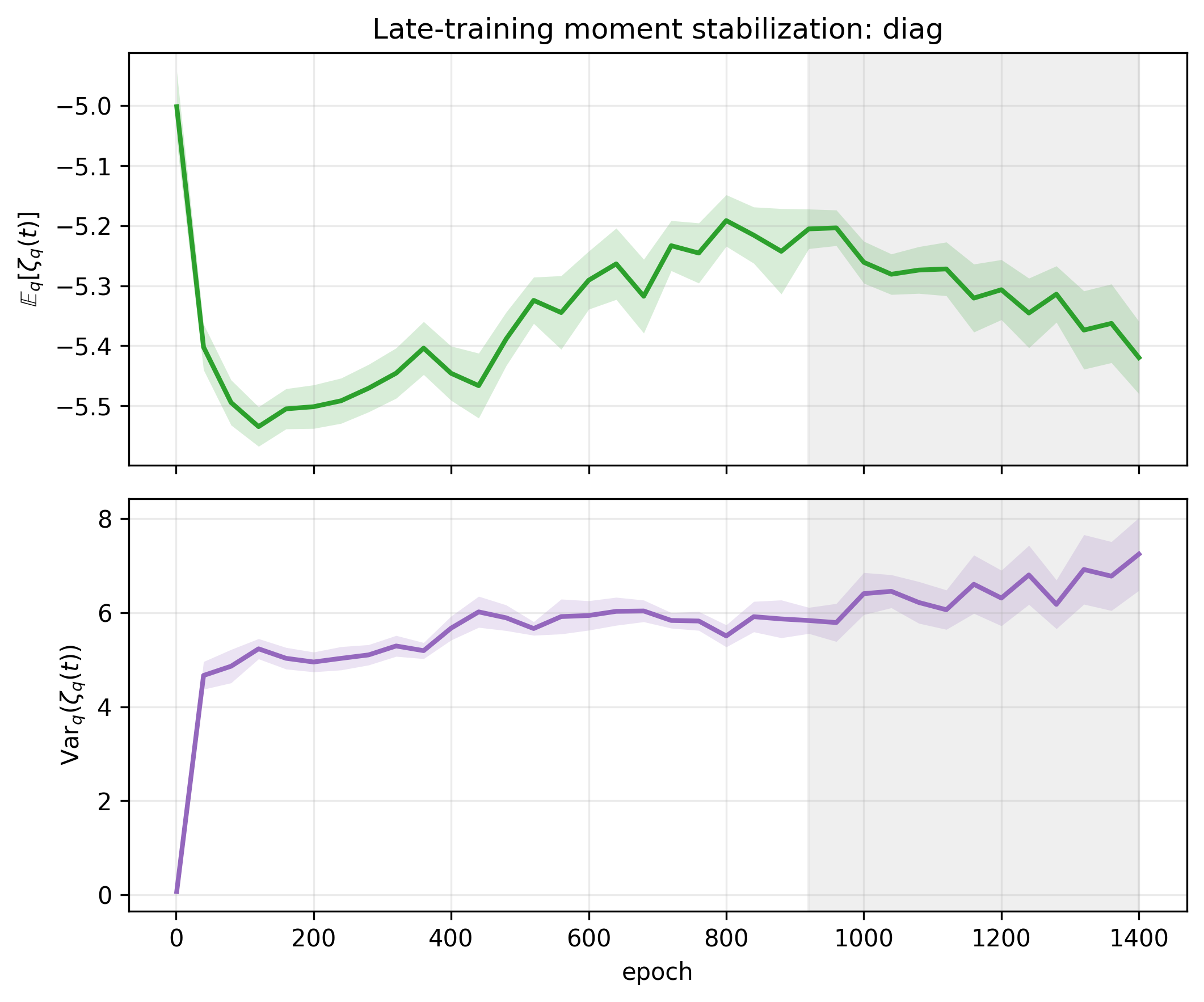}
    \caption{Late-window moment stability for DiagGate.}
    \label{fig:exp2_diag_late_moments}
  \end{subfigure}

  \caption{Drift diagnostics. The top panel compares the far-left
  restoring drift estimate across tail cuts. The lower panels show the
  DiagGate conditional drift profile and late-window moment stability used to
  assess the finite-window drift geometry.}
  \label{fig:exp2_drift_diagnostics}
\end{figure}

\begin{figure}[tph!]
  \centering

  \begin{subfigure}{0.48\textwidth}
    \centering
    \includegraphics[width=\linewidth]{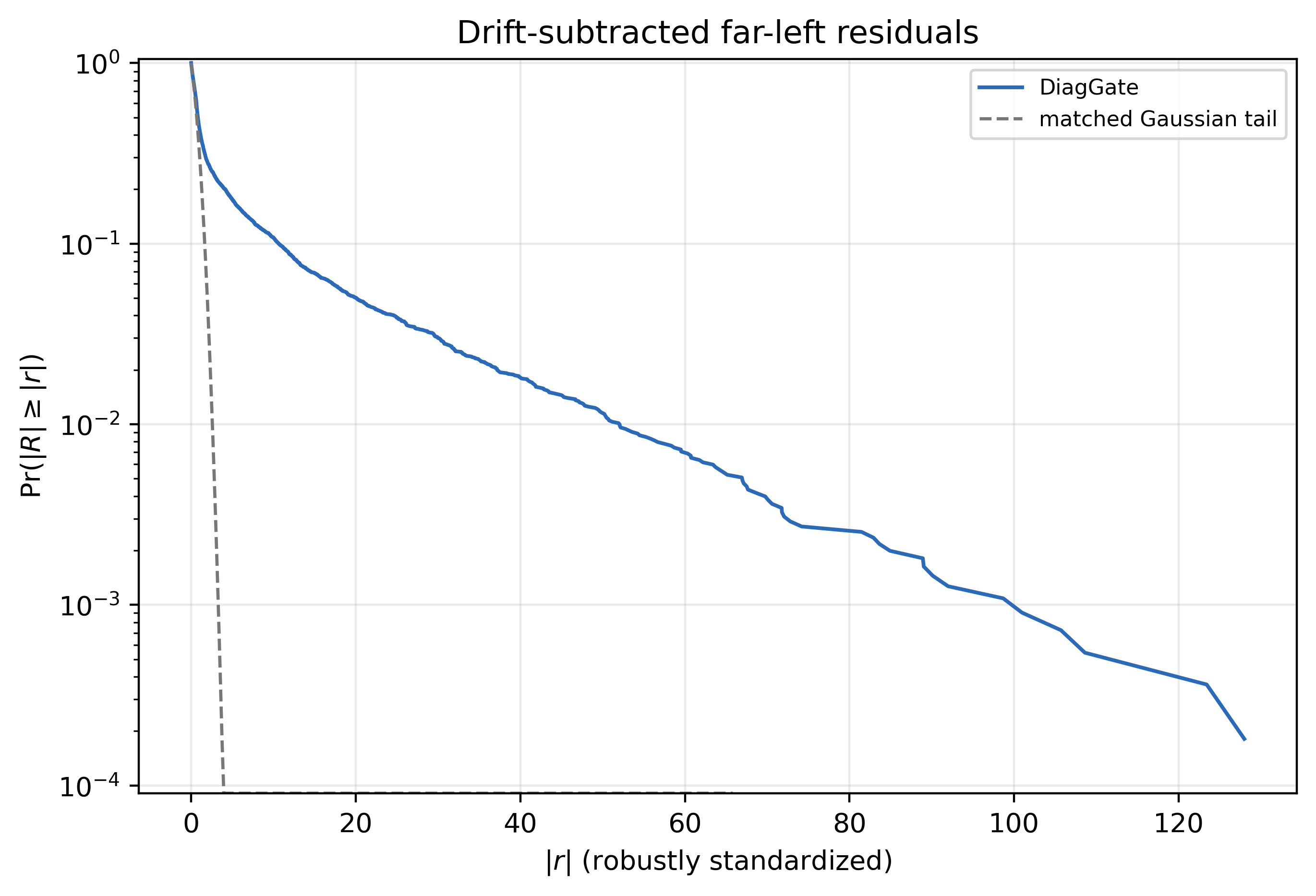}
    \caption{DiagGate residual log-survival.}
    \label{fig:exp2_diag_zeta_residual_logsurvival}
  \end{subfigure}
  \hfill
  \begin{subfigure}{0.48\textwidth}
    \centering
    \includegraphics[width=\linewidth]{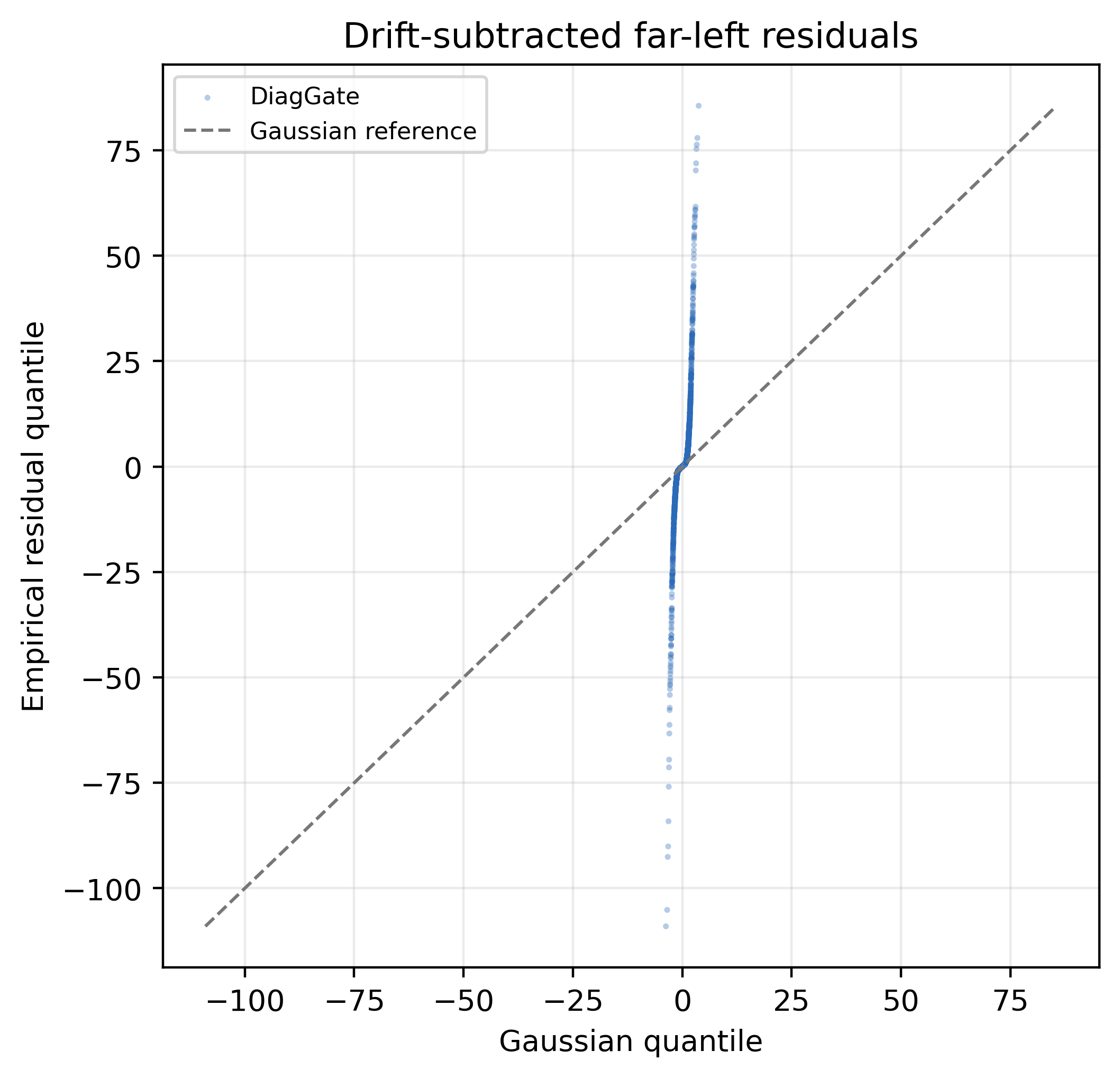}
    \caption{DiagGate residual QQ plot.}
    \label{fig:exp2_diag_zeta_residual_qq}
  \end{subfigure}

  \vspace{0.8em}
  
  \begin{subfigure}{0.48\textwidth}
    \centering
    \includegraphics[width=\linewidth]{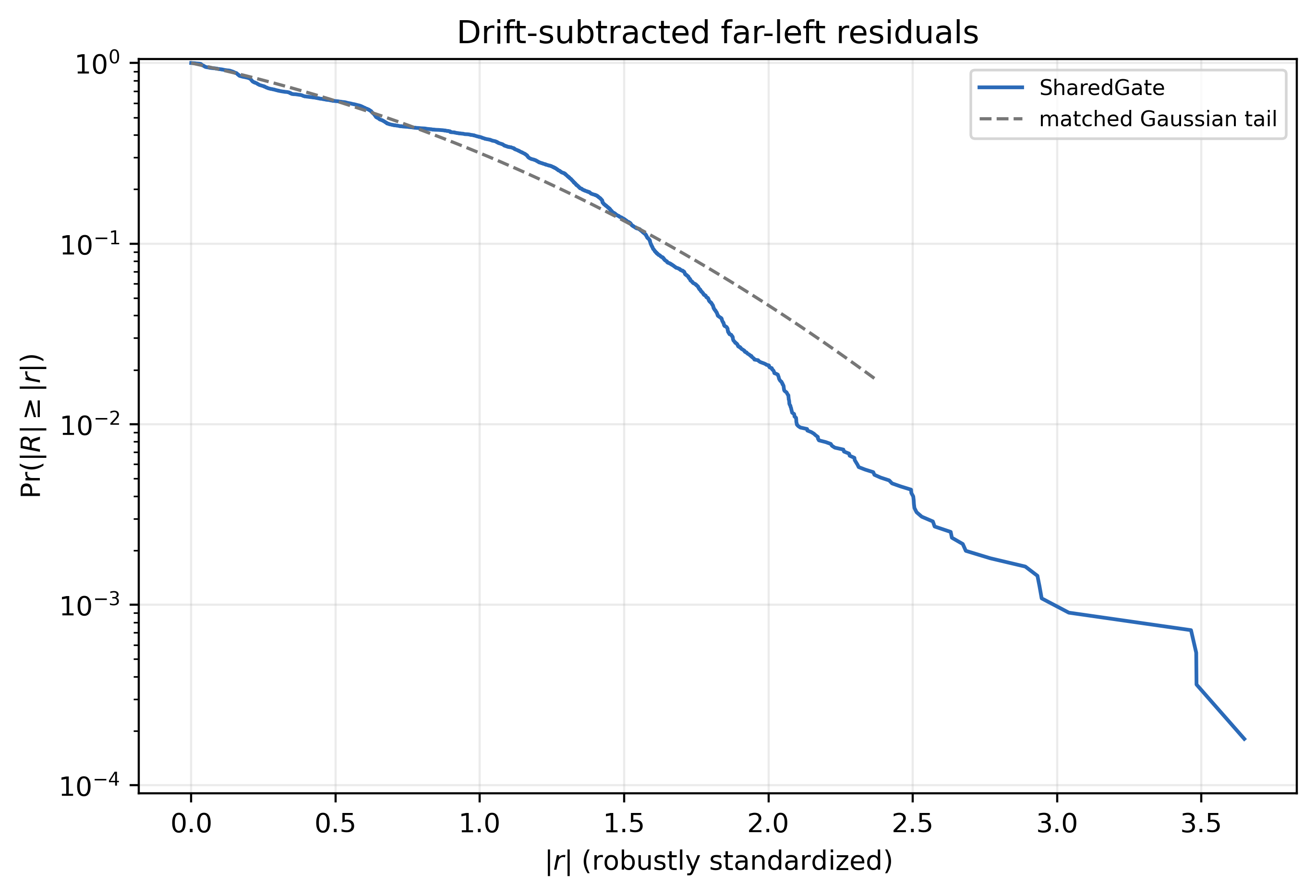}
    \caption{SharedGate residual log-survival.}
    \label{fig:exp2_shared_zeta_residual_logsurvival}
  \end{subfigure}
  \hfill
  \begin{subfigure}{0.48\textwidth}
    \centering
    \includegraphics[width=\linewidth]{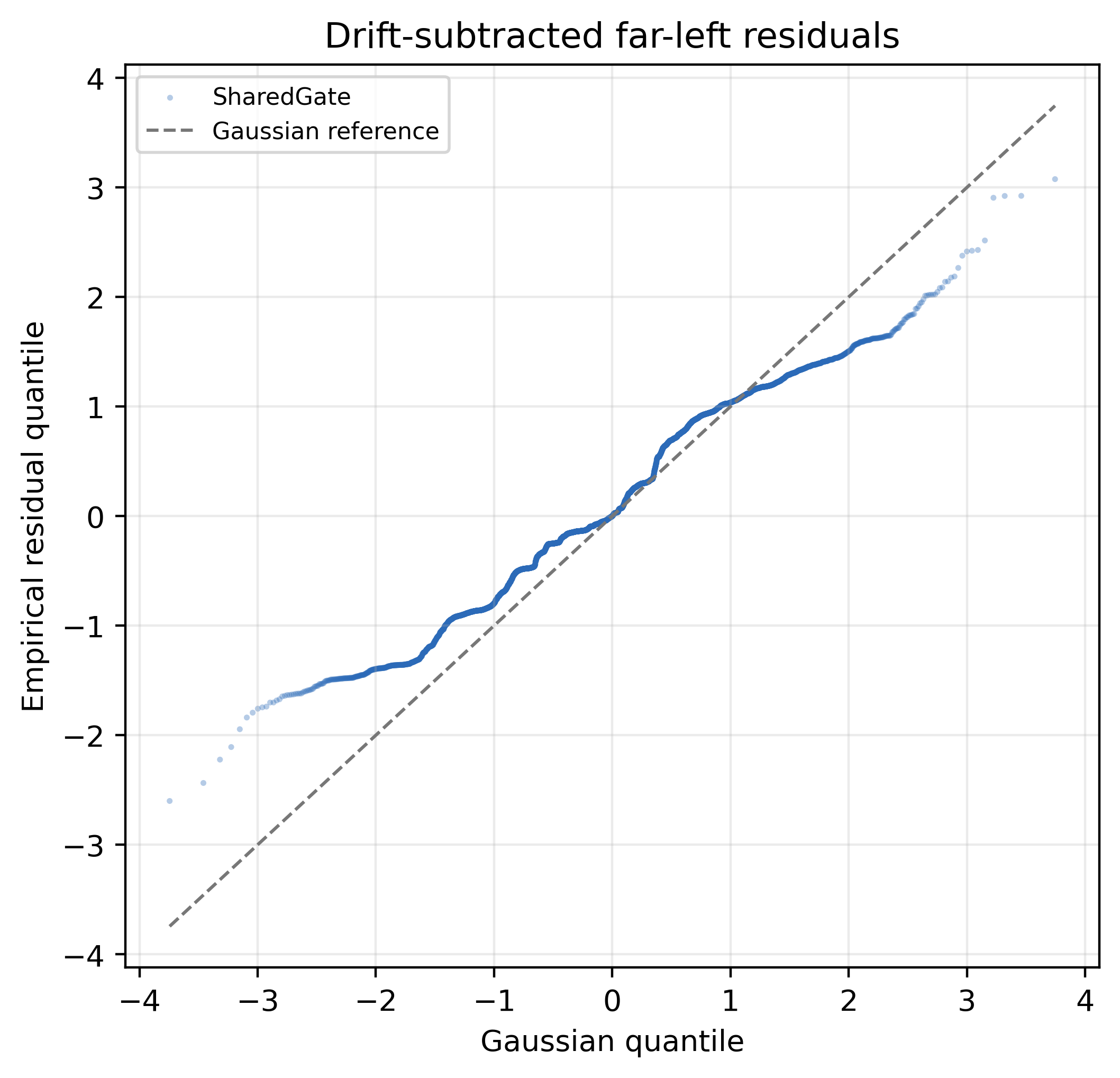}
    \caption{SharedGate residual QQ plot.}
    \label{fig:exp2_shared_zeta_residual_qq}
  \end{subfigure}

  \caption{Far-left log-rate increments for DiagGate (top) and SharedGate (bottom).
  Each residual $r=\Delta\zeta-\widehat F(\zeta)$ is the one-checkpoint increment after subtracting the estimated conditional drift.
  Writing $R$ for the robustly standardized residual, the left panels plot the empirical survival function of its magnitude, $\mathbb P(|R|\ge x)$, on a logarithmic vertical scale alongside the matched Gaussian reference.
  The right panels compare the quantiles of the signed, standardized residuals with Gaussian quantiles.}
  \label{fig:exp2_forcing_diagnostics}
\end{figure}

\begin{figure}[tph!]
  \centering

  \begin{subfigure}{0.92\textwidth}
    \centering
    \includegraphics[width=\linewidth]{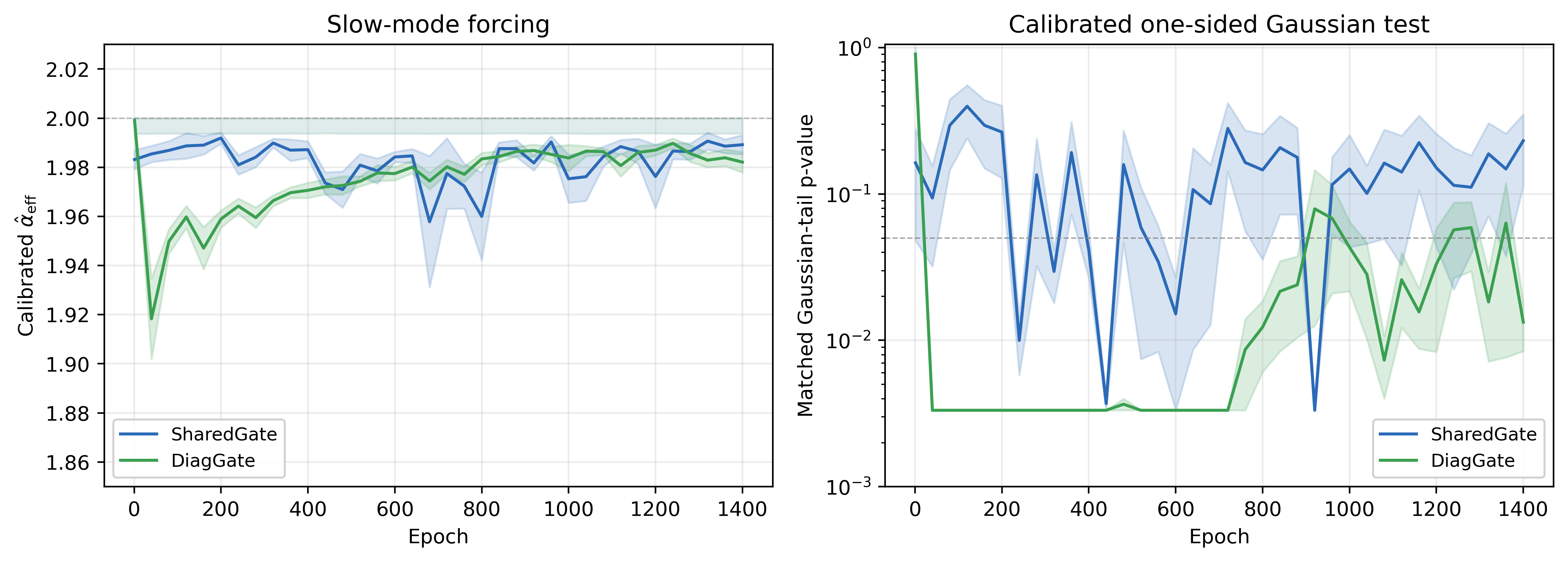}
  \end{subfigure}
  \caption{The panel reports the calibrated slow-mode tail-index $\hat{\alpha}_{\mathrm{eff}}$ across training.}
  \label{fig:exp2_forcing_trajectory}
\end{figure}

\begin{figure}[tph!]
  \centering

  \begin{subfigure}{0.92\textwidth}
    \centering
    \includegraphics[width=\linewidth]{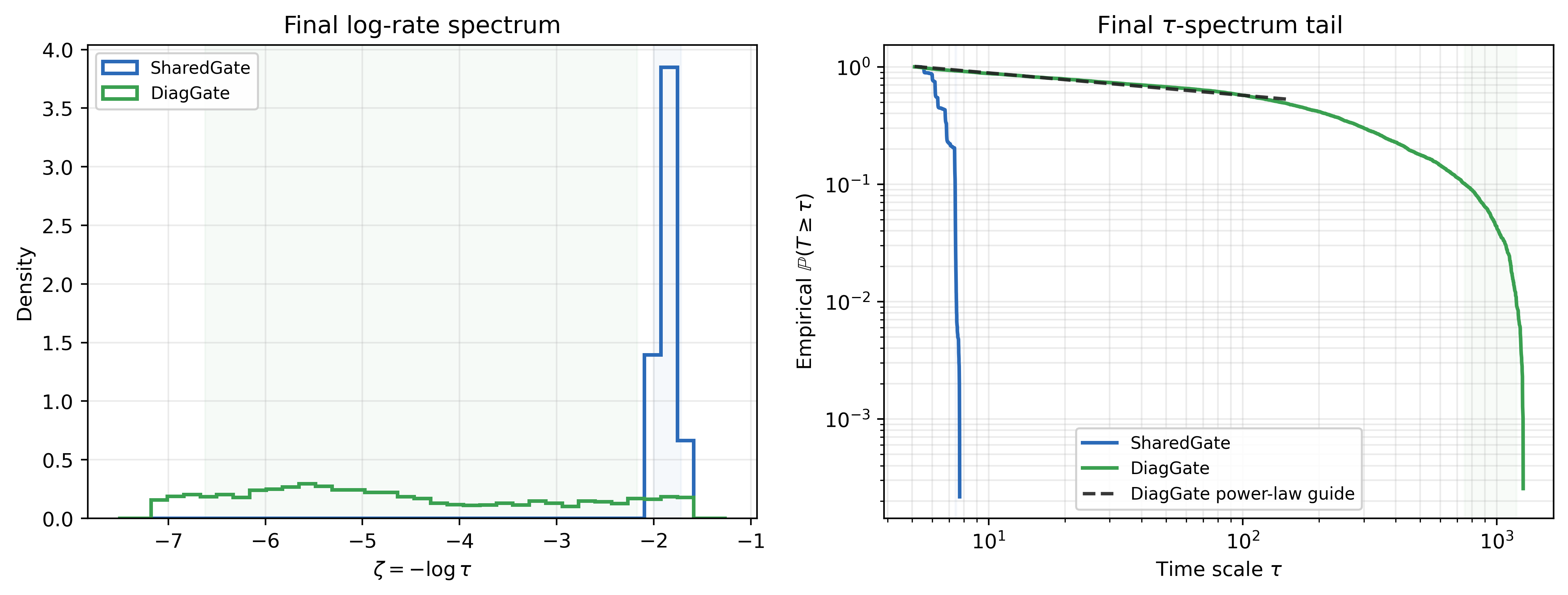}
    \caption{Final log-rate and time-scale spectra.}
    \label{fig:exp2_time_scale_spectrum}
  \end{subfigure}

  \vspace{0.8em}

  \begin{subfigure}{0.48\textwidth}
    \centering
    \includegraphics[width=\linewidth]{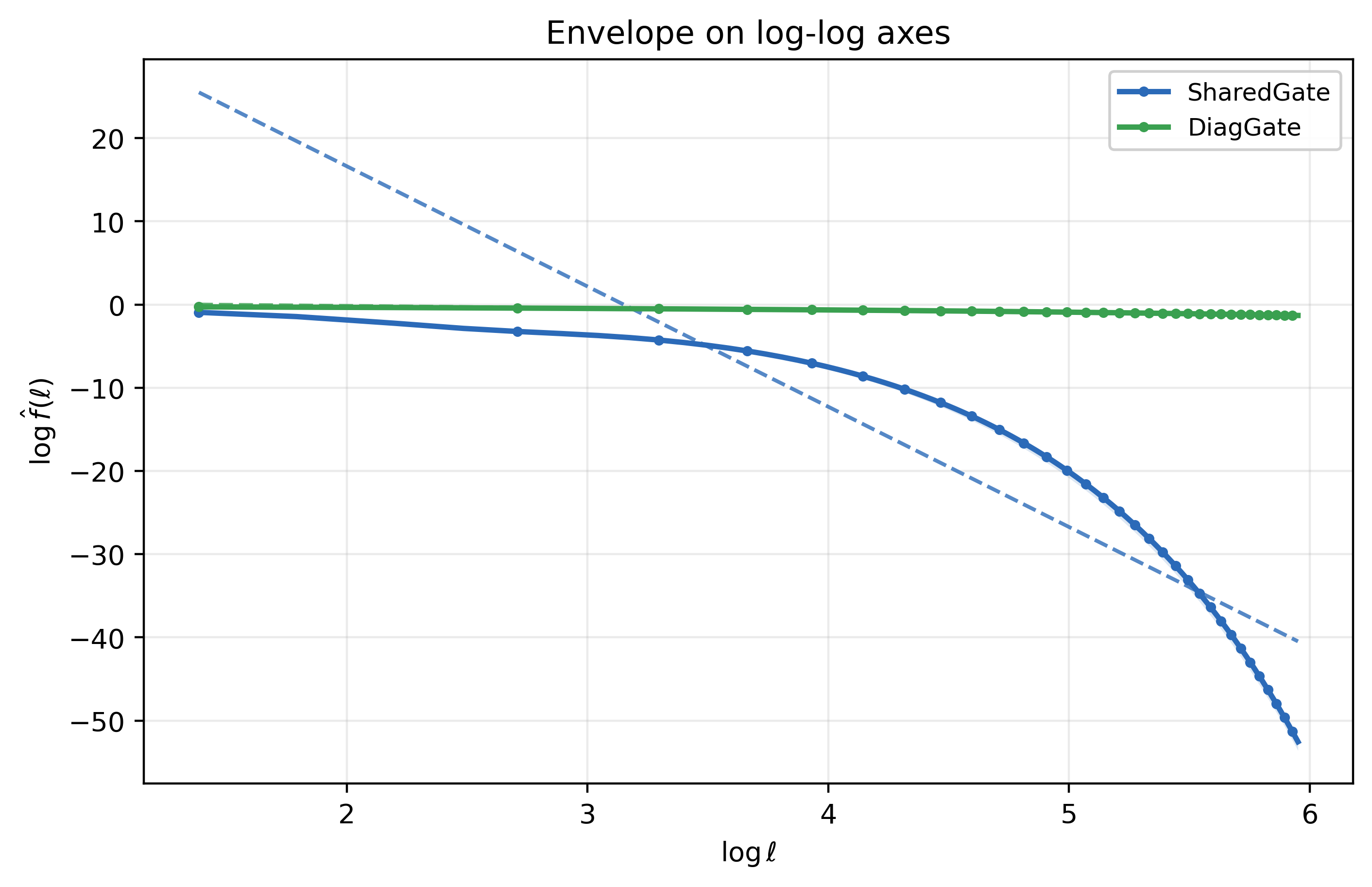}
    \caption{Macroscopic envelope on log-log axes.}
    \label{fig:exp2_envelope_loglog_comparison}
  \end{subfigure}
  \hfill
  \begin{subfigure}{0.48\textwidth}
    \centering
    \includegraphics[width=\linewidth]{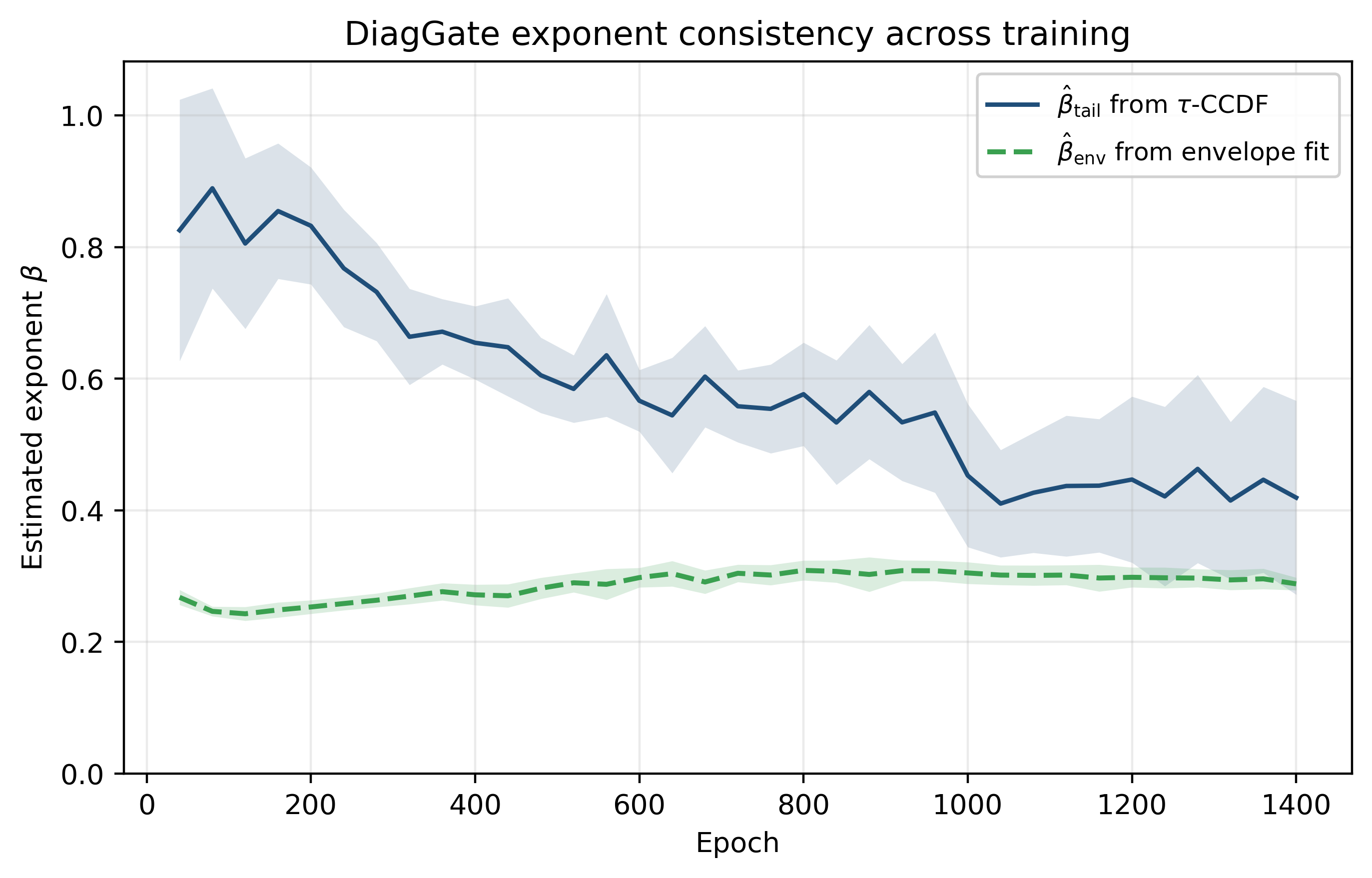}
    \caption{Finite-window decay exponents estimated over training for DiagGate's time-scale survival function ($\hat\beta_{\mathrm{tail}}$) and envelope ($\hat\beta_{\mathrm{env}}$).}
    \label{fig:exp2_beta_env_trajectory}
  \end{subfigure}

  \caption{Spectrum and envelope diagnostics. The top panel shows the
  final pooled log-rate spectrum and time-scale CCDF. The lower panels compare
  the decay of the envelopes and the fitted decay exponents.}
  \label{fig:exp2_spectrum_envelope}
\end{figure}

Both architectures train stably to a comparably low final training loss (DiagGate $\approx0.1$, SharedGate $\approx0.12$ across seeds), so the diagnostic contrast reflects their late-training dynamics rather than a difference in trainability.
The results support the access-route interpretation: DiagGate is the architecture in which the forcing, drift, spectrum, and envelope diagnostics appear together, whereas SharedGate, equally trainable, does not realize the anti-collapsed stationary signature.

\clearpage
\section{Discussion}
\label{sec:discussion}

We developed a stochastic model of how training shapes the effective time scales of the coupled state and parameter dynamics in recurrent neural networks.
Our results demonstrate a possible route to the onset of the anti-collapsed regime: a regime in which the time-scale spectrum of a gated RNN does not collapse onto short scales but retains a broad spread.
Three results carry the framework.
\emph{An effective coarse-grained stochastic model}.
The dynamics of effective time scales is described by a stochastic process with three physically motivated ingredients: a restoring drift inherited from the deterministic structure of the architecture, a Gaussian baseline noise from light-tailed minibatch fluctuations, and a variable-amplitude tempered-L\'evy jump component that summarizes the heavy-tailed forcing.
Although this is not the only admissible mesoscopic surrogate, we argue that it is the simplest one that is both mechanistically interpretable and analytically tractable, so that the phase structure and the jump-amplitude threshold for the root-selected branch can be characterized analytically.
\emph{A rigorous spectrum--envelope correspondence}.
The asymptotic decay class of the macroscopic envelope $f(\ell)$ is fixed by the tail of the stationary time-scale density $p_\infty(\tau)$ through a Laplace transform formalism.
A stationary density with tail $p_\infty(\tau)\sim c\,\tau^{-1-\beta}$ gives an envelope with decay exponent $\beta$, so either observable can be used to estimate the same spectral exponent.
\emph{A sharp phase structure with $\beta$ as a controllable order parameter}.
The characteristic equation identifies the primary jump-driven anti-collapsed branch, with power-law spectra and envelopes.
Within this branch, a codimension-one critical manifold at $\beta=1$ separates concentrated anti-collapse ($\beta>1$, finite mean time scale) from broad anti-collapse ($\beta<1$, divergent mean).
The canonical collapsed OU representative lies outside the regularly varying tail class: its envelope is operationally exponential before crossing to a non-power-law log-normal asymptotic.
The saturated-drift diffusion-only limit and the tempering-controlled branch below the threshold complement this primary root-selected description.
Log-regular tails complete the envelope taxonomy as an ansatz class, rather than a phase derived from the present stochastic model.

The experiments provide two complementary tests of this phase picture.
The frozen-gate negative control shows that heavy-tailed forcing by itself
does not produce anti-collapse: even when a calibrated heavy-tailed update
channel is supplied externally, ConstGate remains spectrally collapsed,
its envelope stays exponential, and its far-left drift plateau has the wrong sign.
Conversely, the trainable diagonal-gate experiment shows the route ingredients appearing together in the architecture that can assign distinct time scales to different neurons. DiagGate develops a populated far-left log-rate tail, non-Gaussian slow-mode statistics, a positive finite-window restoring drift in the tail, a broad time-scale spectrum, and a slowly decaying envelope consistent with the spectral tail; the SharedGate reference does not show this joint signature.
These results do not estimate the latent threshold parameters of the generator directly, but they validate the observable stationary consequences predicted by the stochastic model introduced in this paper.

\paragraph{Future directions.}

Several extensions of the present framework stand out, each of which sharpens a question the analysis raises but does not finish answering.

\emph{Capacity as a formal predicate.}
The phase analysis identifies the conditions under which the proposed far-tail route to anti-collapse is available.
Whether a particular architecture--optimizer pair realizes the route additionally
depends on its \emph{capacity}, i.e. the actual ability of that pair, evaluated under a specified task, data distribution, and experimental conditions, to generate, populate, and maintain a broad spectrum of effective time scales.
The empirical access-route tests used in this paper make realizability operationally testable but do not predict it from first principles.
A quantitative theory that formally defines capacity and the related realizability prediction would close the gap between the conditions identified here and the additional requirements imposed by realistic experimental conditions.

\emph{Sub-polynomial regimes: theoretical accessibility and practical realizability.}
The log-regular regime ($\beta\downarrow 0$) sits inside the asymptotic envelope taxonomy as a well-defined ansatz class, but is not realized as a non-degenerate phase by the present tempered-L\'evy model with asymptotically constant far-left restoring drift.
On both anti-collapsed branches, the model yields a regularly varying time-scale tail and an envelope with a positive leading decay exponent.
Accessing genuinely sub-polynomial envelopes as non-degenerate phases of the learning dynamics is therefore both a theoretical and an empirical question.
On the theoretical side, the drift--noise model used here is the simplest one consistent with the phenomenology, and extensions of it are the natural place to look for log-regular spectra.
The most direct candidate is a critically saturated drift in which the far-left restoring force approaches its asymptotic balance only through slow corrections, rebalancing the stationary tail equation toward the log-regular class.
On the empirical side, the question is whether realistic training dynamics can hold such a regime stably enough for the corresponding envelope decay to be observed: finite-width truncation, concentration near maximal time scales, and the operational fragility of the $\beta\downarrow 0$ boundary may all prevent it from appearing as a non-degenerate phase.
Settling this would simultaneously close a gap in the asymptotic envelope taxonomy and, by reaching even slower envelope decay, unlock significantly milder sample-complexity scaling for long-range learning.

\emph{Multistable and history-dependent generators.}
A separate extension is to relax the assumption that the
mesoscopic dynamics select a single stationary density. If the
effective generator has multiple invariant sectors, or if its
coefficients depend self-consistently on the evolving time-scale
population, different training histories could converge to
different stationary spectra under the same nominal
architecture--optimizer pair. In that setting, the phase portrait
would become multistable or history-dependent rather than a single-generator classification.

\emph{Cross-architecture generalization.} Although the theory is developed and tested on gated RNNs, the drift--noise phase structure and the spectrum--envelope correspondence are architecture-agnostic and apply, in principle, to any differentiable deep-learning model trained with SGD-like optimization.
Sequence models such as state-space models and transformers are natural targets to extend the theory~\citep{sieber2024understanding}.
The predicted scaling laws and phase structure are expected to transfer to other models once the notion of effective learning rates, envelope and related learnability quantities are properly mapped.

\paragraph{Self-organization and a control-theoretic perspective.}
Taken together, the learnability theory and the anti-collapse mechanism point to a self-organization picture of how training navigates the phase diagram.
The collapsed regime is not pathological in itself: when a task requires only short-range temporal structure, its operationally exponential envelope over the characteristic-scale range is entirely adequate, and gradient-based learning succeeds there without difficulty.
Over the median-dominated range relevant to finite-horizon learning, the canonical collapsed envelope is exponential and incurs an exponential sample-complexity cost. At its strict log-normal asymptotic, the cost becomes quasi-polynomial rather than exponential, but remains super-polynomial. On both anti-collapsed branches, the sample complexity is bounded by a power of the lag, including when the sub-threshold logarithmic correction is present.

This sample-complexity pressure is what makes anti-collapse a self-organizing target. Under long-range tasks, the joint dynamics of architecture and optimizer is selectively pushed toward the anti-collapsed regime whenever the conditions allow.
The aggregated heavy-tailed stochastic forcing that drives the transition is, on this reading, not an incidental nuisance to be smoothed away, but the mechanism by which training escapes the collapsed phase into a regime with polynomial sample complexity.
This view resonates with recent evidence that near-optimal performance does not need to concentrate on finely tuned critical manifolds but can occupy extended regions of (control-)parameter space characterized by heterogeneous sensitivities~\citep{bauer2025optimization}. We argue that the anti-collapse phase is a plausible and formal instance of that picture.

A complementary benefit is structural: anti-collapse is incompatible with an identifiable family of failure modes.
Vanishing- and exploding-gradient phenomena are specific, pathological instances of collapsed-envelope behaviour.
The learnability theory of \citet{livi2026learnability} shows that both failure modes can arise only when the macroscopic envelope decays exponentially, and are therefore dynamically excluded outside the collapsed phase.
Enforcing anti-collapse does not guarantee convergence to a task-specific optimum, but it does provide a structural guarantee against the gradient pathologies whose finite-time signature is a fast-exponential envelope decay.

This raises a natural question. Why design deep learning systems
that \emph{rely on} self-organization to enter the anti-collapsed
regime when, in principle, the regime could be enforced by construction?
A control-theoretic perspective on training suggests one answer.
Instead of treating optimization as the primary instrument and the resulting time-scale spectrum as an emergent byproduct, one can take the drift--noise balance itself as the primary design object and constrain the coupled dynamics to live in the anti-collapsed regime.
Given this enforceable constraint, the data and task-specific loss drive task performance within that regime.
The framework identifies the drift and forcing conditions for the proposed route to anti-collapse and provides diagnostics for assessing the resulting time-scale spectrum in practice.

Both readings, self-organization into anti-collapse and anti-collapse enforced by
construction, share the same mechanistic substrate, which we conjecture applies to
any differentiable deep-learning system trained with SGD-like optimization.
Establishing this would place the present analysis within the broader effort toward
a scientific theory of deep learning~\citep{simon2026scientifictheorydeeplearning}.
Beyond explanation, it could turn that theory into a foundation for engineering:
by exposing the order parameter, the access conditions, and the failure modes a
regime excludes, it would supply concrete principles for designing deep-learning
systems whose long-range learning is provably robust by construction.

\clearpage
\bibliographystyle{abbrvnat}
\bibliography{bibliography.bib}

\clearpage
\appendix

\section{Optimization via stochastic gradient descent}
\label{app:SGD}

This appendix summarizes the optimization and BPTT framework that gives rise to the effective learning rates $\mu^{(q)}_{t,\ell}$ introduced in Section~\ref{sec:asymptotic_decay_log_effective_LR}.
The full derivation and empirical validation are discussed in~\citet{livi2026learnability}.

Training RNNs follows standard gradient-based
optimization~\citep{ruder2016overview}.
Under plain SGD with learning rate $\mu>0$, the update at
iteration~$r$ is
\begin{equation}
\label{eq:sgd_update_app}
\theta_{r+1}
\;=\;
\theta_r
-\mu\,\nabla_\theta \mathcal{L}(\theta_r),
\qquad
\mathcal{L}
\;=\;
\sum_{t=1}^{T}\mathcal{E}_t,
\end{equation}
where $\mathcal{E}_t$ is the instantaneous loss at sequence time~$t$.
In practice, gradients are averaged over a randomly selected mini-batch of independent sequences.
Under an adaptive optimizer (Adam, RMSprop, etc.), the global rate is replaced by a diagonal preconditioner:
\begin{equation}
\label{eq:adaptive_update_app}
\theta_{r+1}
\;=\;
\theta_r
-\Lambda_r\,\nabla_\theta \mathcal{L}(\theta_r),
\qquad
\Lambda_r
=\mathrm{diag}(\lambda_{1,r},\dots,\lambda_{P,r}),
\end{equation}
where $P=\dim\theta$ and $\lambda_{i,r}>0$ are per-parameter adaptive
rates determined by the optimizer state.

In recurrent architectures, computing $\nabla_\theta\mathcal{L}$ requires unrolling the dynamics through time.
Let $h_t$ denote the recurrent state at step~$t$, and define the one-step state Jacobian $J_j=\partial h_j/\partial h_{j-1}$ and the parameter--state Jacobian $B_\ell(\theta)=\partial h_\ell/\partial\theta$.
The gradient contribution of the loss at time~$t$ is
\begin{equation}
\label{eq:bptt_one_step_app}
\frac{\partial\mathcal{E}_t}{\partial\theta}
=
\delta_t^\top
\sum_{\ell=1}^{t}
\mathcal{M}_{t,\ell}\,B_\ell(\theta),
\qquad
\mathcal{M}_{t,\ell}
=
\prod_{j=\ell+1}^{t}J_j,
\end{equation}
where $\delta_t=\partial\mathcal{E}_t/\partial h_t$ is the local loss gradient.
The state-transition Jacobian product $\mathcal{M}_{t,\ell}$ transports gradient information backward from step~$t$ to step~$\ell$; its structure determines how temporal credit is distributed across lags.

As in \citet{livi2026learnability}, in Eq.~\eqref{eq:bptt_one_step_app}, $\ell$ is an absolute source-time index. In other sections, $\ell$ denotes a temporal lag relative to the reference time~$t$.
To avoid introducing duplicate notation, we retain the same symbols after this re-indexing. Whenever $\ell$ denotes a temporal lag in a quantity inherited from the source-time BPTT formulation, that quantity is understood after replacing its source-time index by $t-\ell$; for example, $\mathcal{M}_{t,\ell}$ and $B_\ell$ denote $\mathcal{M}_{t,t-\ell}$ and $B_{t-\ell}$, respectively.

\section{Gated RNN architectures and transport factors}
\label{app:gated_rnns}

This appendix collects the recurrent dynamics and the resulting neuron-wise transport factors $\Gamma^{(q)}_{t,\ell}$ for the gated architectures used in this paper.
Under the GELR factorization~\eqref{eq:GELR_factorization}, the per-neuron effective learning rate is $\mu^{(q)}_{t,\ell}=\Lambda^{(q)}_{r,\ell}\,\Gamma^{(q)}_{t,\ell}$, where $\Lambda^{(q)}_{r,\ell}$ is the bounded adaptive base rate and $\Gamma^{(q)}_{t,\ell}$ encodes the state-space transport induced by gating.

For each architecture, $\Gamma^{(q)}_{t,\ell}$ is obtained from a first-order expansion~\citep{livi2025timescale,livi2026learnability} of the Jacobian product $\mathcal{M}_{t,\ell}$ and takes the form
\begin{equation}
\label{eq:transport_factor_general}
\Gamma^{(q)}_{t,\ell}
\;=\;
\gamma^{(0,q)}_{t,\ell}
\;+\;
\gamma^{(1,q)}_{t,\ell},
\end{equation}
where $\gamma^{(0,q)}_{t,\ell}$ is a zeroth-order gate-product and $\gamma^{(1,q)}_{t,\ell}$ collects first-order diagonal corrections.
Full derivations of the one-step Jacobians, the expansion, and the details of the zeroth and first-order terms are given in~\citet{livi2026learnability,livi2025timescale}.

\subsection{Gated RNNs}
\label{app:diag_gated}

Throughout, $x_t\in\mathbb{R}^D$ is the input, $h_t\in\mathbb{R}^H$ the hidden state, and $\odot$ denotes the Hadamard product.
All nonlinearities act elementwise.
These three models share the update template
\begin{equation}
\label{eq:baseline_dynamics}
h_t
\;=\;
(1-s_t)\odot h_{t-1}
\;+\;
s_t\odot \tanh(W_hx_t+U_hh_{t-1}+b_h),
\end{equation}
but differ in how the gate $s_t$ is produced:
\begin{itemize}[nosep]
\item \textbf{DiagGate.}\;
  $s_t=\sigma(W_sx_t+U_sh_{t-1}+b_s)\in(0,1)^H$
  \quad(per-neuron gate).
\item \textbf{SharedGate.}\;
  $s_t=\sigma(w_s^\top x_t+u_s^\top h_{t-1}+b_s)\in(0,1)$
  \quad(global scalar gate).
\item \textbf{ConstGate.}\;
  $s_t=s\in(0,1)$
  \quad(fixed scalar, not learned).
\end{itemize}

\subsection{Product structure and bounded transport}
\label{app:product_structure}

This appendix derives the asymptotic decay
rate~$\bar{\mu}_q$~\eqref{eq:mu_asymptotic_decay_rate} from the
structure of the recurrent transport. Throughout, $\ell$ is read
as the temporal displacement (lag) over which gradient
information decays, consistent with~\eqref{eq:mu_asymptotic_decay_rate};
the BPTT-style index range $\{\ell+1,\dots,t\}$ inherited
from~\eqref{eq:bptt_one_step_app} is retained as a notational
convenience.
The argument has three ingredients, established in turn in the
paragraphs below:
(i) the optimizer-amplitude factor $\Lambda^{(q)}_{r,\ell}$
contributes negligibly and drops out of the rate;
(ii) the zeroth-order gate-product transport produces a finite,
strictly positive rate under standing conditions on the gate
process;
(iii) first-order perturbative corrections to the gate-product
representation are subexponential in $\ell$ and do not modify
the leading rate.

\paragraph{Optimizer-amplitude contribution.}
We first show that the optimizer amplitude drops out of the
rate. The GELR factorization~\eqref{eq:GELR_factorization}
writes the effective learning rate as a product of optimizer
amplitude and transport factor; substituting
into~\eqref{eq:mu_asymptotic_decay_rate} and using linearity of
expectation,
\[
\bar{\mu}_q
=
-\lim_{\ell\to\infty}\frac{1}{\ell}
\Bigl(
\mathbb{E}\!\left[\log|\Lambda^{(q)}_{r,\ell}|\right]
+\mathbb{E}\!\left[\log|\Gamma^{(q)}_{t,\ell}|\right]
\Bigr).
\]
The optimizer-amplitude term contributes
$(1/\ell)\,\mathbb{E}\!\left[\log|\Lambda^{(q)}_{r,\ell}|\right]\to 0$
because $\Lambda^{(q)}_{r,\ell}$ is bounded above and bounded
away from zero in the late-training regime, which makes
$\log|\Lambda^{(q)}_{r,\ell}|$ uniformly $O(1)$ in $\ell$. The
asymptotic rate is therefore determined entirely by the
transport factor:
\[
\bar{\mu}_q
=
-\lim_{\ell\to\infty}\frac{1}{\ell}\,
\mathbb{E}\!\left[\log|\Gamma^{(q)}_{t,\ell}|\right].
\]
The transport factor has a zeroth-order gate-product
contribution plus first-order (and higher) perturbative
corrections, and we compute the limit in two steps: the next
paragraph extracts the leading rate from the zeroth-order gate
product alone, and the final paragraph shows that the
first-order corrections to that rate are subexponential in
$\ell$ and therefore do not modify the leading exponential
decay rate.

\paragraph{Zeroth-order rate.}
We first compute the leading asymptotic rate from the
zeroth-order contribution to the transport factor. Across all
five architectures, this contribution takes the generic
gate-product form
\begin{equation}
\label{eq:generic_product}
\gamma^{(0,q)}_{t,\ell}
\;=\;
\prod_{j=\ell+1}^{t}\chi_{j,q},
\qquad
\chi_{j,q}\in(0,1),
\end{equation}
where $\chi_{j,q}$ is a gate-derived retention factor specific
to each architecture ($1{-}s_{j,q}$ for the baseline models,
$f_{j,q}$ for LSTM, $1{-}z_{j,q}$ for GRU). Each finite-lag
factor is contractive, so $\gamma^{(0,q)}_{t,\ell}\in(0,1)$, and
its log-magnitude is a sum of $\ell$ retention-log terms,
$\log|\gamma^{(0,q)}_{t,\ell}|=\sum_{j=\ell+1}^{t}\log\chi_{j,q}$.

Since the asymptotic rate~\eqref{eq:mu_asymptotic_decay_rate}
is defined as a limit of expectations, we now compute
$\mathbb{E}\!\left[\log|\gamma^{(0,q)}_{t,\ell}|\right]$.
Two standing conditions on the gate process in the
late-training regime make the calculation tractable:
(i) quasi-stationarity, meaning that the marginal law of
$\chi_{j,q}$ is approximately invariant in $j$ across the
late-training window, so that per-summand expectations are
approximately constant;
(ii) finite stationary log moment,
$\mathbb{E}_\infty[|\log\chi_q|]<\infty$, which ensures that the
per-summand expectation $\mathbb{E}[\log\chi_{j,q}]$ is finite.
Linearity of expectation applied to the sum representation of
$\log|\gamma^{(0,q)}_{t,\ell}|$ gives
\[
\mathbb{E}\!\left[\log|\gamma^{(0,q)}_{t,\ell}|\right]
=\sum_{j=\ell+1}^{t}\mathbb{E}[\log\chi_{j,q}].
\]
Under quasi-stationarity, the marginal expectations $a_{j,q}=\mathbb{E}[\log\chi_{j,q}]$ approach the limiting stationary value $a_{\infty,q}=\mathbb{E}_\infty[\log\chi_q]$ in the late-training window.
Therefore, their average deviation vanishes,
\begin{equation}
\frac{1}{\ell}\sum_{j}\bigl(a_{j,q}-a_{\infty,q}\bigr)\longrightarrow 0,
\end{equation}
or equivalently $\sum_j(a_{j,q}-a_{\infty,q})=o(\ell)$.
Here, the sum is understood over the $\ell$ consecutive factors of the lag-$\ell$ transport, using the source-time re-indexing convention stated after Eq.~\eqref{eq:bptt_one_step_app}.
By Ces\`aro averaging, the cumulative deviation is sublinear in the lag, $\sum_j(a_{j,q}-a_{\infty,q})=o(\ell)$.
The premise that $a_{j,q}\to a_{\infty,q}$ over the late-training window is supported empirically by the same drift-closure analysis (Section~\ref{sec:exp_access_route}) that supports Assumption~\ref{ass:quasi_stationarity}. The late-window second moment of the per-unit log-decay-rate distribution $\mathrm{Var}_q(\zeta_q(t))$ is stable, confirming that the gate process settles into a quasi-stationary regime before the averaging is invoked.
Substituting yields
\begin{equation}
\label{eq:zeroth_order_log_identity}
\mathbb{E}\!\left[\log|\gamma^{(0,q)}_{t,\ell}|\right]
=\ell\,\mathbb{E}_\infty[\log\chi_q]+o(\ell),
\end{equation}
where the $o(\ell)$ term is precisely the cumulative deviation
of the marginal expectations from their limiting stationary value.
Dividing~\eqref{eq:zeroth_order_log_identity} by $\ell$ and
letting $\ell\to\infty$ yields a finite zeroth-order asymptotic
rate, strictly positive when
$\mathbb{E}_\infty[\log\chi_q]<0$ (non-degenerate average
contraction). This positivity condition is automatic in the
architectures considered here: the retention factor $\chi_{j,q}$
is produced by sigmoid-output gate operations
($1-\sigma(\cdot)$ for the baseline gates, $f_{j,q}=\sigma(\cdot)$
for LSTM, $1-z_{j,q}=1-\sigma(\cdot)$ for GRU) that strictly lie
in $(0,1)$ for any finite pre-activation, so $\log\chi_q<0$
almost surely under the stationary law and consequently
$\mathbb{E}_\infty[\log\chi_q]<0$.

\paragraph{First-order corrections are subexponential.}
We now show that the first-order corrections to the zeroth-order
rate are subexponential in $\ell$ and therefore do not modify
the leading exponential decay rate. The first-order diagonal
expansion of the recurrent Jacobian product writes each one-step
transport factor as $\chi_{j,q}+\varepsilon\,b_{j,q}$, where
$\varepsilon$ is a small perturbation amplitude and $b_{j,q}$
collects the bounded first-order contributions from gate
sensitivity and recurrent mixing~\citep{livi2025timescale}. We work in the perturbative
regime, in which $\varepsilon$ is small enough that the
first-order expansion is meaningful and the $b_{j,q}$ are
uniformly bounded:
\begin{equation}
\label{eq:bjq_uniform_bound}
\sup_{j}|b_{j,q}|\le M<\infty.
\end{equation}
Substituting $\chi_{j,q}+\varepsilon\,b_{j,q}$ for each one-step
factor in $\Gamma^{(q)}_{t,\ell}$ and expanding to first
order~\citep{livi2025timescale} yields
\begin{equation}
\label{eq:transport_first_order}
\Gamma^{(q)}_{t,\ell}
=\gamma^{(0,q)}_{t,\ell}
+\varepsilon\,\gamma^{(1,q)}_{t,\ell}+O(\varepsilon^2),
\qquad
\gamma^{(1,q)}_{t,\ell}
=\sum_{m=\ell+1}^{t}b_{m,q}
\!\!\prod_{\substack{j=\ell+1\\ j\ne m}}^{t}\!\!\chi_{j,q},
\end{equation}
i.e.\ the first-order envelope $\gamma^{(1,q)}_{t,\ell}$ is a
sum of $\ell$ contributions, each obtained by replacing a single
zeroth-order retention factor $\chi_{m,q}$ by its perturbative
counterpart $b_{m,q}$ and preserving the remaining $\ell-1$
contractive retention factors.

To bound $|\gamma^{(1,q)}_{t,\ell}|$ in terms of
$|\gamma^{(0,q)}_{t,\ell}|$, we proceed term by term. By the
uniform bound~\eqref{eq:bjq_uniform_bound} and the
factor-by-factor identity
$\prod_{j\ne m}\chi_{j,q}=\gamma^{(0,q)}_{t,\ell}/\chi_{m,q}$,
each summand in~\eqref{eq:transport_first_order} satisfies
\[
\left|b_{m,q}\!\!\prod_{\substack{j=\ell+1\\ j\ne m}}^{t}\!\!\chi_{j,q}\right|
\;\le\;M\,\frac{\gamma^{(0,q)}_{t,\ell}}{\chi_{m,q}}.
\]
We further assume non-degeneracy of the gate process in the
late-training regime, in the sense
\begin{equation}
\label{eq:nondegeneracy_chi}
\inf_{j}\chi_{j,q}\ge c>0,
\end{equation}
which bounds every $\chi_{m,q}$ from below by~$c$, so each
summand is at most $(M/c)\,\gamma^{(0,q)}_{t,\ell}$. Summing
over the $\ell$ values of $m$ gives the multiplicative bound
\begin{equation}
\label{eq:transport_first_order_bound}
\bigl|\gamma^{(1,q)}_{t,\ell}\bigr|
\;\le\;\frac{M}{c}\,\ell\,\bigl|\gamma^{(0,q)}_{t,\ell}\bigr|.
\end{equation}

We now convert the
expansion~\eqref{eq:transport_first_order} into a logarithmic
statement, since the asymptotic rate
$\bar{\mu}_q$~\eqref{eq:mu_asymptotic_decay_rate} is defined in
log-space and we want to express the first-order correction as
an \emph{additive} contribution to $\log|\Gamma^{(q)}_{t,\ell}|$.
Factoring $\gamma^{(0,q)}_{t,\ell}$ out
of~\eqref{eq:transport_first_order} gives
\[
\Gamma^{(q)}_{t,\ell}
=\gamma^{(0,q)}_{t,\ell}\!\left[\,1+\varepsilon\,
\frac{\gamma^{(1,q)}_{t,\ell}}{\gamma^{(0,q)}_{t,\ell}}\,\right]
+O(\varepsilon^2);
\]
taking absolute values and using $|ab|=|a||b|$,
\[
\bigl|\Gamma^{(q)}_{t,\ell}\bigr|
=\bigl|\gamma^{(0,q)}_{t,\ell}\bigr|\cdot
\left|\,1+\varepsilon\,
\frac{\gamma^{(1,q)}_{t,\ell}}{\gamma^{(0,q)}_{t,\ell}}\,\right|
+O(\varepsilon^2);
\]
and applying $\log$ together with $\log|ab|=\log|a|+\log|b|$ on
the leading product, with the $O(\varepsilon^2)$ remainder
absorbed additively under $\log$ in the perturbative regime,
yields
\begin{equation}
\label{eq:transport_log_decomposition}
\log\bigl|\Gamma^{(q)}_{t,\ell}\bigr|
=\log\bigl|\gamma^{(0,q)}_{t,\ell}\bigr|
+\log\!\left|\,1+\varepsilon\,
\frac{\gamma^{(1,q)}_{t,\ell}}{\gamma^{(0,q)}_{t,\ell}}\,\right|
+O(\varepsilon^2).
\end{equation}
The multiplicative
bound~\eqref{eq:transport_first_order_bound} controls the
relative first-order correction by
$\bigl|\varepsilon\,\gamma^{(1,q)}_{t,\ell}/\gamma^{(0,q)}_{t,\ell}\bigr|
\le(M\varepsilon/c)\,\ell$, so in the perturbative regime
\begin{equation}
\label{eq:transport_log_correction_bound}
\log\!\left|\,1+\varepsilon\,
\frac{\gamma^{(1,q)}_{t,\ell}}{\gamma^{(0,q)}_{t,\ell}}\,\right|
=O(\log\ell).
\end{equation}
To extract the asymptotic rate
$\bar{\mu}_q$~\eqref{eq:mu_asymptotic_decay_rate}, we divide
\eqref{eq:transport_log_decomposition} by $\ell$ and let
$\ell\to\infty$.
By~\eqref{eq:transport_log_correction_bound}, the first-order
correction contributes $O(\log\ell/\ell)\to 0$ in this limit,
and the leading exponential decay rate $\bar{\mu}_q$ is
therefore governed by the zeroth-order gate product alone.

\section{Numerical validation of the mixture representation and large-width concentration}
\label{app:mixture_and_large_width_validation}
 
The continuous envelope representation for large-width
networks rests on two structural approximations.
The first is the mixture-of-exponentials
approximation~\eqref{eq:mixture_of_exponentials}, which
replaces each neuron's full transport factor by a single
dominant exponential with time scale~$\tau_q$.
The second is the population-limit
passage~\eqref{eq:envelope_integral_weighted} from the
finite sum $H^{-1}\sum_{q=1}^H\exp(-\ell/\tau_q)$ to the
continuous integral $\int\exp(-\ell/\tau)\,p_\infty(\tau)\,d\tau$,
which requires that the empirical time-scale distribution
concentrates around a stable population law as $H$ grows.
This appendix validates each approximation separately:
Section~\ref{app:mixture_exp_validation} tests the
per-neuron exponential approximation at fixed width,
and Section~\ref{app:width_validation} tests the
population-level concentration as width increases.
 
\subsection{Mixture-of-exponentials validation}
\label{app:mixture_exp_validation}

This subsection tests whether replacing each neuron's full
transport factor by a single dominant exponential
$\exp(-\ell/\tau_q)$ faithfully reproduces the shape of
the intensive envelope at fixed width~$H$.
The comparison is between the actual intensive envelope
(computed from the full per-neuron transport factors) and
the mixture envelope (computed from the fitted time scales~$\tau_q$ alone).

We evaluate DiagGate as a representative architecture (Appendix~\ref{app:gated_rnns}) on a synthetic delayed-regression task with $H=64$, $T=300$, $D=8$, and $N_{\mathrm{seq}}=256$ for 200~epochs, under three optimizers (SGD, Adam, RMSProp) \citep{ruder2016overview} and five random seeds.
For each trained network, the per-neuron transport factors
$|\mu^{(q)}_{t,\ell}|$ are computed from the first-order diagonal expansion at 35~lags spanning $1 \le \ell \le 140$.
For each neuron~$q$, the asymptotic decay rate $\bar{\mu}_q$ is extracted by linear regression of $\log|\mu^{(q)}_{t,\ell}|$ on $\ell$ at large lags ($\ell \ge 16$), and the time scale $\tau_q = 1/\bar{\mu}_q$ is recorded.
 
Two envelopes are compared: the \emph{actual} intensive envelope $f(\ell) = H^{-1}\sum_q |\mu^{(q)}_{t,\ell}|$ and the \emph{mixture} envelope $f_{\mathrm{mix}}(\ell) = H^{-1}\sum_q \exp(-\ell/\tau_q)$.
Following~\citet{livi2026learnability}, we report the Spearman rank correlation (testing whether the two envelopes rank lags identically) and the Pearson correlation on the $\log_{10}$-transformed envelopes (testing whether the decay shapes are linearly related in log-space).

\begin{table}[h]
\centering
\small
\begin{tabular}{l l c c c}
\toprule
Architecture & Optimizer & Spearman & Pearson (min) & $R^2_{\mathrm{exp}}$ (median) \\
\midrule
DiagGate  & SGD     & $1$ & $1$ & $1$ \\
DiagGate  & Adam    & $1$ & $0.99$ & $1$ \\
DiagGate  & RMSProp & $1$ & $1$ & $1$ \\
\bottomrule
\end{tabular}
\caption{Mixture-of-exponentials validation for DiagGate with 3~optimizers, 5~seeds.
Spearman and Pearson report the worst case across the 5~seeds per configuration.
Median per-neuron coefficient of determination $R^2$ of the exponential slope fit, measuring how well each neuron's transport decays as a single exponential.}
\label{tab:mixture_validation}
\end{table}

The Spearman rank correlation is 1 in all runs without exception: the mixture preserves the monotonic decay ordering under every optimizer tested over the sampled lag range ($1 \le \ell \le 140$).
This confirms that the per-neuron exponential approximation faithfully captures the decay structure of the envelope, which is the prerequisite for the population-limit passage tested in Section~\ref{app:width_validation}.
 
For architectures with simple gating (DiagGate), the
single-exponential-per-neuron model is essentially exact
($R^2 \approx 1$), and the mixture reproduces the actual envelope to
high fidelity (Pearson $\ge 0.999$) across all optimizers.

\subsection{Large-width population concentration}
\label{app:width_validation}
Given that the per-neuron exponential approximation is faithful (Section~\ref{app:mixture_exp_validation}), the remaining question is whether the intensive envelope $f(\ell) = H^{-1}\sum_{q=1}^H\exp(-\ell/\tau_q)$ can be reliably described by the continuous integral $\int\exp(-\ell/\tau)\,p_\infty(\tau)\,d\tau$ as $H$~grows.
This requires that the empirical time-scale distribution $H^{-1}\sum_q\delta_{\tau_q}$ concentrates around a stable population law~$p_\infty(\tau)$.
We test this by varying $H$ and measuring both the convergence of the time-scale distribution and the stability of the intensive envelope.
We use again DiagGate trained with fixed-learning-rate SGD ($\mathrm{lr}=10^{-3}$) at six widths\\
$H \in \{16, 32, 64, 128, 256, 512\}$, each with 5~random seeds ($T=300$, $D=8$, $N_{\mathrm{seq}}=128$, 200~epochs).
The $H=512$ runs serve as the reference population for both the distributional and envelope diagnostics reported below.
We use two diagnostics.
The first tracks the empirical time-scale distribution itself:
at each width, we compute the Wasserstein-1 distance~$W_1$ between the pooled-across-seeds $\log\tau$ histogram and the $H=512$ reference histogram, and record the within-seed median $\tau_{\mathrm{med}}$ together with its cross-seed standard deviation.
Both quantities are expected to shrink as~$H$ grows if the empirical spectrum is concentrating to a stable population law.
Table~\ref{tab:width_convergence} reports the results.
The informative reading is not the absolute magnitude of $\tau_{\mathrm{med}}$, which is architecture-dependent, with DiagGate's per-neuron gates admitting very slow retention modes ($\tau_{\mathrm{med}}\approx 200$). The important result is the rate at which $W_1$ and the cross-seed spread of $\tau_{\mathrm{med}}$ decay with~$H$.
For DiagGate, $W_1$ falls from $0.114$ at $H=16$ to $0.013$ at $H=256$, nearly an order of magnitude, while the cross-seed spread of $\tau_{\mathrm{med}}$ contracts from $\pm 9.0$ to $\pm 0.5$.
Fitting a power law to the cross-seed variability of the median gives $\mathrm{std}(\tau_{\mathrm{med}})\sim H^{-1.01}$, thus slightly deviating from the $H^{-1/2}$ law-of-large-numbers rate expected from independent neuron-wise draws from $p_\infty$.
\begin{table}[h]
\centering
\small
\begin{tabular}{r c c}
\toprule
$H$ & $W_1$ & $\tau_{\mathrm{med}} \pm \mathrm{std}$ \\
\midrule
16  & 0.114 & $184.5 \pm 9.0$ \\
32  & 0.104 & $201.8 \pm 7.4$ \\
64  & 0.059 & $197.1 \pm 3.2$ \\
128 & 0.037 & $200.0 \pm 2.4$ \\
256 & 0.013 & $198.9 \pm 0.5$ \\
512 & ---   & $199.1 \pm 0.4$ \\
\bottomrule
\end{tabular}
\caption{Concentration of the $\log\tau$ distribution to the
$H=512$ reference (DiagGate).
$W_1$: Wasserstein-1 distance between the pooled-across-seeds
empirical $\log\tau$ histogram at width~$H$ and the reference
histogram; ``---'' marks the reference width itself
(self-comparison).
$\tau_{\mathrm{med}}$: median $\tau$ within each seed, reported
as mean~$\pm$~standard deviation across the 5~seeds.
The informative signature of concentration is the rate at
which~$W_1$ and the cross-seed standard deviation shrink with $H$.}
\label{tab:width_convergence}
\end{table}

The second diagnostic tracks the intensive envelope directly.
For each width and each seed we compute $f(\ell)=H^{-1}\|\mu_{t,\ell}\|_1$ as a function of lag~$\ell$ and seed-average to obtain $f_H(\ell)$.
We then report three complementary numbers.
The Pearson correlation coefficient~$r$ between $\log_{10}f_H(\ell)$ and $\log_{10}f_{H=512}(\ell)$, with lag~$\ell$ treated as the sample axis, measures how closely the shape of the envelope as a function of lag at finite width matches the shape at the reference width.
Because Pearson correlation is affine-invariant in each of its arguments, $r$ is insensitive to a uniform vertical shift or rescaling of $\log f_H$ relative to $\log f_{H=512}$: it detects changes in relative lag-to-lag variation but not in overall level.
To also catch absolute level agreement, we report the log-sup and log-RMSE errors
\begin{align}
\Delta_\infty(H) &= \sup_\ell
\bigl|\log_{10} f_H(\ell) - \log_{10} f_{H=512}(\ell)\bigr|,
\\
\Delta_2(H) &=
\Bigl(\tfrac{1}{L}\sum_{\ell}
\bigl(\log_{10} f_H(\ell) - \log_{10} f_{H=512}(\ell)\bigr)^2
\Bigr)^{1/2},
\end{align}
evaluated over the subset of lags at which the reference envelope is resolved above $f_{H=512}(\ell)\ge 10^{-3}$.
As a threshold-free companion diagnostic, we also report the envelope-weighted log-RMSE
\begin{equation}
\Delta_2^{w}(H) =
\left(
\frac{\sum_\ell
f_{H=512}(\ell)\,
\bigl(\log_{10} f_H(\ell) - \log_{10} f_{H=512}(\ell)\bigr)^2}
{\sum_\ell f_{H=512}(\ell)}
\right)^{1/2},
\end{equation}
summed over all 35~lags, which uses the reference envelope itself as a weight so that the deep-tail regime contributes negligibly to the metric by construction rather than by an imposed cutoff.

Results for all three scale-sensitive metrics, together with~$r$, are collected in Table~\ref{tab:envelope_logerr}.
The shape agreement is near-perfect at every tested width ($r = 1.0$), consistent with the Laplace-transform smoothing that maps spectra to envelopes.
The absolute level agreement, captured by $\Delta_\infty$ and $\Delta_2$, is a strictly more demanding criterion: at $H=16$ the worst lag disagrees with the reference by about $0.014$ decades, and both quantities contract to the third decimal place by $H=256$ ($\Delta_\infty\approx 3\times10^{-4}$).
The envelope-weighted $\Delta_2^{w}$, which avoids the threshold by using the reference envelope itself to down-weight deep-tail lags, tells the same story at slightly lower overall magnitude: at $H=16$, $\Delta_2^{w}\approx 4\times10^{-3}$, contracting to sub-$10^{-3}$-decade agreement by $H=256$.
We note that the threshold-free metric lands at values comparable to the thresholded $\Delta_2$, indicating that the $10^{-3}$ cutoff is doing only what it is meant to do---excluding the numerically suppressed tail---rather than selecting lags to flatter the numbers.
The cross-seed variance of the envelope decays with width as $\mathrm{var}(f)\sim H^{-1.79}$, well beyond the $1/H$ rate expected from i.i.d. population-level concentration (Assumption~\ref{ass:mean_field}).
\begin{table}[h]
\centering
\small
\begin{tabular}{r c c c c}
\toprule
$H$ & $\Delta_\infty$ & $\Delta_2$ & $\Delta_2^{w}$ & $r$ \\
\midrule
16  & 0.0140 & 0.0052 & 0.0044 & 1 \\
32  & 0.0055 & 0.0021 & 0.0018 & 1 \\
64  & 0.0041 & 0.0015 & 0.0013 & 1 \\
128 & 0.0007 & 0.0003 & 0.0002 & 1 \\
256 & 0.0003 & 0.0001 & 0.0001 & 1 \\
512 & ---    & ---    & ---    & ---   \\
\bottomrule
\end{tabular}
\caption{Log-space envelope agreement with the $H=512$
reference (DiagGate).
$\Delta_\infty$ and $\Delta_2$ are the sup-norm and RMS of
$\log_{10}f_H(\ell)-\log_{10}f_{H=512}(\ell)$ over the
35~lags at which $f_{H=512}(\ell)\ge 10^{-3}$.
$\Delta_2^{w}$ is the envelope-weighted log-RMSE summed over
all 35~lags, with weight $f_{H=512}(\ell)$.
$r$ is the Pearson correlation between $\log_{10}f_H$ and
$\log_{10}f_{H=512}$ with lag as the sample axis.
All quantities are evaluated on the seed-averaged
envelope $f_H(\ell)$.
``---'' marks the reference width (self-comparison).}
\label{tab:envelope_logerr}
\end{table}

Two patterns emerge.
First, both diagnostics concentrate with width faster than
their respective law-of-large-numbers baselines:
$\mathrm{std}(\tau_{\mathrm{med}})$ decays faster than the
$H^{-1/2}$ rate expected for independent neurons, and
$\mathrm{var}(f)$ decays faster than the $1/H$ rate.
Trained networks across seeds and widths are therefore
more consistent than independent samples from a common
population $p_\infty$ would predict---a signature of mild
negative inter-neuron correlations induced by training---and,
for the purposes of the present validation, the exponents rule
out any anti-concentration trend at the widths tested.
Second, the envelope is systematically more stable than the
spectrum that generates it, but the magnitude of the
difference depends sensitively on which envelope diagnostic
one reads.
Pearson correlation saturates at $r=1.000$ already at
$H=16$, whereas the absolute-level diagnostics
$\Delta_\infty$ and $\Delta_2$ visibly contract with~$H$,
reaching sub-$10^{-3}$-decade agreement by $H=128$.
The gap between the two readings is the expected behavior.
Pearson correlation in log space is affine-invariant: it ignores any uniform vertical shift or rescaling of the envelope and so reports only how well its lag-to-lag shape matches the reference.
Because the Laplace map from $p_\infty(\tau)$ to $f(\ell)$ is a smoothing integral, this shape is reproduced almost as soon as the spectrum is approximately in place, whereas the overall level of $f(\ell)$ is fixed by the bulk mass of $p_\infty$ and is recovered only as that mass concentrates with width.
The two readings together therefore indicate that the envelope shape is correct at every tested width, while its level converges smoothly with $H$ and lags the saturation of the correlation.
Taken together, the two diagnostics jointly validate the passage from the finite mixture $H^{-1}\sum_q e^{-\ell/\tau_q}$ to the continuous integral $\int e^{-\ell/\tau}p_\infty(\tau)\,d\tau$ that enters the asymptotic analysis~\eqref{eq:envelope_integral}.
The empirical time-scale distribution concentrates around a stable population law~$p_\infty(\tau)$, and the intensive envelope converges to the integral representation in both shape and level as $H$ grows.

\section{Technical structure and modeling scope of the L{\'e}vy generator}
\label{app:levy_generator}

This appendix derives the infinitesimal
generator~\eqref{eq:generator_definition} from the L\'evy-driven
SDE~\eqref{eq:zeta_SDE_mixed}, gives the explicit structure of its
jump component, and discusses the modeling hypotheses under which this
description is natural.
Standard references for the L{\'e}vy process theory used below
include~\citep{applebaum2009levy,sato1999levy}; for the specific
class of tempered stable processes we rely
on~\citet{rosinski2007tempering,kuchler2013tempered}.

\subsection{From the L{\'e}vy-driven SDE to the generator}
\label{app:sde_to_generator}

The SDE~\eqref{eq:zeta_SDE_mixed} gives the sample-path dynamics of a
representative log-decay rate.
To analyze phases, however, we need the stationary density of the
population and, in particular, its far-left-tail behavior.
The generator is the object that makes this passage possible: it
describes the infinitesimal evolution of observables, and its adjoint
gives the forward equation for the density.

Throughout the generator and forward-equation derivations, we take
\(\varphi\in C_c^2(\mathbb R)\): that is, \(\varphi\) is twice
continuously differentiable and has compact support, so there exists
an \(R_0<\infty\) such that \(\varphi(\zeta)=0\) for
\(|\zeta|>R_0\). This compact-support assumption is imposed only on
the auxiliary test function used to identify the forward operator; it
is not an assumption on the density itself. Its role is technical: it
makes the integration-by-parts steps below free of boundary terms.
After the forward equation has been derived, the test function
disappears, and the stationary density~\(\rho_\infty\) may have the
non-compact exponentially decaying far-left tail analyzed in
Appendix~\ref{app:characteristic_derivation}. The only non-compact
probe used below is the separate exponential probe introduced to
identify the L\'evy--Khintchine symbol in
Appendix~\ref{app:jump_operator_symbol}.

For such a test function~$\varphi$, write
$\mathbb{E}_{\zeta}$ for expectation conditional on starting at
$\zeta(0)=\zeta$:
\begin{equation}
\label{eq:appendix_conditional_expectation}
\mathbb{E}_{\zeta}\!\left[\varphi(\zeta(t))\right]
=
\mathbb{E}\!\left[\varphi(\zeta(t))\,\middle|\,\zeta(0)=\zeta\right]
=
\int_{\mathbb{R}}\varphi(u)\,p_t(u\mid \zeta)\,du,
\end{equation}
where $p_t(u\mid \zeta)$ is the transition density.
The infinitesimal generator is the derivative at time zero of the
conditional expectation in~\eqref{eq:appendix_conditional_expectation}:
\begin{equation}
\label{eq:appendix_generator_semigroup_limit}
\mathcal{L}_{\omega}\varphi(\zeta)
=
\left.
\frac{d}{dt}
\mathbb{E}_{\zeta}\!\left[\varphi(\zeta(t))\right]
\right|_{t=0}.
\end{equation}
Under the standard construction of L\'evy-driven SDEs, solutions of
SDEs of the form~\eqref{eq:zeta_SDE_mixed} are Markov jump-diffusions,
so the operator in~\eqref{eq:appendix_generator_semigroup_limit}
governs the infinitesimal evolution of the law of~$\zeta(t)$
\citep{applebaum2009levy,schilling2016introduction}.

The remaining task is to identify this derivative for the concrete
SDE~\eqref{eq:zeta_SDE_mixed}.
The It\^o formula for jump processes is the bookkeeping rule that
does this for a transformed process~$\varphi(\zeta(t))$.
It separates the change in~$\varphi(\zeta(t))$ into a predictable
$dt$ contribution, a Brownian martingale contribution, and a
compensated-jump martingale contribution.
The predictable part is the systematic infinitesimal drift of the
observable; the martingale parts describe zero-mean fluctuations.
Therefore, after taking the conditional expectation
$\mathbb{E}_\zeta$ and differentiating at~$t=0$, only the predictable
$dt$ contribution remains, and that contribution is exactly the
generator applied to~$\varphi$.
For the SDE~\eqref{eq:zeta_SDE_mixed}, It\^o's formula~\citep{applebaum2009levy} gives
\begin{align}
\label{eq:appendix_jump_ito_formula}
d\varphi(\zeta(t))
&=
\Big[
F(\zeta(t-))\,\partial_\zeta\varphi(\zeta(t-))
+
\eta_G\,\partial_{\zeta\zeta}\varphi(\zeta(t-))
\nonumber\\
&\qquad\qquad
+
\mathcal{J}_{\alpha_{\mathrm{jump}},\lambda}[\varphi](\zeta(t-))
\Big]dt
+
\sqrt{2\eta_G}\,\partial_\zeta\varphi(\zeta(t-))\,dW_t
+
dM_t^\varphi .
\end{align}
Here $\zeta(t-)$ denotes the left limit before a possible jump at
time~$t$, $M_t^\varphi$ collects the compensated jump martingale
terms, and the predictable jump contribution is exactly the operator
\eqref{eq:jump_operator_explicit}.
Taking $\mathbb{E}_\zeta$, using that the Brownian and compensated
jump martingales have zero mean, and integrating from~$0$ to~$t$
gives
\begin{align}
\label{eq:appendix_jump_ito_expectation}
\mathbb{E}_{\zeta}\!\left[\varphi(\zeta(t))\right]-\varphi(\zeta)
&=
\mathbb{E}_{\zeta}\!\Biggl[
\int_0^t
\Bigl(
F(\zeta(s-))\,\partial_\zeta\varphi(\zeta(s-))
+
\eta_G\,\partial_{\zeta\zeta}\varphi(\zeta(s-))
\nonumber\\
&\qquad\qquad
+
\mathcal{J}_{\alpha_{\mathrm{jump}},\lambda}[\varphi](\zeta(s-))
\Bigr)\,ds\Biggr] .
\end{align}
Since $\zeta(s-)\to\zeta$ as $s\downarrow 0$ and the test function is
smooth, the integrand in~\eqref{eq:appendix_jump_ito_expectation}
converges to its value at the initial point.
Dividing~\eqref{eq:appendix_jump_ito_expectation} by~$t$ and taking
$t\downarrow 0$ therefore yields
\begin{align}
\label{eq:appendix_jump_ito_drift}
\left.
\frac{d}{dt}
\mathbb{E}_{\zeta}\!\left[\varphi(\zeta(t))\right]
\right|_{t=0}
&=
F(\zeta)\,\partial_\zeta\varphi(\zeta)
+
\eta_G\,\partial_{\zeta\zeta}\varphi(\zeta)
+
\mathcal{J}_{\alpha_{\mathrm{jump}},\lambda}[\varphi](\zeta).
\end{align}
Together with the definition of the generator in
\eqref{eq:appendix_generator_semigroup_limit}, this gives
\eqref{eq:generator_definition}.

\subsection{Jump operator, L{\'e}vy--Khintchine symbol, and tempering}
\label{app:jump_operator_symbol}

In the main text, the nonlocal term
$\mathcal{J}_{\alpha_{\mathrm{jump}},\lambda}[\varphi]$ in the
generator~\eqref{eq:generator_definition} was described as the
compensated L\'evy jump operator that accounts for the nonlocal
redistribution of $\zeta$ caused by jumps.
We now give its explicit form.

Let $\nu_{\alpha_{\mathrm{jump}},\lambda}(dy)$ denote the tempered
L{\'e}vy measure~\eqref{eq:tempered_levy_measure}.
The compensated jump operator acting on $\varphi$ is
\begin{equation}
\label{eq:jump_operator_explicit}
\mathcal{J}_{\alpha_{\mathrm{jump}},\lambda}[\varphi](\zeta)
=
\int_{\mathbb{R}\setminus\{0\}}
\left(
\varphi(\zeta+y)
-
\varphi(\zeta)
-
y\,\mathbf{1}_{|y|<1}\,\partial_\zeta \varphi(\zeta)
\right)
\nu_{\alpha_{\mathrm{jump}},\lambda}(dy).
\end{equation}
The integrand has a direct reading: the difference
$\varphi(\zeta+y)-\varphi(\zeta)$ measures the net effect of a jump
of size~$y$ on the test function, and the integral aggregates these
contributions over all possible jump sizes weighted by their frequency
given by the L\'evy measure
$\nu_{\alpha_{\mathrm{jump}},\lambda}(dy)$.
The subtraction $y\,\mathbf{1}_{|y|<1}\,\partial_\zeta\varphi(\zeta)$
is the standard L\'evy--Khintchine
compensation~\citep{applebaum2009levy,sato1999levy}: when
$\alpha_{\mathrm{jump}}\ge 1$, small jumps are so frequent that the
raw integral
$\int[\varphi(\zeta+y)-\varphi(\zeta)]\,
\nu_{\alpha_{\mathrm{jump}},\lambda}(dy)$ would diverge.
The compensation removes the first-order (linear) effect of small
jumps under the standard truncation convention, leaving a convergent
nonlocal contribution.

Because the L\'evy measure depends on the jump size~$y$ but not on
the current position~$\zeta$, the
operator~\eqref{eq:jump_operator_explicit} is translation-invariant.
This is the step that makes the nonlocal term usable in the tail
analysis: exponentials are diagonalized by translation-invariant
operators. Fix a real~$\xi$ in the domain where the integral below is
finite (for the tempered measure, this includes
$|\xi|\le \lambda$).  To compute the symbol, evaluate the operator on
the exponential probe
\begin{equation}
\label{eq:appendix_exponential_probe}
\varphi_\xi(\zeta)=e^{\xi\zeta},
\qquad
\varphi_\xi(\zeta+y)=e^{\xi y}\varphi_\xi(\zeta),
\qquad
\partial_\zeta\varphi_\xi(\zeta)=\xi\varphi_\xi(\zeta).
\end{equation}
The compact-support restriction used above is not needed for this
separate eigenfunction calculation: here the goal is not to derive the
forward equation by integration by parts, but to identify the
L\'evy--Khintchine symbol of the translation-invariant jump operator.
Every term in the integrand of
\eqref{eq:jump_operator_explicit} is therefore a scalar multiple of
the same exponential~$\varphi_\xi(\zeta)$.
Substituting~\eqref{eq:appendix_exponential_probe} into
\eqref{eq:jump_operator_explicit} gives
\begin{align}
\label{eq:levy_symbol_definition}
\mathcal{J}_{\alpha_{\mathrm{jump}},\lambda}[\varphi_\xi](\zeta)
&=
\int_{\mathbb{R}\setminus\{0\}}
\left(
e^{\xi(\zeta+y)}
-
e^{\xi\zeta}
-
y\,\mathbf 1_{|y|<1}\xi e^{\xi\zeta}
\right)
\nu_{\alpha_{\mathrm{jump}},\lambda}(dy)
\nonumber\\
&=
e^{\xi\zeta}
\underbrace{
\int_{\mathbb{R}\setminus\{0\}}
\left(
e^{\xi y}
-
1
-
\xi y\,\mathbf{1}_{|y|<1}
\right)
\nu_{\alpha_{\mathrm{jump}},\lambda}(dy)
}_{\displaystyle \Psi(\xi;\omega)} .
\end{align}
The scalar factor $\Psi(\xi;\omega)$ is the real-exponential
L\'evy--Khintchine symbol~\citep{applebaum2009levy,schilling2016introduction}.
Thus
\begin{equation}
\label{eq:appendix_jump_exponential_eigenfunction}
\mathcal{J}_{\alpha_{\mathrm{jump}},\lambda}[e^{\xi\,\cdot\,}](\zeta)
=
\Psi(\xi;\omega)\,e^{\xi\zeta}.
\end{equation}
This property is what makes the generator formalism analytically
powerful: on exponential probes, the nonlocal jump term becomes
multiplication by~$\Psi(\xi;\omega)$ instead of an integral operator.
In Appendix~\ref{app:characteristic_derivation}, the same
diagonalization principle is applied to the adjoint forward equation:
once the far-left-tail stationary density is written in exponential
form, the drift, diffusion, and jump terms all become scalar multiples
of that exponential.
The stationary balance then reduces to the algebraic characteristic equation analyzed in Appendix~\ref{app:nonlocal_balance}, where uniqueness of the spectral exponent is established and the explicit jump-amplitude threshold is derived.

\subsection{Modeling scope and justification}
\label{app:modeling_scope}

\paragraph{Empirical motivation for heavy-tailed forcing.}
Recent empirical work provides converging evidence that heavy-tailed statistics are
pervasive in trained deep networks. On the optimization side, minibatch gradient
noise exhibits both a diffusive background and intermittent heavy-tailed
excursions~\citep{simsekli2019tail,gurbuzbalaban2021heavy}. On the parameter side,
the empirical spectral densities of trained weight matrices develop heavy-tailed,
approximately power-law bulks, a signature of implicit self-regularization
characterized through random matrix theory~\citep{martin2021implicit}.
This evidence motivates a generator that accommodates both regimes within a single framework.
Heavy-tailed gradient noise and heavy-tailed weight-matrix spectra are the
best-documented contributors in the deep-learning literature, but the generator
does not single out any one of them.
At the mesoscopic scale, $\alpha_{\mathrm{jump}}$ summarizes the aggregated
heavy-tailed fingerprint of the effective forcing on $\zeta$, irrespective of
whether its dominant microscopic source is gradient fluctuations, heavy-tailed
weight spectra, data or pre-activation noise, or some combination of these.

\paragraph{Structural modeling closure.}
The L\'evy-driven generator should be read as a coarse-grained Markov closure for many small and intermittent increments in~$\zeta$ whose aggregate law is stable-like at the mesoscopic scale. It is not a microscopic description of learning dynamics.
Rather, it is the minimal stochastic model that lets us represent diffusion, heavy-tailed jump forcing, and finite-size tempering in one analytically tractable and justifiable model.

\paragraph{Tempering regime.}
The measure~\eqref{eq:tempered_levy_measure} is that of a symmetric
tempered $\alpha$-stable L\'evy process: the small-jump density
follows the symmetric $\alpha$-stable law
$|y|^{-1-\alpha_{\mathrm{jump}}}\,dy$ on scales
$|y|\ll 1/\lambda$, and the exponential factor $e^{-\lambda|y|}$
suppresses jumps on scales $|y|\gg 1/\lambda$.
This implements, at the level of the generator, the fact that
arbitrarily large jumps in log-decay rates cannot occur in any
finite-size system with bounded activations, finite hidden
dimension, and finite training horizon.
The tempering parameter $\lambda$ encodes the scale at which that finite-size cutoff sets in.
The process retains infinite activity---almost every sample path contains infinitely many (mostly small) jumps in any finite time interval---while moderate-scale increments remain heavy-tailed and arbitrarily large excursions are exponentially suppressed, so tempered stable processes interpolate between $\alpha$-stable behavior at short scales and Gaussian behavior at long scales~\citep{rosinski2007tempering}.

Analytically, tempering makes the compensated L\'evy--Khintchine symbol finite up to the boundary $\beta=\lambda$.
This finite endpoint value underlies the root classification and explicit jump-amplitude threshold derived in Appendix~\ref{app:existence_uniqueness}.

\paragraph{Applications of tempered stable processes.}
The same class of tempered stable generators has been deployed as a modeling tool in several unrelated domains where heavy-tailed fluctuations coexist with finite-size cutoffs.
In mathematical finance, the CGMY model of \citet{carr2002fine} uses tempered stable L\'evy processes to capture the jump structure of asset log-returns and resolves volatility-smile anomalies that Gaussian models cannot; this construction is now standard in the broader L\'evy-process literature on financial modeling~\citep{cont2004financial}.
In hydrology, \citet{meerschaert2008tempered} introduced tempered anomalous diffusion to describe non-Fickian tracer transport in heterogeneous aquifers, where exponential tempering captures the natural cutoff of retention times in the porous medium.
In statistical physics, \citet{stanislavsky2008diffusion} derived anomalous-diffusion and non-exponential-relaxation laws from inverse tempered $\alpha$-stable subordinators, showing that tempering produces relaxation behavior intermediate between Cole--Cole subdiffusion and classical exponential decay.
Across these domains, the tempered stable generator plays the same structural role it plays here: a principled interpolation between a heavy-tailed short-scale regime and a regularized long-scale regime, with the tempering parameter encoding the scale at which finite-size effects restore tameness.

\subsection{Near-equilibrium drift linearization}
\label{app:drift_structure}

In Section~\ref{sec:stochastic_modeling}, we used a near-equilibrium
linearization~\eqref{eq:drift_linearization} of the effective drift~$F(\zeta)$ around the bulk
stable zero~$\zeta^*$. This appendix gives its full derivation. The complementary far-left-tail
closure~\eqref{eq:drift_saturation} is validated empirically in Section~\ref{sec:exp_access_route}.

Assuming $F$ is smooth in a neighborhood of~$\zeta^*$, Taylor-expanding
the effective drift around the bulk equilibrium yields
\begin{equation}
\label{eq:bulk_drift_linearization}
F(\zeta)
\;=\;
F(\zeta^*) + F'(\zeta^*)\,(\zeta-\zeta^*) + O\!\big((\zeta-\zeta^*)^2\big)
\;\approx\;
-\gamma\,(\zeta-\zeta^*),
\qquad
\gamma = -F'(\zeta^*)>0,
\end{equation}
where the zeroth-order term vanishes because $F(\zeta^*)=0$ by definition of~$\zeta^*$.
The minus sign in~$-\gamma\,(\zeta-\zeta^*)$ is chosen so that $\gamma>0$ plays the role of a local restoring rate: since $\zeta^*$ is a stable zero, $F'(\zeta^*)<0$, and defining $\gamma=-F'(\zeta^*)$ yields the standard mean-reverting form.
Substituting~\eqref{eq:bulk_drift_linearization} into the drift--diffusion--jump SDE~\eqref{eq:zeta_SDE_mixed} shows that, in a neighborhood of~$\zeta^*$, the dynamics reduce to a locally mean-reverting process whose Gaussian component corresponds to an OU approximation with equilibrium level~$\zeta^*$ and rate~$\gamma$.

This type of local linearization around a stable equilibrium is a standard approximation in stochastic descriptions of learning dynamics.
In particular, \citet{mandt2017sgd} show that, under a quadratic approximation of the loss near an optimum, constant-step SGD can be approximated by a multivariate OU process, yielding a Gaussian stationary distribution.
More broadly, statistical-physics descriptions of high-dimensional learning dynamics often reduce the evolution to effective stochastic equations for low-dimensional collective variables or order parameters.
For example, \citet{Bonnaire_2024} formulate an effective one-dimensional stochastic dynamics in potential form for a collective variable in Tensor-PCA, while \citet{gerbelot2024rigorous} rigorously derive closed effective Gaussian stochastic dynamics for a finite set of order parameters in high-dimensional stochastic optimization.

These results support linearizing the restoring drift near a populated stable equilibrium. This local approximation applies on either side of the jump-amplitude threshold, which concerns the far-left tail.
When $\eta_J=0$ and the restoring drift is globally linear, the stationary log-rate distribution is Gaussian. This is the OU representative used for the collapsed-envelope calculation in Section~\ref{sec:collapsed_regime} and Appendix~\ref{app:exponential_envelope_full_details}.

\subsection{A Gaussian-confining null and its falsifiability}
\label{app:gaussian_null}

To distinguish the saturated-drift mechanism~\eqref{eq:drift_saturation} from stronger confinement, we consider a null model combining Gaussian-only forcing with a confining potential that grows faster than linearly.
We call this the Gaussian-confining null. The result below excludes a regularly varying stationary time-scale tail; it does not exclude a broad distribution over a finite observational window.

Specifically, set $\eta_J=0$ in~\eqref{eq:zeta_SDE_mixed} and assume
\begin{equation}
\label{eq:confining_drift_null}
F(\zeta)=-V'(\zeta),
\qquad
\frac{V(\zeta)}{|\zeta|}\longrightarrow\infty
\quad\text{as }|\zeta|\to\infty.
\end{equation}
This includes the OU case, with $V(\zeta)\propto(\zeta-\zeta^*)^2$, as a representative example.
The growth condition applies to the potential, not to the drift: it excludes the saturated far-left closure~\eqref{eq:drift_saturation}, whose corresponding potential satisfies $V(\zeta)\sim-\kappa\zeta$ as $\zeta\to-\infty$.
Under~\eqref{eq:confining_drift_null} with $\eta_J=0$, the stationary forward equation \eqref{eq:appendix_stationary_forward_equation} reduces to the reversible Fokker--Planck equation
\begin{equation*}
\label{eq:gaussian_null_fokker_planck}
0 =
-\partial_\zeta\!\bigl(F(\zeta)\,\rho_\infty\bigr)
+\eta_G\,\partial_{\zeta\zeta}\rho_\infty,
\end{equation*}
whose stationary solution is the explicit Gibbs form
\begin{equation}
\label{eq:gaussian_null_gibbs}
\rho_\infty(\zeta)
=
Z^{-1}\exp\!\left(-\frac{V(\zeta)}{\eta_G}\right),
\qquad
Z=\int_\mathbb{R}\exp\!\left(-\frac{V(\zeta')}{\eta_G}\right)d\zeta'<\infty,
\end{equation}
where finiteness of~$Z$ is a consequence of the super-linear growth of~$V$.
This diffusion-only stationary solution satisfies detailed balance.
The full jump-driven model used for anti-collapse is stationary through nonlocal drift--jump balance and is not assumed to satisfy detailed balance.

\begin{proposition}[No anti-collapse under the Gaussian-confining null]
\label{prop:gaussian_null}
Assume a drift of the form~\eqref{eq:confining_drift_null} and $\eta_J=0$. Then, the stationary log-rate density $\rho_\infty$~\eqref{eq:gaussian_null_gibbs} decays faster than any exponential in the far-left tail. That is, for every $\beta>0$,
\begin{equation}
\label{eq:gaussian_null_superexp}
\lim_{\zeta\to-\infty}
\rho_\infty(\zeta)\,e^{-\beta\zeta}=0.
\end{equation}
The induced time-scale density $p_\infty(\tau)$ is therefore not regularly varying at infinity, and the envelope $f(\ell)$ is not power-law in~$\ell$.
\end{proposition}

\begin{proof}
Fix $\beta>0$. The confining assumption~\eqref{eq:confining_drift_null} states that
$V(\zeta)/|\zeta|\to\infty$ as $\zeta\to-\infty$. Applying the definition of this
divergence with threshold $M=\beta\eta_G+1$, there exists $\zeta_0<0$ such that
\begin{equation}
\frac{V(\zeta)}{|\zeta|}\;\ge\;\beta\eta_G+1,
\qquad\text{i.e.}\qquad
V(\zeta)\;\ge\;(\beta\eta_G+1)\,|\zeta|,
\qquad\text{for all }\zeta\le\zeta_0 .
\end{equation}
Dividing by $\eta_G>0$ and splitting the constant,
\begin{equation}
\frac{V(\zeta)}{\eta_G}
\;\ge\;
\frac{(\beta\eta_G+1)\,|\zeta|}{\eta_G}
\;=\;
\left(\beta+\frac{1}{\eta_G}\right)|\zeta|
\;=\;
\beta|\zeta|+\frac{|\zeta|}{\eta_G},
\qquad \zeta\le\zeta_0 .
\end{equation}
For $\zeta<0$ we have $\zeta=-|\zeta|$, hence $-\beta\zeta=\beta|\zeta|$.
Substituting the Gibbs form~\eqref{eq:gaussian_null_gibbs} and using the bound above, for $\zeta\le\zeta_0,\eta_G>0$ we have:
\begin{equation}
\begin{aligned}
\rho_\infty(\zeta)\,e^{-\beta\zeta}
&=
Z^{-1}\exp\!\left(-\frac{V(\zeta)}{\eta_G}\right)e^{\beta|\zeta|}
=
Z^{-1}\exp\!\left(-\frac{V(\zeta)}{\eta_G}+\beta|\zeta|\right)\\[2pt]
&\le
Z^{-1}\exp\!\left(-\left(\beta|\zeta|+\frac{|\zeta|}{\eta_G}\right)+\beta|\zeta|\right)
=
Z^{-1}\exp\!\left(-\frac{|\zeta|}{\eta_G}\right)
\;\xrightarrow[\;\zeta\to-\infty\;]{}\;0.
\end{aligned}
\end{equation}
Since this holds for every $\beta>0$, $\rho_\infty$ decays super-exponentially in the far-left tail.
The change of variable $\tau=e^{-\zeta}$ transfers the far-left super-exponential decay into super-polynomial decay of $p_\infty(\tau)$ at large~$\tau$, which lies outside the regularly varying class required by Proposition~\ref{prop:tauberian} for a power-law envelope.
\end{proof}

\section{Nonlocal tail balance and characteristic equation}
\label{app:nonlocal_balance}

This appendix derives the characteristic equation $\Phi(\beta;\omega)=0$ for the root-selected branch of anti-collapse and the jump-amplitude threshold $\eta_J^*$ above which its interior root determines the spectral exponent~$\beta$.

\subsection{Far-left-tail reduction and the characteristic equation}
\label{app:characteristic_derivation}

Throughout this subsection, we assume that $\rho(\cdot,t)$ and $\rho_\infty$ are twice differentiable in~$\zeta$, integrable, and decay sufficiently fast at infinity for the integration-by-parts boundary terms to vanish; the compensated jump integrals below are also assumed finite.

\paragraph{From the generator to the forward equation.}
For a time-homogeneous Markov process with generator $\mathcal L_\omega$, the
population density $\rho(\zeta,t)$ obeys the Kolmogorov forward (Fokker--Planck)
equation $\partial_t\rho=\mathcal L_\omega^{*}\rho$, where $\mathcal L_\omega^{*}$
is the formal adjoint of $\mathcal L_\omega$ \eqref{eq:generator_definition} defined by the duality
\begin{equation}
\label{eq:forward_adjoint_duality}
\int_{\mathbb R}\left(\mathcal L_\omega\varphi(\zeta)\right)\rho(\zeta,t)\ d\zeta
=
\int_{\mathbb R}\varphi(\zeta)\left(\mathcal L_\omega^{*}\rho(\zeta,t)\right)\,d\zeta
\end{equation}
for smooth compactly supported test functions $\varphi$.
This is the standard generator--adjoint route to the forward equation for L\'evy-driven Markov processes~\citep{applebaum2009levy,sato1999levy,schilling2016introduction}.

Below, we derive the adjoint $\mathcal L_\omega^{*}$ explicitly.
Because it is the sum of a drift, a diffusion, and a jump term, the adjoint splits linearly,
\begin{align}
\label{eq:appendix_generator_integral_split}
\int_{\mathbb R}
\mathcal L_\omega\varphi(\zeta)\rho(\zeta,t)\,d\zeta
&=
\underbrace{\int_{\mathbb R}
F(\zeta)\,\partial_\zeta\varphi(\zeta)\rho(\zeta,t)\,d\zeta}_{\mathrm{drift}}
\nonumber\\
&\quad+
\underbrace{\int_{\mathbb R}\eta_G\,\partial_{\zeta\zeta}\varphi(\zeta)\rho(\zeta,t)\,d\zeta}_{\mathrm{diffusion}}
\nonumber\\
&\quad+
\underbrace{\int_{\mathbb R} \mathcal J_{\alpha_{\mathrm{jump}},\lambda}[\varphi](\zeta) \rho(\zeta,t)\,d\zeta}_{\mathrm{jump}}.
\end{align}
We compute the adjoint of each of these three pieces separately and then add the results.

In the following, given two functions $\varphi$ and~$g$, the notation
$[\varphi g]_{-\infty}^{+\infty}$ means
\begin{equation}
\label{eq:appendix_boundary_limits}
\lim_{R\to\infty}\varphi(R)g(R,t)
-
\lim_{R\to\infty}\varphi(-R)g(-R,t).
\end{equation}
The elementary integration-by-parts identity used below is
\begin{equation}
\label{eq:appendix_integration_by_parts_rule}
\int_{\mathbb R}
\partial_\zeta\varphi(\zeta)\,g(\zeta,t)\,d\zeta
=
\Big[\varphi(\zeta)g(\zeta,t)\Big]_{-\infty}^{+\infty}
-
\int_{\mathbb R}
\varphi(\zeta)\,\partial_\zeta g(\zeta,t)\,d\zeta .
\end{equation}
Thus an integration by parts moves one derivative from the test
function~\(\varphi\) onto the density-weighted factor~\(g\), at the
cost of a boundary term. Compact support of~\(\varphi\) will make the
boundary terms vanish in the identities below.

\begin{itemize}
\item 
For the drift term, set $g(\zeta,t)=F(\zeta)\rho(\zeta,t)$.
Applying~\eqref{eq:appendix_integration_by_parts_rule} once gives
\begin{equation}
\label{eq:appendix_drift_integration_by_parts}
\int_{\mathbb R}
F(\zeta)\,\partial_\zeta\varphi(\zeta)\rho(\zeta,t)\,d\zeta
=
\int_{\mathbb R}
\partial_\zeta\varphi(\zeta)\,g(\zeta,t)\,d\zeta
=
\Big[\varphi(\zeta)g(\zeta,t)\Big]_{-\infty}^{+\infty}
-
\int_{\mathbb R}
\varphi(\zeta)\,\partial_\zeta g(\zeta,t)\,d\zeta .
\end{equation}
Because~$\varphi$ is compactly supported, there is an
$R_0<\infty$ such that $\varphi(\zeta)=0$ whenever
$|\zeta|>R_0$. Hence both limits in
$[\varphi g]_{-\infty}^{+\infty}$ are zero, so the boundary term
vanishes. Substituting back $g=F\rho$ gives
\begin{equation}
\label{eq:appendix_drift_adjoint_identity}
\int_{\mathbb R}
F(\zeta)\,\partial_\zeta\varphi(\zeta)\rho(\zeta,t)\,d\zeta
=
-\int_{\mathbb R}
\varphi(\zeta)\,\partial_\zeta\!\big(F(\zeta)\rho(\zeta,t)\big)\,d\zeta,
\end{equation}
which is the drift contribution to the forward operator.

\item
For the diffusion term, we apply
\eqref{eq:appendix_integration_by_parts_rule} twice: first to move one
derivative from~\(\partial_{\zeta\zeta}\varphi\) onto~\(\rho\), and
then again to move the remaining derivative from~\(\partial_\zeta
\varphi\) onto~\(\rho\). This gives
\begin{align}
\label{eq:appendix_diffusion_integration_by_parts}
\int_{\mathbb R}
\eta_G\,\partial_{\zeta\zeta}\varphi(\zeta)\rho(\zeta,t)\,d\zeta
&=
\eta_G
\int_{\mathbb R}
\partial_\zeta\!\big(\partial_\zeta\varphi(\zeta)\big)
\rho(\zeta,t)\,d\zeta
\nonumber\\
&=
\eta_G\Big[\partial_\zeta\varphi(\zeta)\rho(\zeta,t)\Big]_{-\infty}^{+\infty}
-
\int_{\mathbb R}
\eta_G\,\partial_\zeta\varphi(\zeta)\partial_\zeta\rho(\zeta,t)\,d\zeta
\nonumber\\
&=
\eta_G\Big[\partial_\zeta\varphi(\zeta)\rho(\zeta,t)\Big]_{-\infty}^{+\infty}
-
\eta_G
\left(
\Big[\varphi(\zeta)\partial_\zeta\rho(\zeta,t)\Big]_{-\infty}^{+\infty}
-
\int_{\mathbb R}
\varphi(\zeta)\partial_{\zeta\zeta}\rho(\zeta,t)\,d\zeta
\right)
\nonumber\\
&=
\eta_G\Big[\partial_\zeta\varphi(\zeta)\rho(\zeta,t)\Big]_{-\infty}^{+\infty}
-
\eta_G\Big[\varphi(\zeta)\partial_\zeta\rho(\zeta,t)\Big]_{-\infty}^{+\infty}
+
\int_{\mathbb R}
\varphi(\zeta)\,\eta_G\,\partial_{\zeta\zeta}\rho(\zeta,t)\,d\zeta .
\end{align}
Both boundary terms vanish because $\varphi$ and
$\partial_\zeta\varphi$ are compactly supported. Therefore
\begin{equation}
\label{eq:appendix_diffusion_adjoint_identity}
\int_{\mathbb R}
\eta_G\,\partial_{\zeta\zeta}\varphi(\zeta)\rho(\zeta,t)\,d\zeta
=
\int_{\mathbb R}
\varphi(\zeta)\,\eta_G\,\partial_{\zeta\zeta}\rho(\zeta,t)\,d\zeta.
\end{equation}

\item 
For the jump term, the analogue of integration by parts is to move the
shift and the compensator from the test function to the density.
We begin by expanding the backward jump operator
\eqref{eq:jump_operator_explicit} inside the density-weighted integral:
\begin{equation}
\begin{aligned}
\label{eq:appendix_jump_adjoint_expansion}
\int_{\mathbb R}
\mathcal{J}_{\alpha_{\mathrm{jump}},\lambda}[\varphi](\zeta)
\rho(\zeta,t)\,d\zeta
&=
\int_{\mathbb R\setminus\{0\}}
\Bigg[
\int_{\mathbb R}
\varphi(\zeta+y)\rho(\zeta,t)\,d\zeta
-
\int_{\mathbb R}
\varphi(\zeta)\rho(\zeta,t)\,d\zeta \\
&\qquad\qquad
-
y\,\mathbf 1_{|y|<1}
\int_{\mathbb R}
\partial_\zeta\varphi(\zeta)\rho(\zeta,t)\,d\zeta
\Bigg]
\nu_{\alpha_{\mathrm{jump}},\lambda}(dy).
\end{aligned}
\end{equation}
The first inner integral is shifted by the change of variables
$x=\zeta+y$; this converts a destination value of the test function
into a source value of the density. The compensator term is moved from
$\partial_\zeta\varphi$ to~$\rho$ by integration by parts:
\begin{equation}
\label{eq:appendix_jump_shift_id}
\int_{\mathbb R}\varphi(\zeta+y)\rho(\zeta,t)\,d\zeta
=
\int_{\mathbb R}\varphi(x)\rho(x-y,t)\,dx,
\end{equation}
\begin{equation}
\label{eq:appendix_jump_compensator_id}
\int_{\mathbb R}\partial_\zeta\varphi(\zeta)\rho(\zeta,t)\,d\zeta
=
-\int_{\mathbb R}\varphi(\zeta)\partial_\zeta\rho(\zeta,t)\,d\zeta .
\end{equation}
Substituting~\eqref{eq:appendix_jump_shift_id} and~\eqref{eq:appendix_jump_compensator_id} into~\eqref{eq:appendix_jump_adjoint_expansion}, renaming the dummy integration variable back to~$\zeta$, and collecting the terms multiplying~$\varphi(\zeta)$ gives the first equality below.
The second equality only changes the order of the $\zeta$- and $y$-integrals:
\begin{equation}
\begin{aligned}
\label{eq:appendix_jump_collect_terms}
\int_{\mathbb R}
\mathcal{J}_{\alpha_{\mathrm{jump}},\lambda}[\varphi](\zeta)
\rho(\zeta,t)\,d\zeta
&=
\int_{\mathbb R\setminus\{0\}}
\int_{\mathbb R}
\varphi(\zeta)
\Big(
\rho(\zeta-y,t)
-
\rho(\zeta,t)
\nonumber\\
&\qquad\qquad\qquad\qquad
+
y\,\partial_\zeta\rho(\zeta,t)\,\mathbf 1_{|y|<1}
\Big)
d\zeta\,
\nu_{\alpha_{\mathrm{jump}},\lambda}(dy)
\nonumber\\
&=
\int_{\mathbb R}
\varphi(\zeta)
\underbrace{\left[
\int_{\mathbb R\setminus\{0\}}
\Big(
\rho(\zeta-y,t)
-
\rho(\zeta,t)
+
y\,\partial_\zeta\rho(\zeta,t)\,\mathbf 1_{|y|<1}
\Big)
\nu_{\alpha_{\mathrm{jump}},\lambda}(dy)
\right]}_{\mathcal I^*_{\alpha_{\mathrm{jump}},\lambda}[\rho](\zeta,t)}
d\zeta .
\end{aligned}
\end{equation}
The bracketed term is the adjoint jump contribution. Thus,
\begin{equation}
\label{eq:appendix_jump_adjoint_identity}
\int_{\mathbb R}
\mathcal{J}_{\alpha_{\mathrm{jump}},\lambda}[\varphi](\zeta)
\rho(\zeta,t)\,d\zeta
=
\int_{\mathbb R}
\varphi(\zeta)\,
\mathcal I^*_{\alpha_{\mathrm{jump}},\lambda}[\rho](\zeta,t)\,d\zeta,
\end{equation}
where
\begin{equation}
\label{eq:appendix_adjoint_jump_operator}
\mathcal I^*_{\alpha_{\mathrm{jump}},\lambda}[\rho](\zeta,t)
=
\int_{\mathbb R\setminus\{0\}}
\Big(
\rho(\zeta-y,t)
-
\rho(\zeta,t)
+
y\,\partial_\zeta\rho(\zeta,t)\,\mathbf 1_{|y|<1}
\Big)
\nu_{\alpha_{\mathrm{jump}},\lambda}(dy)
\end{equation}
is the forward jump operator.
\end{itemize}

The backward operator~\eqref{eq:jump_operator_explicit} shifts the test
function forward by~$y$, while the forward
operator~\eqref{eq:appendix_adjoint_jump_operator} shifts the density
backward by~$y$. The opposite shift
directions follow from a simple probabilistic asymmetry between
the two pictures: the backward operator asks what happens to a
particle currently at~$\zeta$ when a jump of size~$y$ occurs, and
the particle ends up at the destination~$\zeta+y$; the forward
operator asks where the probability mass observed at~$\zeta$
came from, and mass arriving at~$\zeta$ must have started at the
source~$\zeta-y$. The shift on the test function therefore tracks
where the particle is going, while the shift on the density
tracks where the particle came from. The compensator
$y\,\partial_\zeta\rho(\zeta,t)\,\mathbf 1_{|y|<1}$ is the adjoint of the
compensator in~\eqref{eq:jump_operator_explicit}: the minus sign in
the backward operator becomes a plus sign after integration by parts.
It plays the same role as in the backward generator, namely making the small-jump integral converge for $\alpha_{\mathrm{jump}}\in(1,2)$.

The three derivations above give the pieces of the adjoint of the generator
\begin{equation}
\label{eq:appendix_adjoint_generator}
\mathcal L_\omega^{*}\rho(\zeta,t)
=
-\partial_\zeta\!\big(F(\zeta)\rho(\zeta,t)\big)
+\eta_G\,\partial_{\zeta\zeta}\rho(\zeta,t)
+\mathcal I^*_{\alpha_{\mathrm{jump}},\lambda}[\rho](\zeta,t).
\end{equation}
Substituting the explicit adjoint~\eqref{eq:appendix_adjoint_generator} into the identity $\partial_t\rho=\mathcal L_\omega^{*}\rho$ gives the forward equation~\eqref{eq:general_forward_equation} used in the main text.

\paragraph{Stationary far-left-tail reduction and characteristic equation.}
The forward equation~\eqref{eq:general_forward_equation} is the density equation needed for the tail analysis.
Assumptions~\ref{ass:quasi_stationarity} and~\ref{ass:mean_field} supply the deterministic quasi-stationary population description used here.
The purpose of the reduction is to characterize the admissible far-left behavior of that density; it is not a separate existence-and-uniqueness theorem for an invariant measure of the full state-dependent SDE.
A positive exponent $\beta$ makes the candidate tail $\rho_\infty(\zeta)\sim e^{\beta\zeta}$ integrable as $\zeta\to-\infty$ and induces the integrable time-scale tail $p_\infty(\tau)\sim\tau^{-1-\beta}$.
Proposition~\ref{lem:Phi_properties} determines when the characteristic equation has an interior root $0<\beta<\lambda$.

In the late-training quasi-stationary regime, $\partial_t\rho\approx0$, so \eqref{eq:general_forward_equation} reduces to the stationary balance $\mathcal L_\omega^*\rho_\infty=0$.
Substituting $\rho=\rho_\infty$ into the explicit adjoint \eqref{eq:appendix_adjoint_generator} then gives the stationary nonlocal Fokker--Planck equation:
\begin{equation}
\label{eq:appendix_stationary_forward_equation}
0
=
-\partial_\zeta\!\big(F(\zeta)\rho_\infty(\zeta)\big)
+
\eta_G \partial_{\zeta\zeta}\rho_\infty(\zeta)
+
\mathcal I^*_{\alpha_{\mathrm{jump}},\lambda}[\rho_\infty](\zeta).
\end{equation}

In the far-left tail ($\zeta\to -\infty$), we use the constant-drift closure $F(\zeta)=\kappa+o(1)$ from Eq.~\ref{eq:drift_saturation}.
Starting from the stationary forward equation~\eqref{eq:appendix_stationary_forward_equation}, the far-left-tail balance becomes asymptotically translation invariant:
\begin{equation}
\label{eq:appendix_tail_balance_equation}
0
=
- \kappa\,\partial_\zeta \rho_\infty(\zeta)
+
\eta_G\,\partial_{\zeta\zeta}\rho_\infty(\zeta)
+
\mathcal I^*_{\alpha_{\mathrm{jump}},\lambda}[\rho_\infty](\zeta).
\end{equation}

Because the far-left-tail balance~\eqref{eq:appendix_tail_balance_equation}
has constant coefficients, exponentials are eigenfunctions of every
term: drift, diffusion, and nonlocal jump operator alike.
We therefore seek an asymptotic tail of the form
\begin{equation}
\label{eq:appendix_exponential_tail_ansatz}
\rho_\infty(\zeta)\sim c\,e^{\beta\zeta},
\qquad \beta>0,
\ \zeta\to-\infty,
\end{equation}
where $c>0$ is a free amplitude.

Before carrying out this substitution term by term, we verify that the leading contribution of the nonlocal jump operator is governed by the same far-left-tail region.
For fixed jump sizes, both $\zeta$ and $\zeta-y$ remain in the far-left tail as $\zeta\to-\infty$, so the exponential substitution applies directly. A jump large enough to connect a fixed bulk region to~$\zeta$ must instead have magnitude of order~$|\zeta|$.
The tempered measure suppresses such a contribution at the rate $e^{-\lambda|\zeta|}$, up to a polynomial factor, whereas an interior tail mode decays as $e^{-\beta|\zeta|}$ with $\beta<\lambda$.
For $\beta<\lambda$, the factor $e^{-(\lambda-\beta)|\zeta|}$ makes the contribution from jumps arriving directly from the bulk negligible relative to the candidate tail.
This justifies evaluating the leading jump balance on the exponential tail mode.

At the positive threshold $\eta_J=\eta_J^*>0$, the characteristic equation has the unique positive root $\beta=\lambda$.
The candidate tail and the contribution from jumps arriving directly from the bulk then have the same exponential rate $e^{-\lambda|\zeta|}$.
This comparison therefore leaves open whether the stationary time-scale tail is regularly varying and how the corresponding envelope decays.

Having established this asymptotic separation for interior modes, we substitute~\eqref{eq:appendix_exponential_tail_ansatz} into each term of~\eqref{eq:appendix_tail_balance_equation}.
The drift and diffusion terms give:
\begin{equation}
\begin{aligned}
\label{eq:appendix_drift_diffusion_on_exponential_tail}
-\kappa\,\partial_\zeta\!\left(c\,e^{\beta\zeta}\right)
&=
-\kappa\beta\,c\,e^{\beta\zeta}, \\
\eta_G\,\partial_{\zeta\zeta}\!\left(c\,e^{\beta\zeta}\right)
&=
\eta_G\beta^2\,c\,e^{\beta\zeta}.
\end{aligned}
\end{equation}
For the jump term, substituting the same exponential tail factor into the forward jump operator~\eqref{eq:appendix_adjoint_jump_operator} gives:
\begin{align}
\label{eq:appendix_adjoint_jump_on_exponential_tail}
\mathcal I^*_{\alpha_{\mathrm{jump}},\lambda}
\!\left[c\,e^{\beta\,\cdot}\right](\zeta)
&=
\int_{\mathbb R\setminus\{0\}}
\Big(
c\,e^{\beta(\zeta-y)}
-
c\,e^{\beta\zeta}
+
y\,\beta c\,e^{\beta\zeta}\mathbf 1_{|y|<1}
\Big)
\nu_{\alpha_{\mathrm{jump}},\lambda}(dy)
\nonumber\\
&=
c\,e^{\beta\zeta}
\int_{\mathbb R\setminus\{0\}}
\Big(
e^{-\beta y}
-1
+
\beta y \mathbf 1_{|y|<1}
\Big)
\nu_{\alpha_{\mathrm{jump}},\lambda}(dy).
\end{align}
The scalar integral in~\eqref{eq:appendix_adjoint_jump_on_exponential_tail} is the adjoint jump symbol:
\begin{equation}
\label{eq:appendix_adjoint_jump_symbol}
\Psi^*(\beta;\omega)
=
\int_{\mathbb R\setminus\{0\}}
\Big(
e^{-\beta y}
-1
+\beta y \mathbf 1_{|y|<1}
\Big)
\nu_{\alpha_{\mathrm{jump}},\lambda}(dy).
\end{equation}
Thus, the jump term contributes
\begin{equation}
\label{eq:appendix_jump_eigenvalue_tail}
\mathcal I^*_{\alpha_{\mathrm{jump}},\lambda}
\!\left[c\,e^{\beta\,\cdot}\right](\zeta)
=
\Psi^*(\beta;\omega)\,c\,e^{\beta\zeta}.
\end{equation}
The adjoint symbol $\Psi^*(\beta;\omega)$ and the backward symbol $\Psi(\xi;\omega)$ in~\eqref{eq:levy_symbol_definition} are related by the substitution $\xi=-\beta$: $\Psi^*(\beta;\omega)=\Psi(-\beta;\omega)$.
The star emphasizes that \(\Psi^*\) is the symbol of the adjoint jump operator \eqref{eq:appendix_adjoint_jump_operator}. For the symmetric tempered measure~\eqref{eq:tempered_levy_measure}, the backward symbol \(\Psi(\cdot\,;\omega)\) is even, so \(\Psi(-\beta;\omega)=\Psi(\beta;\omega)\).
Thus, the forward and backward symbols coincide in the symmetric model; without the symmetry assumption, they would have to be kept distinct.

Combining
\eqref{eq:appendix_drift_diffusion_on_exponential_tail} and
\eqref{eq:appendix_jump_eigenvalue_tail} in the tail balance
\eqref{eq:appendix_tail_balance_equation} gives
\begin{equation}
\label{eq:appendix_tail_balance_common_factor}
0
=
\left(
-\kappa\beta
+
\eta_G\beta^2
+
\Psi^*(\beta;\omega)
\right)c\,e^{\beta\zeta}.
\end{equation}
Since \(c>0\) and \(e^{\beta\zeta}>0\), the common factor can be divided out. This yields the characteristic function
\begin{equation}
\label{eq:appendix_characteristic_function}
\Phi(\beta;\omega)
=
\eta_G\beta^2
+
\Psi^*(\beta;\omega)
-
\kappa\beta,
\end{equation}
and interior positive roots $\beta\in(0,\lambda)$ of the associated characteristic equation
\begin{equation}
\label{eq:appendix_characteristic_equation}
\Phi(\beta;\omega)=0,
\end{equation}
determine the spectral exponent.

The endpoint root $\beta=\lambda$ is classified separately in Proposition~\ref{lem:Phi_properties}.

\subsection{Existence and uniqueness of the spectral exponent, and the jump-amplitude threshold $\eta_J^*$}
\label{app:existence_uniqueness}

This subsection classifies the positive roots of the characteristic equation~\eqref{eq:appendix_characteristic_equation} on $(0,\lambda]$, identifies the condition for a unique interior root, and converts that condition into an explicit threshold on the jump amplitude~$\eta_J$.

The vanishing $\Phi(0;\omega)=0$, the negative initial slope $\Phi'(0;\omega)=-\kappa<0$, and the strict convexity of $\Phi(\cdot;\omega)$ on the open domain $(0, \lambda)$ are immediate from standard L\'evy--Khintchine symbol calculus for translation-invariant generators~\citep{applebaum2009levy,sato1999levy}; we keep these facts explicit in the proof for self-containment but they are not the contribution of this subsection.
The tempered-stable jump measure used here is standard: it is a proper tempered $\alpha$-stable L\'evy measure in the sense of \citet{rosinski2007tempering}, with exponential tempering and a symmetric measure; equivalently, it is the symmetric bilateral tempered-stable special case of~\citet{kuchler2013tempered}.

The model-specific step is the boundary evaluation at the tempering scale $\beta=\lambda$: exponential tempering makes $\Psi^*(\beta;\omega)$ finite at the closed boundary $\beta=\lambda$ and infinite immediately beyond, which allows the roots on the domain $(0,\lambda]$ to be classified by the sign of $\Phi(\lambda;\omega)$.
Reading off the $\eta_J$-dependence of this condition then yields the explicit threshold~$\eta_J^*$.

\begin{proposition}[Root classification and jump-amplitude threshold for the spectral exponent]
\label{lem:Phi_properties}
For the symmetric tempered L\'evy measure~\eqref{eq:tempered_levy_measure}
with $\eta_J>0$, the characteristic function
\eqref{eq:appendix_characteristic_function} satisfies:
\begin{enumerate}[label=(\roman*)]
\item $\Phi(0;\omega)=0$ and $\Phi'(0;\omega)=-\kappa<0$.
\item $\Phi(\cdot\,;\omega)$ is strictly convex on $[0,\lambda)$ and continuous on $[0,\lambda]$.
\item $\Phi(\lambda;\omega)$ is finite, whereas $\Phi(\beta;\omega)=+\infty$ for every $\beta>\lambda$.
\end{enumerate}
On the closed nonnegative exponential-moment domain $[0,\lambda]$, $\beta=0$ is always the trivial root. The positive-root structure on the half-open interval $(0,\lambda]$ is as follows.
If $\Phi(\lambda;\omega)>0$, it has a unique interior positive root $\beta\in(0,\lambda)$. If $\Phi(\lambda;\omega)=0$, it has no interior positive root, but $\beta=\lambda$ is the unique positive endpoint root. If $\Phi(\lambda;\omega)<0$, it has no positive root in $(0,\lambda]$.

Consequently, an interior spectral root exists if and only if
\begin{equation}
\label{eq:appendix_existence_condition}
\Phi(\lambda;\omega)
=
\eta_G\lambda^2
+
\Psi^*(\lambda;\omega)
-
\kappa\lambda
>
0.
\end{equation}
Since $\Psi^*(\lambda;\omega)$ is proportional to~$\eta_J$, this is equivalent to
\begin{equation}
\label{eq:appendix_threshold_etaJ}
\eta_J
>
\eta_J^*
=
\frac{[\kappa\lambda - \eta_G\lambda^2]^+}
     {\psi_0(\alpha_{\mathrm{jump}},\lambda)},
\end{equation}
where $[x]^+=\max(x,0)$ and $\psi_0(\alpha_{\mathrm{jump}},\lambda) = \Psi^*(\lambda;\omega)/\eta_J > 0$ is a positive constant that depends only on $\alpha_{\mathrm{jump}}$ and $\lambda$.
When $\eta_J^*>0$, the three cases are equivalently $\eta_J>\eta_J^*$ for a unique interior root, $\eta_J=\eta_J^*$ for the unique endpoint root $\beta=\lambda$, and $0<\eta_J<\eta_J^*$ for no positive root.
When $\lambda\ge\kappa/\eta_G$, the threshold vanishes and every $\eta_J>0$ gives a unique interior root.
\end{proposition}

\begin{proof}
The first two properties are standard consequences of L\'evy--Khintchine, or cumulant-generating-function, calculus on the exponential-moment domain~\citep{sato1999levy}; we record the short verification to make clear where symmetry enters.

\begin{itemize}
\item
\emph{Standard facts at $\beta=0$ and convexity.}
At $\beta=0$, the integrand of~\eqref{eq:appendix_adjoint_jump_symbol}
collapses to zero, so $\Psi^*(0;\omega)=0$ and
$\Phi(0;\omega)=0$. For $\beta\in[0,\lambda)$, differentiating under
the integral in~\eqref{eq:appendix_adjoint_jump_symbol} gives
\begin{equation}
\label{eq:appendix_first_derivative_symbol}
\Psi^{*\prime}(\beta;\omega)
=
\int_{\mathbb R\setminus\{0\}}
\left(
-y e^{-\beta y}
+
y\mathbf 1_{|y|<1}
\right)
\nu_{\alpha_{\mathrm{jump}},\lambda}(dy).
\end{equation}
At $\beta=0$, the integrand in
\eqref{eq:appendix_first_derivative_symbol} becomes
\begin{equation}
-y+y\mathbf 1_{|y|<1}
=
\begin{cases}
0, & |y|<1,\\
-y, & |y|\ge 1.
\end{cases}
\end{equation}
Equivalently,
$-y+y\mathbf 1_{|y|<1}=-y\mathbf 1_{|y|\ge 1}$.
Therefore
\begin{equation}
\Psi^{*\prime}(0;\omega)
=
-\int_{|y|\ge 1}y\,
\nu_{\alpha_{\mathrm{jump}},\lambda}(dy).
\end{equation}
This integral is zero because the domain $\{|y|\ge1\}$ is symmetric,
the tempered measure~\eqref{eq:tempered_levy_measure} is symmetric,
and the function $y$ is odd: the contribution from $+u$ cancels the
contribution from $-u$. Hence $\Psi^{*\prime}(0;\omega)=0$. Since
\begin{equation}
\Phi'(\beta;\omega)
=
2\eta_G\beta+\Psi^{*\prime}(\beta;\omega)-\kappa,
\end{equation}
we obtain
\begin{equation}
\Phi'(0;\omega)
=
\Psi^{*\prime}(0;\omega)-\kappa
=
-\kappa<0,
\end{equation}
because $\kappa>0$.

On the open domain $\beta\in[0,\lambda)$, differentiating twice gives
\begin{equation}
\Phi''(\beta;\omega)
=
2\eta_G
+
\int_{\mathbb R\setminus\{0\}}
y^2 e^{-\beta y}\,
\nu_{\alpha_{\mathrm{jump}},\lambda}(dy).
\end{equation}
The integral is finite on $[0,\lambda)$ by exponential tempering, and
the diffusion term satisfies $\eta_G>0$, so
$\Phi''(\beta;\omega)>0$ on $[0,\lambda)$. This proves strict
convexity on the domain. The closed-boundary value at $\beta=\lambda$ is treated next.

\item 
\emph{Finiteness at $\beta=\lambda$ and divergence beyond.}
We now show that the adjoint symbol~\eqref{eq:appendix_adjoint_jump_symbol}
is finite at the tempering boundary $\beta=\lambda$, but diverges as
soon as $\beta>\lambda$. The split below isolates where this can
happen. Near the origin, convergence is controlled by the
small-jump compensator. In the large-jump tails, convergence is
controlled by the competition between the exponential factor
$e^{-\beta y}$ in the adjoint symbol and the tempering factor
$e^{-\lambda |y|}$ in the L\'evy measure. Positive large jumps are
always damped, while negative large jumps can make
$e^{-\beta y}$ grow against the tempering.

For the small-jump region $\{|y|<1\}$, the compensator in~\eqref{eq:appendix_adjoint_jump_symbol} is active.
Taylor-expanding $e^{-\beta y}$ around $y=0$ and truncating to second order gives
\begin{equation}
\label{eq:appendix_compensated_taylor}
e^{-\beta y}-1+\beta y
=
\left(1-\beta y+\tfrac{1}{2}\beta^2y^2+O(y^3)\right)
-1+\beta y
=
\tfrac{1}{2}\beta^2 y^2+O(y^3)
\qquad\text{as } y\to 0,
\end{equation}
so the compensator $\beta y\,\mathbf 1_{|y|<1}$ cancels the
linear term $-\beta y$ in $e^{-\beta y}-1$. The compensated
integrand therefore starts at order $y^2$. Near the origin the
tempering factor satisfies $e^{-\lambda |y|}\to1$, so the singular
part of the L\'evy density in~\eqref{eq:tempered_levy_measure} is
$|y|^{-1-\alpha_{\mathrm{jump}}}$. Multiplying this singular factor
by the order-$y^2$ compensated numerator gives
$O(|y|^{1-\alpha_{\mathrm{jump}}})$, which is integrable at zero
because $\alpha_{\mathrm{jump}}<2$, i.e.
\begin{equation}
\int_0^1 y^{1-\alpha_{\mathrm{jump}}}\,dy < \infty \quad\Longleftrightarrow\quad 1-\alpha_{\mathrm{jump}}>-1 \quad\Longleftrightarrow\quad \alpha_{\mathrm{jump}}<2.
\end{equation}
Thus, the small-jump region is finite for every fixed $\beta\ge0$.

On the positive large-jump region $\{y>1\}$, the compensator
$\beta y\,\mathbf 1_{|y|<1}$ vanishes because $|y|>1$.
The contribution of this region to
\eqref{eq:appendix_adjoint_jump_symbol} is therefore
\begin{equation}
\int_1^\infty
\big(e^{-\beta y}-1\big)\,
\eta_J c_{\alpha_{\mathrm{jump}}}
e^{-\lambda y}y^{-1-\alpha_{\mathrm{jump}}}\,dy .
\end{equation}
For $\beta\ge0$ and $y>1$, we have $0<e^{-\beta y}\le1$, so
$e^{-\beta y}-1\le0$ and $|e^{-\beta y}-1|\le1$.
To prove that this contribution is finite, it is enough to show
absolute integrability:
\begin{equation}
\int_1^\infty
\left|e^{-\beta y}-1\right|\,
\eta_J c_{\alpha_{\mathrm{jump}}}
e^{-\lambda y}y^{-1-\alpha_{\mathrm{jump}}}\,dy
\le
\eta_J c_{\alpha_{\mathrm{jump}}}
\int_1^\infty
e^{-\lambda y}y^{-1-\alpha_{\mathrm{jump}}}\,dy.
\end{equation}
The right-hand side is finite because $\lambda>0$ gives exponential
decay at infinity. Hence, the positive large-jump region is absolutely
integrable for every $\beta\ge0$ and imposes no restriction on the
allowed range of $\beta$.

The only remaining possible obstruction is the negative large-jump
region $\{y<-1\}$. On this region the compensator again vanishes
because $|y|>1$, so the contribution to
\eqref{eq:appendix_adjoint_jump_symbol} is
\[
\int_{-\infty}^{-1}
(e^{-\beta y}-1)\,
\eta_J c_{\alpha_{\mathrm{jump}}}
e^{-\lambda |y|}|y|^{-1-\alpha_{\mathrm{jump}}}\,dy .
\]
This is the dangerous tail: for negative~$y$,
$e^{-\beta y}=e^{\beta |y|}$ grows with $|y|$, while the
tempered measure contributes the damping factor
$e^{-\lambda |y|}$. To put these two exponentials on the same
positive variable, set $u=-y$. Then $u>1$, $|y|=u$, and
\eqref{eq:tempered_levy_measure} gives
\begin{equation}
\label{eq:negative_tail_change_of_variable}
\nu_{\alpha_{\mathrm{jump}},\lambda}(dy)\big|_{y<-1}
=
\eta_J\,c_{\alpha_{\mathrm{jump}}}\,
u^{-1-\alpha_{\mathrm{jump}}}\,e^{-\lambda u}\,du,
\qquad u>1
\end{equation}
where the sign from $dy=-du$ is absorbed by reversing the integration
limits. Since $e^{-\beta y}=e^{\beta u}$, the negative-tail
contribution is equivalently
\begin{equation}
\label{eq:negative_tail_integrand}
\int_1^\infty
\eta_J\,c_{\alpha_{\mathrm{jump}}}\,
(e^{\beta u}-1)\,
e^{-\lambda u}\,
u^{-1-\alpha_{\mathrm{jump}}}\,du .
\end{equation}
Equivalently, the integrand after the change of variables is
\begin{equation}
\label{eq:negative_tail_integrand_density}
\eta_J\,c_{\alpha_{\mathrm{jump}}}\,
(e^{\beta u}-1)\,
e^{-\lambda u}\,
u^{-1-\alpha_{\mathrm{jump}}},
\qquad u>1.
\end{equation}
The threshold is now visible: the leading exponential factor for large~$u$ is $e^{\beta u}e^{-\lambda u}=e^{(\beta-\lambda)u}$.
We now evaluate the boundary case, $\beta=\lambda$, and then the supercritical case, $\beta>\lambda$.

At the boundary $\beta=\lambda$, the exponentials in \eqref{eq:negative_tail_integrand_density} combine and the integrand simplifies to
\begin{equation}
\label{eq:negative_tail_boundary_integrand}
\eta_J\,c_{\alpha_{\mathrm{jump}}}\,
(1-e^{-\lambda u})\,u^{-1-\alpha_{\mathrm{jump}}}.
\end{equation}
Since $0\le 1-e^{-\lambda u}\le 1$, the absolute value of the boundary contribution is bounded by
\begin{equation}
\eta_J c_{\alpha_{\mathrm{jump}}}
\int_1^\infty
u^{-1-\alpha_{\mathrm{jump}}}\,du,
\end{equation}
which is finite because $\alpha_{\mathrm{jump}}>0$. The negative
large-jump region therefore contributes a finite amount at
$\beta=\lambda$.

Combining the small-jump, positive large-jump, and negative
large-jump estimates gives $\Psi^*(\lambda;\omega)<\infty$, and
therefore $\Phi(\lambda;\omega)<\infty$.

The same estimates also show that $\Psi^*(\beta;\omega)$ depends
continuously on $\beta$ up to the boundary. The only point requiring
an explicit uniform bound is again the negative tail. For
$u>1$ and $0\le\beta\le\lambda$,
\begin{equation}
\label{eq:appendix_neg_tail_uniform_bound}
0
\;\le\;
(e^{\beta u}-1)\,e^{-\lambda u}
\;\le\;
(e^{\lambda u}-1)\,e^{-\lambda u}
\;=\;
1-e^{-\lambda u}
\;\le\;
1 .
\end{equation}
Thus the absolute value of the negative-tail integrand is bounded, uniformly over $\beta\in[0,\lambda]$, by $\eta_J c_{\alpha_{\mathrm{jump}}} u^{-1-\alpha_{\mathrm{jump}}}$, which is integrable on $[1,\infty)$ because $\alpha_{\mathrm{jump}}>0$.
The small-jump and positive large-jump regions admit the corresponding uniform integrable bounds from the estimates above.
Dominated convergence then gives continuity of $\Psi^*$ on $[0,\lambda]$, and hence continuity of $\Phi$ on $[0,\lambda]$.

Beyond the boundary $\beta>\lambda$, the same negative-tail
integrand diverges. Since $\beta-\lambda>0$, choose
$U$ large enough that $e^{\beta u}-1\ge \tfrac12 e^{\beta u}$
for all $u\ge U$. Then, for $u\ge U$,
\eqref{eq:negative_tail_integrand_density} is bounded below by
a positive constant times $e^{(\beta-\lambda)u}\,u^{-1-\alpha_{\mathrm{jump}}}$.
This lower bound is not integrable on $[U,\infty)$: the exponential
growth $e^{(\beta-\lambda)u}$ dominates the polynomial decay
$u^{-1-\alpha_{\mathrm{jump}}}$.
Therefore, the negative-tail contribution to $\Psi^*(\beta;\omega)$ diverges.
The small-jump and positive large-jump regions remain finite, as shown above, so $\Psi^*(\beta;\omega)=+\infty$ and hence $\Phi(\beta;\omega)=+\infty$ for every $\beta>\lambda$.

\item 
\emph{Existence and uniqueness.}
The facts just proved imply that $\Phi(0;\omega)=0$,
$\Phi'(0;\omega)<0$, that $\Phi$ is strictly convex on
$[0,\lambda)$, and that $\Phi$ is continuous on $[0,\lambda]$.
The negative initial slope means that
$\Phi(\beta;\omega)<0$ for all sufficiently small
$\beta>0$.

If $\Phi(\lambda;\omega)>0$, continuity on $[0,\lambda]$ implies that $\Phi$ must cross zero somewhere in $(0,\lambda)$. Strict convexity gives uniqueness: after starting below zero, a strictly convex function can cross the horizontal axis at most once before the endpoint.

If $\Phi(\lambda;\omega)<0$, convexity gives, for every
$\beta\in(0,\lambda)$,
\begin{equation}
\Phi(\beta;\omega)
\le
\left(1-\frac{\beta}{\lambda}\right)\Phi(0;\omega)
+
\frac{\beta}{\lambda}\Phi(\lambda;\omega)
=
\frac{\beta}{\lambda}\Phi(\lambda;\omega)
<0.
\end{equation}
Thus, no positive root exists in $(0,\lambda]$.

If $\Phi(\lambda;\omega)=0$, then $\beta=\lambda$ is an endpoint
root. To exclude an interior root, fix $\beta\in(0,\lambda)$ and
choose $\gamma\in(\beta,\lambda)$. Convexity on the closed interval,
obtained by continuity at~$\lambda$, gives
$\Phi(\gamma;\omega)\le0$. Strict convexity on $[0,\gamma]$ then
gives
\begin{equation}
\Phi(\beta;\omega)
<
\left(1-\frac{\beta}{\gamma}\right)\Phi(0;\omega)
+
\frac{\beta}{\gamma}\Phi(\gamma;\omega)
\le0.
\end{equation}
Hence there is no interior positive root, and $\beta=\lambda$ is
the unique positive endpoint root.

These three cases establish the root classification stated in the proposition. In particular, a unique interior root $\beta\in(0,\lambda)$ exists if and only if $\Phi(\lambda;\omega)>0$.

\item
\emph{Threshold on $\eta_J$.}
It remains to rewrite the endpoint condition $\Phi(\lambda;\omega)>0$ as a threshold on the jump amplitude.
The only dependence of $\Psi^*(\lambda;\omega)$ on $\eta_J$ comes from the prefactor of the tempered L\'evy measure~\eqref{eq:tempered_levy_measure}.
Hence
\begin{equation}
\Psi^*(\lambda;\omega)
=
\eta_J\,\psi_0(\alpha_{\mathrm{jump}},\lambda),
\end{equation}
where symmetry makes the sign explicit. Pair the contributions at $y=u$ and $y=-u$, with $u>0$.
Their small-jump compensator terms cancel, while the remaining terms sum to
$(e^{-\lambda u}-1)+(e^{\lambda u}-1) = 2(\cosh(\lambda u)-1)$.
Therefore,
\begin{equation}
\label{eq:appendix_psi0_positive}
\begin{aligned}
\psi_0(\alpha_{\mathrm{jump}},\lambda)
&=
\int_{\mathbb R\setminus\{0\}}
\Big(
e^{-\lambda y}
-1
+\lambda y\,\mathbf 1_{|y|<1}
\Big)
c_{\alpha_{\mathrm{jump}}}
\frac{e^{-\lambda |y|}}
{|y|^{1+\alpha_{\mathrm{jump}}}}\,dy
\\
&=
2c_{\alpha_{\mathrm{jump}}}
\int_0^\infty
\bigl(\cosh(\lambda u)-1\bigr)
e^{-\lambda u}
u^{-1-\alpha_{\mathrm{jump}}}\,du
>0.
\end{aligned}
\end{equation}
Equation~\eqref{eq:appendix_psi0_positive} shows that, at the tempering boundary $\beta=\lambda$, the jump contribution has a strictly positive coefficient per unit jump amplitude.
Consequently, $\Phi(\lambda;\omega)$ is affine and strictly increasing in~$\eta_J$, which allows the endpoint condition $\Phi(\lambda;\omega)>0$ to be solved as an explicit threshold.

Substituting this factorization into \eqref{eq:appendix_existence_condition} gives
\begin{equation}
\Phi(\lambda;\omega)
=
\eta_G\lambda^2
+
\eta_J\,\psi_0(\alpha_{\mathrm{jump}},\lambda)
-
\kappa\lambda .
\end{equation}
Therefore $\Phi(\lambda;\omega)>0$ is equivalent to
\begin{equation}
\eta_J
>
\frac{\kappa\lambda-\eta_G\lambda^2}
     {\psi_0(\alpha_{\mathrm{jump}},\lambda)}.
\end{equation}
Since the jump amplitude satisfies $\eta_J\ge0$, a negative right-hand side imposes no positive lower bound on the jump amplitude.
The effective threshold is therefore the positive part of the algebraic threshold:
\begin{equation}
\eta_J^*=\frac{\left[\kappa\lambda-\eta_G\lambda^2\right]^+}{\psi_0(\alpha_{\mathrm{jump}},\lambda)},
\end{equation}
which is exactly \eqref{eq:appendix_threshold_etaJ}.

If $\lambda>\kappa/\eta_G$, then $\eta_G\lambda^2-\kappa\lambda>0$, so the drift--diffusion contribution is already positive.
If $\lambda=\kappa/\eta_G$, that contribution is zero, but every $\eta_J>0$ gives
$\Phi(\lambda;\omega)=\eta_J\psi_0(\alpha_{\mathrm{jump}},\lambda)>0$.
Thus, $\eta_J^*=0$ in both cases, and every positive jump amplitude produces an interior root.
\end{itemize}

\end{proof}

\subsection{Monotonicity and smoothness of the spectral exponent}
\label{app:beta_smoothness}

Let $\eta_J>\eta_J^*$, so that the characteristic equation~\eqref{eq:appendix_characteristic_equation} has a unique interior root $\beta(\omega)\in(0,\lambda)$ (Proposition~\ref{lem:Phi_properties}); $\Phi'$ denotes the derivative with respect to~$\beta$.
At this root, $\Phi'(\beta;\omega)>0$: strict convexity and $\Phi(0;\omega)=\Phi(\beta;\omega)=0$ place $\Phi(\cdot;\omega)$ strictly below zero on $(0,\beta)$, so its slope at the positive root is strictly positive.
For $0<\beta<\lambda$, small-jump compensation and exponential tempering provide locally uniform integrable bounds for the differentiated symbol integrand, so $\Phi$ is smooth in $(\beta,\omega)$.
The implicit function theorem~\citep{krantz2002implicit} then makes the unique root $\beta(\omega)$ a smooth function of the control parameters $\omega$ throughout the root-selected branch $\{\omega:\Phi(\lambda;\omega)>0\}$.

Differentiating the identity $\Phi(\beta(\omega);\omega)=0$ with respect to $\eta_J$, with the other control parameters fixed, gives $\Phi'(\beta;\omega)\,\partial_{\eta_J}\beta+\partial_{\eta_J}\Phi(\beta;\omega)=0$, where the last partial derivative holds $\beta$ fixed.
In~\eqref{eq:appendix_characteristic_function} only the jump symbol depends on $\eta_J$, linearly and with positive coefficient for each $0<\beta<\lambda$, so $\partial_{\eta_J}\Phi(\beta;\omega)=\Psi^*(\beta;\omega)/\eta_J>0$ and
\begin{equation}
\label{eq:appendix_beta_jump_derivative}
\partial_{\eta_J}\beta
=
-\frac{\Psi^*(\beta;\omega)/\eta_J}{\Phi'(\beta;\omega)}
<0 .
\end{equation}
A stronger jump forcing therefore lowers the spectral exponent and broadens the time-scale spectrum; geometrically, increasing $\eta_J$ raises $\Phi(\cdot;\omega)$ pointwise on $(0,\lambda)$ and moves the unique positive crossing leftward.
Since $\partial_{\eta_J}\beta\neq0$, the gradient $\nabla_\omega\beta$ vanishes nowhere on this branch, so every nonempty level set of $\beta$ is a smooth hypersurface of codimension one in control-parameter space.
This justifies the manifold description of $\mathcal{M}_c$ in~\eqref{eq:critical_manifold_definition}.

\subsection{Limiting cases and the symmetry assumption}
\label{app:limiting_cases_diffusion_only}

\paragraph{Pure diffusion limit.}
Setting $\eta_J=0$ removes the jump term from \eqref{eq:zeta_SDE_mixed}; consequently, $\alpha_{\mathrm{jump}}$ and~$\lambda$ no longer enter the diffusion-only model.
At the level of the characteristic equation, $\Phi(\beta;\omega)=\eta_G\beta^2-\kappa\beta$, whose nonzero root is $\beta_{\mathrm{diff}}=\kappa/\eta_G$.

With $\eta_J=0$, the jump operator in the stationary forward equation
\eqref{eq:appendix_stationary_forward_equation} vanishes. In the
far-left tail, the closure~\eqref{eq:drift_saturation} replaces
$F(\zeta)$ by its leading constant value~$\kappa$. The stationary
diffusion balance is therefore
\begin{equation}
\label{eq:appendix_diffusion_stationary_current}
0
=
-\kappa\,\partial_\zeta\rho_\infty(\zeta)
+
\eta_G\,\partial_{\zeta\zeta}\rho_\infty(\zeta)
=
-\partial_\zeta\!\left[
\kappa\rho_\infty(\zeta)
-
\eta_G\,\partial_\zeta\rho_\infty(\zeta)
\right].
\end{equation}
The bracket is the stationary probability current of the diffusion-only process.
Equation~\eqref{eq:appendix_diffusion_stationary_current} shows that it is constant in~$\zeta$. For a normalizable stationary density with no flux at infinity, this constant is zero.
Hence $\eta_G\,\partial_\zeta\rho_\infty(\zeta)=\kappa\rho_\infty(\zeta)$.

The zero-current relation can be rewritten, using
$\beta_{\mathrm{diff}}=\kappa/\eta_G$, as
\begin{equation}
\frac{\partial_\zeta\rho_\infty(\zeta)}{\rho_\infty(\zeta)} = \beta_{\mathrm{diff}}.
\end{equation}
Choose a fixed reference point $\zeta_0$ sufficiently far in the
tail and consider $\zeta\le\zeta_0$. Since
$\rho_\infty(\zeta)>0$ in this region, the division is well defined.
To integrate the equation between $\zeta_0$ and $\zeta$, use $s$ as
the dummy spatial variable.
In the resulting left-hand integral, make the substitution $r=\rho_\infty(s)$, for which $dr=\partial_s\rho_\infty(s)\,ds$.
The corresponding integration limits are $\rho_\infty(\zeta_0)$ at $s=\zeta_0$ and $\rho_\infty(\zeta)$ at $s=\zeta$. Therefore,
\begin{equation}
\label{eq:appendix_diffusion_tail_solution}
\begin{aligned}
\int_{\zeta_0}^{\zeta}
\frac{\partial_s\rho_\infty(s)}
     {\rho_\infty(s)}\,ds
&=
\beta_{\mathrm{diff}}
\int_{\zeta_0}^{\zeta}ds,
\\
\int_{\rho_\infty(\zeta_0)}^{\rho_\infty(\zeta)}
\frac{dr}{r}
&=
\beta_{\mathrm{diff}}(\zeta-\zeta_0),
\\
\log\!\frac{\rho_\infty(\zeta)}
{\rho_\infty(\zeta_0)}
&=
\beta_{\mathrm{diff}}(\zeta-\zeta_0),
\\
\frac{\rho_\infty(\zeta)}
{\rho_\infty(\zeta_0)}
&=
e^{\beta_{\mathrm{diff}}(\zeta-\zeta_0)},
\\
\rho_\infty(\zeta)
&=
\rho_\infty(\zeta_0)
e^{-\beta_{\mathrm{diff}}\zeta_0}
e^{\beta_{\mathrm{diff}}\zeta}
=
\underbrace{
\rho_\infty(\zeta_0)
e^{-\beta_{\mathrm{diff}}\zeta_0}
}_{\displaystyle c>0}
e^{\beta_{\mathrm{diff}}\zeta}.
\end{aligned}
\end{equation}
Because $\zeta_0$ is fixed, the underbraced factor is a positive
constant independent of~$\zeta$. Thus the leading far-left behavior
is $\rho_\infty(\zeta)=c\,e^{\beta_{\mathrm{diff}}\zeta}$.
Under the change of variables~\eqref{eq:tau_def}, this gives the power-law time-scale form~\eqref{eq:stationary_timescale_power_law}.

This is a special diffusion-only route to anti-collapse, distinct from the jump-driven mechanism presented in the main text.

\paragraph{Relation to the threshold.}
Because $\lambda$ belongs to the jump measure, the comparison $\lambda\ge\kappa/\eta_G$ concerns the mixed stochastic model with both $\eta_G,\eta_J>0$ and determines how its interior-root branch behaves as $\eta_J\downarrow0$.
When $\lambda\ge\kappa/\eta_G$, the threshold vanishes: a unique interior root exists for every $\eta_J>0$ and approaches $\beta_{\mathrm{diff}}$ as $\eta_J\downarrow0$.
When $\lambda<\kappa/\eta_G$, the threshold is positive, so the root-selected branch exists only for $\eta_J>\eta_J^*$ and cannot be continued to $\eta_J=0$ through admissible interior roots.
This does not remove the separate saturated-drift diffusion-only solution derived above, for which $\lambda$ is not a model parameter.

\paragraph{Justification of the symmetric measure.}
Throughout this paper, we adopted the symmetric tempered L\'evy measure~\eqref{eq:tempered_levy_measure}, i.e. with equal weight on positive and negative jumps.
Three considerations support this choice.
First, at the level of an effective population model, there is no special reason to favour upward over downward fluctuations in the log-effective decay rate.
Second, the choice is compatible with empirical observations that gradient noise in stochastic optimization is well-described by a symmetric $\alpha$-stable distribution~\citep{simsekli2019tail,gurbuzbalaban2021heavy}, which motivates treating the effective heavy-tailed forcing as symmetric at the population level.
Third, symmetry yields the structural properties used in the proof of Proposition~\ref{lem:Phi_properties}.
Combined with strict convexity, this gives existence and uniqueness whenever the boundary condition $\Phi(\lambda;\omega)>0$ is satisfied.

\section{Envelope asymptotics under representative stationary tail classes}
\label{app:envelope_tail_calculations}

This appendix provides the detailed asymptotic calculations
supporting the envelope decay laws of
Section~\ref{sec:envelope_laws_from_tails}.
All derivations are carried out for the unweighted envelope
\begin{equation}
\label{eq:app_envelope_integral}
f_0(\ell)
=
\int_{0}^{\infty} e^{-\ell/\tau}\,p_\infty(\tau)\,d\tau, \qquad \ell\ge 0.
\end{equation}
The sandwich bound~\eqref{eq:sandwich} ensures that the full (weighted) envelope $f(\ell)$ in Eq.~\ref{eq:envelope_integral_weighted} shares the same scaling class, with only the prefactor absorbing the bounded optimizer-dependent factor.

The finite-width representation~\eqref{eq:mixture_of_exponentials}
involves both diagonal transport factors and neuron-wise
optimizer contributions
(via the Rayleigh projection~$\Lambda^{(q)}_{r,\ell}$).
The unweighted envelope $f_0$ retains only the diagonal
terms; optimizer effects enter the full envelope $f(\ell)$
solely through the bounded prefactors captured by the
sandwich bound~\eqref{eq:sandwich}.
\citet{livi2026learnability} demonstrates that the
diagonal approximation preserves the asymptotic trend,
so the scaling laws derived below apply
to the full transport envelope up to sub-leading corrections.

Three canonical tail classes are treated below, ordered from simplest to most involved, following the same sequence as Section~\ref{sec:envelope_laws_from_tails}.
The exponential envelope (collapsed-regime representative, $\eta_J=0$) is handled first via a direct power-series expansion of \eqref{eq:app_envelope_integral}.
The power-law envelope on the root-selected branch and the logarithmic envelope of the log-regular ansatz follow by recasting~\eqref{eq:app_envelope_integral} as a standard Laplace transform (Section~\ref{app:laplace_form}).
For these two cases, the calculations derive the leading-order forms by direct evaluation.
Karamata's Tauberian theorem for Laplace transforms (Proposition~\ref{prop:tauberian}, Section~\ref{app:laplace_form}) then certifies that the leading terms so obtained are the correct asymptotic equivalences as $\ell\to\infty$, via the one-to-one correspondence between the (log-)regularly varying behavior of a density near the origin and the asymptotic decay of its Laplace transform at infinity.

\subsection{Exponential envelope}
\label{app:exponential_envelope_full_details}

This subsection treats the canonical OU representative of the collapsed regime.
We set $\eta_J=0$ and use the bulk drift linearization~\eqref{eq:drift_linearization} as a global approximation over the concentrated stationary spectrum.
The resulting pure-diffusion SDE is an OU process, as discussed in Section~\ref{sec:collapsed_regime}.
Because this calculation concerns fluctuations around the bulk equilibrium rather than the far-left-tail balance, we derive the envelope through a small-variance expansion in $\zeta$-space; the characteristic equation is not needed.

In this pure-diffusion representative, the stationary density of the log-effective decay rate is Gaussian:
\begin{equation}
\label{eq:app_rho_gaussian}
\rho_\infty(\zeta)
=
\frac{1}{\sqrt{2\pi\sigma^2}}\,
\exp\!\!\left(-\frac{(\zeta-\mu)^2}{2\sigma^2}\right),
\qquad
\sigma^2=\frac{\eta_G}{\gamma},
\
\mu=\mathbb{E}[\zeta]=\zeta^*.
\end{equation}
Using~\eqref{eq:tau_def}, the corresponding stationary time-scale density is log-normal,
\begin{equation}
\label{eq:app_p_lognormal}
p_\infty(\tau)
=
\frac{\rho_\infty(-\log\tau)}{\tau}
=
\frac{1}{\tau\sqrt{2\pi\sigma^2}}\,
\exp\!\!\left(-\frac{(\log\tau+\mu)^2}{2\sigma^2}\right),
\qquad \tau>0.
\end{equation}
Its median is $\tau_\ast=e^{-\mu}$, all moments are finite, and its right tail decays faster than any polynomial. Consequently, the log-normal distribution is not regularly varying, so the hypotheses of Karamata's Tauberian theorem in Proposition~\ref{prop:tauberian} do not apply. We therefore derive the corresponding envelope directly in $\zeta$-space.

\paragraph{Derivation in $\zeta$-space.}
The envelope integral~\eqref{eq:app_envelope_integral} involves $e^{-\ell/\tau}$ averaged against $p_\infty(\tau)$.
A direct expansion of $e^{-\ell/\tau}$ in powers of $\ell/\tau$ is not useful: it involves the inverse moments $\mathbb{E}[\tau^{-n}] = \exp(n\mu+n^2\sigma^2/2)$, whose rapid growth makes the resulting moment series diverge for every $\ell>0$.
We therefore pass to $\zeta=-\log\tau$, where the distribution is Gaussian.
This permits a Taylor expansion in the log-rate domain and a closed-form evaluation using the Gaussian moment-generating function~\citep{durrett2019probability}.

Write $\zeta=\mu+\sigma Z$ with $\mu=\mathbb{E}[\zeta]$ and $Z\sim N(0,1)$, so that $\tau=e^{-\zeta}=\tau_\ast\,e^{-\sigma Z}$ where $\tau_\ast=e^{-\mu}$.
The envelope integral~\eqref{eq:app_envelope_integral} becomes
\begin{equation}
\label{eq:app_collapsed_zeta}
f_0(\ell)
=
\mathbb{E}_Z\!\left[
e^{-(\ell/\tau_\ast)e^{\sigma Z}}
\right]
=
\mathbb{E}_Z\!\left[
e^{-(\ell/\tau_\ast)e^{-\sigma Z}}
\right],
\end{equation}
where the second equality follows because $Z$ and $-Z$ have the same standard normal distribution.

For small $\sigma$, Taylor-expand $e^{-\sigma Z}$ to second order:
\begin{equation}
\label{eq:app_exp_inner}
e^{-\sigma Z}
=
1-\sigma Z+\tfrac{1}{2}\sigma^2 Z^2
+R_3(\sigma,Z),
\end{equation}
where $R_3(\sigma,Z)$ is the third-order Taylor remainder,
bounded by $|R_3(\sigma,Z)|\le \tfrac{1}{6}\sigma^3|Z|^3\,e^{\sigma|Z|}$
via the Lagrange form.
Multiplying by $-\ell/\tau_\ast$:
\begin{align}
\label{eq:app_exp_exponent}
-\frac{\ell}{\tau_\ast}\,e^{-\sigma Z}
&=
-\frac{\ell}{\tau_\ast}
+\frac{\ell\sigma}{\tau_\ast}\,Z
-\frac{\ell\sigma^2}{2\tau_\ast}\,Z^2
-\frac{\ell}{\tau_\ast}\,R_3(\sigma,Z).
\end{align}
Substituting into~\eqref{eq:app_collapsed_zeta} and factoring
out the leading exponential:
\begin{equation}
\label{eq:app_exp_factored}
f_0(\ell)
=
e^{-\ell/\tau_\ast}\;
\mathbb{E}_Z\!\left[
\exp\!\left(
\frac{\ell\sigma}{\tau_\ast}\,Z
-\frac{\ell\sigma^2}{2\tau_\ast}\,Z^2
-\frac{\ell}{\tau_\ast}\,R_3(\sigma,Z)
\right)\right].
\end{equation}
In the regime where $\ell\sigma^2/\tau_\ast\ll 1$, the quadratic
term and the remainder $R_3$ contribute negligibly
to the expectation (the Gaussian tails of $Z$ ensure that
the moments of $|Z|^k e^{\sigma|Z|}$ are finite for every~$k$).
Dropping them and using the Gaussian moment generating function
$\mathbb{E}[e^{aZ}]=e^{a^2/2}$ yields
\begin{equation}
\label{eq:app_exp_correction}
f_0(\ell)
\approx
e^{-\ell/\tau_\ast}\;
e^{\,\ell^2\sigma^2/(2\tau_\ast^2)}.
\end{equation}
The correction factor $e^{\,\ell^2\sigma^2/(2\tau_\ast^2)}$
is negligible (i.e.\ its logarithm is small compared to
$\ell/\tau_\ast$) precisely when
\begin{equation}
\label{eq:app_concentrated_regime}
\ell \;\ll\; \tau_\ast/\sigma^2.
\end{equation}
Therefore, in this regime the envelope is well-approximated by a single exponential:
\begin{equation}
\label{eq:app_collapsed_decay}
f_0(\ell)\sim e^{-\ell/\tau_\ast}.
\end{equation}

For integer lags in this range, $a=e^{-1/\tau_\ast}$ gives the geometric form $f_0(\ell)\sim a^\ell$ used in Section~\ref{sec:exponential_envelope}.

In this OU representative, the variance $\sigma^2=\eta_G/\gamma$ is controlled by the local restoring rate of the bulk drift: large~$\gamma$ keeps the OU process tightly concentrated around its equilibrium~$\zeta^*$, giving a small $\sigma^2$.
When $\eta_J=0$ (no jumps) and $\gamma$ is moderate to large, the diffusion is tightly confined and $\sigma^2$ is small.
The condition~\eqref{eq:app_concentrated_regime} then encompasses an operationally vast range of lags, typically exceeding any practical sequence length, so that~\eqref{eq:app_collapsed_decay} is a relevant approximation for applications.

\begin{remark}[Scaling beyond the median-dominated regime]
\label{rem:collapsed_range}
Equation~\eqref{eq:app_collapsed_decay} is a concentrated-regime approximation, not a strict asymptotic equivalence as $\ell\to\infty$.
The regime $\ell\sigma^2/\tau_\ast\ll 1$ is the \emph{median-dominated} regime: the integral~\eqref{eq:app_envelope_integral} draws its mass from a neighborhood of the median $\tau_\ast$, and the log-normal spread is too narrow (relative to $\ell^{-1}$) to matter.
In this regime, the distribution is concentrated, yielding the exponential decay~\eqref{eq:app_collapsed_decay}.

As $\ell$ approaches the crossover scale $\tau_\ast/\sigma^2$, the correction factor~\eqref{eq:app_exp_correction} ceases to be negligible: its exponent $\ell^2\sigma^2/(2\tau_\ast^2)$ becomes comparable to the leading exponent $\ell/\tau_\ast$, so the width of the log-normal distribution can no longer be ignored.
At larger lags, the kernel $e^{-\ell/\tau}$ increasingly suppresses the contribution from time scales near the median~$\tau_\ast$.
The dominant contribution consequently moves into the far-right tail of the log-normal time-scale distribution, toward time scales that grow almost linearly with~$\ell$, up to a logarithmic correction.
Equivalently, the relevant log-rates satisfy $\zeta=-\log\tau=-\log\ell+O(\log\log\ell)$.

Because $\rho_\infty(\zeta)$ is Gaussian, reaching this region incurs a Gaussian penalty on the logarithmic lag scale. At leading logarithmic order,
\begin{equation}
\log f_0(\ell)\sim-(\log\ell)^2/(2\sigma^2).
\end{equation}
Therefore, the decay of the envelope is slower than any stretched exponential $\exp(-a\ell^q)$ with $a,q>0$, but faster than any power law $\ell^{-m}$ with $m>0$.
Thus, the collapsed envelope crosses from an operational single-exponential regime to a non-power-law log-normal asymptotic, while the anti-collapsed envelope remains distinguished by its regularly varying power-law decay.

The strict log-normal asymptotic does not remove the learnability separation motivating the paper. Under the sample-complexity relation derived by \citet{livi2026learnability}, $N(\ell)$ scales as $f_0(\ell)^{-\kappa_\alpha}$. Therefore,
\begin{equation}
\label{eq:app_lognormal_sample_complexity}
\log N(\ell)
\sim
\frac{\kappa_\alpha}{2\sigma^2}(\log\ell)^2,
\qquad
N(\ell)
=
\exp\!\left[
\left(
\frac{\kappa_\alpha}{2\sigma^2}+o(1)
\right)
(\log\ell)^2
\right].
\end{equation}
This growth is \emph{quasi-polynomial}: it is slower than exponential in $\ell$ but faster than every fixed polynomial. On both anti-collapsed branches the sample complexity is polynomially bounded in $\ell$.
\end{remark}

\begin{remark}[Anti-collapse strictly below the threshold, $0<\eta_J<\eta_J^*$]
\label{rem:sub_threshold_envelope}
The derivation above assumes $\eta_J=0$ and a globally linear restoring drift.
It does not describe the mixed model below the threshold.
For $0<\eta_J<\eta_J^*$ under the saturated closure~\eqref{eq:drift_saturation}, $\Phi(\lambda;\omega)<0$ and Proposition~\ref{lem:Phi_properties} excludes a positive characteristic root in $(0,\lambda]$, so no root-selected power law exists.
The stationary time-scale density is regularly varying with index $-1-\lambda$ (Corollary~\ref{cor:subthreshold_density}).
The survival function and the unweighted envelope both have index $-\lambda$ and carry the same logarithmic correction as the density (Proposition~\ref{prop:subthreshold_survival} and Corollary~\ref{cor:subthreshold_envelope}).
When $\eta_J^*>0$, the equality case $\eta_J=\eta_J^*$ is not part of this remark: the characteristic equation then has the algebraic endpoint root $\beta=\lambda$, but neither the interior-mode argument nor Proposition~\ref{prop:subthreshold_survival} determines its stationary tail or envelope.
\end{remark}

\subsection{Standard Laplace form and Tauberian correspondence}
\label{app:laplace_form}

The remaining two envelope classes (power-law and logarithmic) arise from time-scale distributions $p_\infty(\tau)$ whose mass extends to arbitrarily large~$\tau$.
For these classes, we recast the envelope integral as a standard Laplace transform and invoke the classical Tauberian correspondence between the small-argument behavior of a density and the large-argument decay of its Laplace transform.
The material in this subsection is standard Laplace/Tauberian theory; we follow \citet{feller1971introduction} throughout and give self-contained statements so that the development is fully accessible.

\paragraph{Recasting as a standard Laplace transform.}
The kernel $e^{-\ell/\tau}$
in~\eqref{eq:app_envelope_integral} becomes a standard
Laplace kernel $e^{-\ell\bar{\mu}}$ when written in terms of
the asymptotic decay
rate~$\bar{\mu}=\tau^{-1}$~\eqref{eq:tau_def}.
Rates and time scales are dual descriptions of the same
quantity (cf.\ the discussion in
Section~\ref{sec:asymptotic_decay_log_effective_LR}): the
analysis uses time scales~$\tau$ when interpreting spectral
organization, and returns to rates~$\bar{\mu}$ when the
Laplace structure is needed.
Substituting $\bar{\mu}=\tau^{-1}$ and using
$|d\tau/d\bar{\mu}|=\bar{\mu}^{-2}$ gives the rate density
\begin{equation}
\label{eq:app_change_tau_u}
p_{\bar{\mu}}(\bar{\mu})=\bar{\mu}^{-2}\,p_\infty(1/\bar{\mu}),
\end{equation}
and the envelope
integral~\eqref{eq:app_envelope_integral} becomes
\begin{equation}
\label{eq:app_envelope_laplace}
f_0(\ell)=\int_{0}^{\infty} e^{-\ell\bar{\mu}}\,p_{\bar{\mu}}(\bar{\mu})\,d\bar{\mu},
\end{equation}
i.e.\ a standard Laplace transform of the rate density
$p_{\bar{\mu}}$.
The right tail of $p_\infty(\tau)$ as $\tau\to\infty$ maps
to the behavior of $p_{\bar{\mu}}(\bar{\mu})$ near
$\bar{\mu}=0$, and the large-$\ell$ asymptotics
of~\eqref{eq:app_envelope_laplace} are governed precisely by
this small-$\bar{\mu}$ behavior.

\paragraph{Tail taxonomy.}
Before stating the Tauberian correspondence, we introduce the tail-class definitions used below.
A measurable function $L:(0,\infty)\to(0,\infty)$ is \emph{slowly varying at infinity} if
\begin{equation}
\label{eq:app_slowly_varying_infinity}
\lim_{x\to\infty}\frac{L(tx)}{L(x)}=1
\qquad\text{for every }t>0.
\end{equation}
Typical examples are $L(x)=\log x$, $L(x)=(\log x)^\vartheta$,
and constants; any such $L$ grows or decays more slowly
than any positive or negative power of $x$.
A function $g(x)=x^\rho L(x)$ with $\rho\in\mathbb{R}$ and
$L$ slowly varying at infinity is \emph{regularly varying
at infinity with index~$\rho$}:
\begin{equation}
\label{eq:app_regularly_varying_infinity}
g(x)=x^\rho\,L(x),\qquad
L\text{ slowly varying at infinity}.
\end{equation}
We use the following two representative density tails for the explicit
envelope calculations below.
\emph{Power-law tail.} The pure power-law density tail
\begin{equation}
\label{eq:app_power_law_tail}
p_\infty(\tau)\sim c\,\tau^{-1-\beta},
\qquad c>0,\ \beta>0,
\end{equation}
is regularly varying with index $\rho=-(1+\beta)$ and constant
slowly varying part ($L(\tau)=c$).
\emph{Log-regularly varying tail.} The regularly varying
tail at index $\rho=-1$ with logarithmic slowly varying
modulation ($L(\tau)=(\log\tau)^{-(1+\vartheta)}$),
\begin{equation}
\label{eq:app_log_regularly_varying_tail}
p_\infty(\tau)\sim \tau^{-1}\,(\log\tau)^{-(1+\vartheta)},
\qquad \vartheta>0,
\end{equation}
sits at the integrability boundary of the regularly varying class.
Indices $\rho<-1$ are integrable for the standard slowly varying corrections considered here, indices $\rho>-1$ are not, and at $\rho=-1$ integrability is decided by the slowly varying factor.
The logarithmic exponent $-(1+\vartheta)$ with $\vartheta>0$ in~\eqref{eq:app_log_regularly_varying_tail} makes the right tail integrable, since $\int_A^\infty \tau^{-1}(\log\tau)^{-1-\vartheta}\,d\tau<\infty$ for every $A>1$.
This form is therefore a canonical integrable representative at the boundary index $\rho=-1$.

The qualifier ``boundary'' refers to this regular-variation integrability threshold, not to the stochastic generator of the main text where no positive root of the characteristic equation reaches $\beta=0$ (Proposition~\ref{lem:Phi_properties}).
We use~\eqref{eq:app_log_regularly_varying_tail} as the ansatz tail at the $\beta\downarrow 0$ boundary, defining the log-regular regime in the Tauberian taxonomy, and use it to derive the logarithmic envelope class in Section~\ref{app:log_envelope_derivation_full}.

The Tauberian theorem below is stated on the rate side, so we also need the rate-side counterpart of slow variation: a function $L_0$ is \emph{slowly varying at zero} if
\begin{equation}
\label{eq:app_slowly_varying_zero}
L_0(\bar{\mu})=\tilde L(1/\bar{\mu}),
\qquad \tilde L\text{ slowly varying at infinity}.
\end{equation}
Under the reciprocal change of
variables~\eqref{eq:app_change_tau_u}, the general
heavy-tailed form on the time-scale side maps to a
small-rate density of the same regularly varying type:
\begin{equation}
\label{eq:app_tau_tail_general_to_rate}
p_\infty(\tau)\sim \tau^{-1-\beta}\,\tilde L(\tau),
\qquad \tau\to\infty
\quad\Longrightarrow\quad
p_{\bar{\mu}}(\bar{\mu})\sim \bar{\mu}^{\beta-1}\,
\tilde L(1/\bar{\mu}),
\qquad \bar{\mu}\downarrow 0,
\end{equation}
where $\tilde L$ is slowly varying at infinity.
Specialized to~\eqref{eq:app_power_law_tail}
and~\eqref{eq:app_log_regularly_varying_tail}, the
implication~\eqref{eq:app_tau_tail_general_to_rate} gives
$p_{\bar{\mu}}(\bar{\mu})\sim c\,\bar{\mu}^{\beta-1}$ in the
power-law case and
$p_{\bar{\mu}}(\bar{\mu})\sim
\bar{\mu}^{-1}\bigl(\log(1/\bar{\mu})\bigr)^{-(1+\vartheta)}$
in the log-regularly varying case. These are the
small-$\bar{\mu}$ asymptotic forms required to apply
Proposition~\ref{prop:tauberian}.

\paragraph{Karamata's Tauberian theorem.}
With the tail-class definitions of the previous paragraph
in place, the classical correspondence between the
small-argument behavior of a density and the large-argument
decay of its Laplace transform specializes to the two forms
required below:

\begin{proposition}[Karamata's Tauberian theorem for Laplace transforms \citep{feller1971introduction}]
\label{prop:tauberian}
Let $p:(0,\infty)\to[0,\infty)$ be locally integrable and
eventually monotone in a neighborhood of the origin, and let
$\hat p(\ell)=\int_{0}^{\infty} e^{-\ell\bar{\mu}}\,p(\bar{\mu})\,d\bar{\mu}$
denote its Laplace transform.
\begin{enumerate}
\item[(i)] \emph{Regularly varying case.}
If $p(\bar{\mu})\sim c\,\bar{\mu}^{\beta-1}\,L(\bar{\mu})$
as $\bar{\mu}\downarrow 0$ for some $\beta>0$, $c>0$, and
$L$ slowly varying at zero, then
\begin{equation}
\hat p(\ell)\;\sim\;c\,\Gamma(\beta)\,\ell^{-\beta}\,L(1/\ell),
\qquad \ell\to\infty,
\end{equation}
where $\Gamma(\beta)$ is the Gamma integral~\eqref{eq:app_gamma_integral}.
Conversely, under the same regularity hypotheses, this
Laplace-transform asymptotic implies
$p(\bar{\mu})\sim c\,\bar{\mu}^{\beta-1}L(\bar{\mu})$
as $\bar{\mu}\downarrow 0$.
\item[(ii)] \emph{Log-regularly varying case.}
If
$p(\bar{\mu})\sim \bar{\mu}^{-1}\bigl(\log(1/\bar{\mu})\bigr)^{-(1+\vartheta)}$
as $\bar{\mu}\downarrow 0$ for some $\vartheta>0$, then
\begin{equation}
\hat p(\ell)\;\sim\;\frac{1}{\vartheta}\,(\log\ell)^{-\vartheta},
\qquad \ell\to\infty.
\end{equation}
\end{enumerate}
\end{proposition}

Proposition~\ref{prop:tauberian} is applied in this paper
with $p=p_{\bar{\mu}}$, so that
$\hat p(\ell)=f_0(\ell)$ is the envelope
in~\eqref{eq:app_envelope_laplace} and the Laplace variable
$\ell$ is the envelope lag used throughout the manuscript.
The converse in (i) is included because the power-law
case is used in both directions: the forward implication
derives the envelope from the small-rate density, while the
converse justifies reading an observed power-law envelope as
evidence for the corresponding regularly varying small-rate
density. The log-regularly varying case, defining the log-regular
regime, is used here only in the forward direction, as an ansatz tail.

The proposition's regularity hypothesis (local
integrability of $p$ and eventual monotonicity on a right
neighborhood of $\bar{\mu}=0$) is satisfied in both cases.
For the power-law specialization used in case~(i),
we have $L=1$, so eventual
monotonicity is immediate from the explicit power form
(decreasing for $\beta<1$, constant for $\beta=1$,
increasing for $\beta>1$);
more general regularly varying densities
$c\,\bar{\mu}^{\beta-1}\,L(\bar{\mu})$ require the
eventual-monotonicity hypothesis stated in the proposition.
For case~(ii), a direct differentiation of
$g(\bar{\mu})=\bar{\mu}^{-1}\bigl(\log(1/\bar{\mu})\bigr)^{-(1+\vartheta)}$
gives
\begin{equation}
g'(\bar{\mu})
=
\bar{\mu}^{-2}\bigl(\log(1/\bar{\mu})\bigr)^{-(2+\vartheta)}
\bigl[(1+\vartheta)-\log(1/\bar{\mu})\bigr],
\end{equation}
which is strictly negative whenever
$\log(1/\bar{\mu})>1+\vartheta$, i.e.\ for
$\bar{\mu}<e^{-(1+\vartheta)}$. The density is therefore
strictly decreasing on
$\bigl(0,e^{-(1+\vartheta)}\bigr)$, which suffices for the
eventual-monotonicity hypothesis.

Proposition~\ref{prop:tauberian} is the theorem used below.
The explicit derivations in
Sections~\ref{app:powerlaw_envelope_derivation_full}
and~\ref{app:log_envelope_derivation_full} are included not
as proofs of the Tauberian theorem, but to show directly how
the two canonical tail models used in this paper produce their
envelope asymptotics and prefactors.
Section~\ref{app:powerlaw_envelope_derivation_full} derives the
power-law envelope~\eqref{eq:app_env_powerlaw_decay} from the
regularly varying density~\eqref{eq:app_small_rate_power},
and Section~\ref{app:log_envelope_derivation_full} derives the
logarithmic envelope~\eqref{eq:app_log_envelope_decay} from
the log-regularly varying density~\eqref{eq:app_small_rate_log}.

\subsection{Power-law envelope}
\label{app:powerlaw_envelope_derivation_full}

When the jump amplitude exceeds $\eta_J^*$, the stationary time-scale density has the power-law tail
\begin{equation}
\label{eq:app_power_tail}
p_\infty(\tau)
\sim
c\,\tau^{-1-\beta},
\qquad \beta > 0,
\ \tau\to\infty.
\end{equation}
Under the reciprocal change of variables~\eqref{eq:app_change_tau_u}, the tail~\eqref{eq:app_power_tail} translates to
\begin{equation}
\label{eq:app_small_rate_power}
p_{\bar{\mu}}(\bar{\mu})
=
\bar{\mu}^{-2}\,p_\infty(1/\bar{\mu})
\sim
c\,\bar{\mu}^{\beta-1},
\qquad \bar{\mu}\downarrow 0.
\end{equation}

\paragraph{Dominant contribution.}
For large~$\ell$ the kernel $e^{-\ell/\tau}$ suppresses
contributions from $\tau\ll \ell$, so the integral
in~\eqref{eq:app_envelope_integral} is dominated by
$\tau$ of order $\ell$ and larger.
Let $\tau_{\min}>0$ be any fixed lower cutoff for the
tail-surrogate integral.
Because $e^{-\ell/\tau}\le e^{-\ell/\tau_{\min}}$ for all
$\tau\le\tau_{\min}$, the contribution from
$(0,\tau_{\min})$ is $O(e^{-\ell/\tau_{\min}})$ and can be
neglected at the polynomial scale.
Replacing the full density $p_\infty$ by its tail form on
$[\tau_{\min},\infty)$ defines the exact tail surrogate
\begin{equation}
\label{eq:app_f_tail_surrogate}
f_{0,\mathrm{tail}}(\ell)
=
\int_{\tau_{\min}}^{\infty}
e^{-\ell/\tau}\,c\,\tau^{-1-\beta}\,d\tau.
\end{equation}
The tail equivalence~\eqref{eq:app_power_tail}, the
positivity of the kernel, and the exponentially small
contribution from $(0,\tau_{\min})$ imply by a standard
sandwich argument that
$f_0(\ell)\sim f_{0,\mathrm{tail}}(\ell)$ as
$\ell\to\infty$.
Note that $\tau_{\min}$ drops out of the final
asymptotic~\eqref{eq:app_env_powerlaw_decay} below and should not
be confused with the upper cutoff $\tau_{\max}$ introduced in
Section~\ref{sec:observable_envelopes} for finite-network
truncation.

\paragraph{Change of variables.}
Let $s=\ell/\tau$, so that $\tau=\ell/s$ and
$d\tau=-\ell\, s^{-2}\,ds$.
When $\tau=\tau_{\min}$ the new variable takes the value
$s=\ell/\tau_{\min}$, and as $\tau\to\infty$ we have
$s\to 0$.
The power of~$\tau$ becomes
\begin{equation}
\label{eq:app_tau_power_sub}
\tau^{-1-\beta}
=
\left(\frac{\ell}{s}\right)^{\!-1-\beta}
=
\ell^{-1-\beta}\,s^{1+\beta}.
\end{equation}
Substituting \eqref{eq:app_tau_power_sub} and the
differential into \eqref{eq:app_f_tail_surrogate},
and reversing the limits of integration
(which absorbs the minus sign from $d\tau$):
\begin{align}
f_{0,\mathrm{tail}}(\ell)
&=
\int_{0}^{\ell/\tau_{\min}}
e^{-s}\;
c\;\ell^{-1-\beta}\,s^{1+\beta}\;
\ell\,s^{-2}
\;ds
\nonumber\\[4pt]
&=
c\,\ell^{-1-\beta+1}
\int_{0}^{\ell/\tau_{\min}}
s^{1+\beta-2}\,e^{-s}\,ds
\nonumber\\[4pt]
&=
c\,\ell^{-\beta}
\int_{0}^{\ell/\tau_{\min}} s^{\beta-1}\,e^{-s}\,ds.
\label{eq:app_env_powerlaw_decay_intermediate}
\end{align}

\paragraph{Asymptotic limit.}
As $\ell\to\infty$, the upper limit of integration
$\ell/\tau_{\min}\to\infty$, and the incomplete Gamma
integral in~\eqref{eq:app_env_powerlaw_decay_intermediate}
converges to the complete Gamma function:
\begin{equation}
\label{eq:app_gamma_integral}
\int_0^\infty s^{\beta-1}\,e^{-s}\,ds
=
\Gamma(\beta),
\qquad \beta>0.
\end{equation}

Combining with $f_0(\ell)\sim f_{0,\mathrm{tail}}(\ell)$ gives:
\begin{equation}
\label{eq:app_env_powerlaw_decay}
f_0(\ell)\sim c\,\Gamma(\beta)\,\ell^{-\beta},
\qquad \beta>0, \ \ell\to\infty.
\end{equation}
Since $\Gamma(\beta)$ is a finite positive constant for each
$\beta>0$, it enters as a prefactor that does not affect the
polynomial scaling class.
This direct calculation is the case~(i) specialization of Proposition~\ref{prop:tauberian}; the converse direction in that proposition is the one used when interpreting an observed power-law envelope as evidence for a regularly varying density.

\subsection{Logarithmic envelope}
\label{app:log_envelope_derivation_full}

We restate the log-regularly varying density
tail~\eqref{eq:app_log_regularly_varying_tail} from the
taxonomy:
\begin{equation}
\label{eq:app_marginal_tail_tau}
p_\infty(\tau)
\sim
\frac{1}{\tau(\log\tau)^{1+\vartheta}},
\qquad \tau\to\infty,\ \vartheta>0.
\end{equation}
Using the change of
variables~\eqref{eq:app_change_tau_u},
we obtain as $\bar{\mu}\downarrow 0$:
\begin{equation}
\label{eq:app_small_rate_log}
p_{\bar{\mu}}(\bar{\mu})
=
\bar{\mu}^{-2}\,p_\infty(1/\bar{\mu})
\sim
\frac{1}{\bar{\mu}\,(\log(1/\bar{\mu}))^{1+\vartheta}}.
\end{equation}

\paragraph{Laplace transform asymptotics.}
From \eqref{eq:app_envelope_laplace},
\begin{equation}
\label{eq:app_log_laplace_restate}
f_0(\ell)=\int_{0}^{\infty} e^{-\ell\bar{\mu}}\,p_{\bar{\mu}}(\bar{\mu})\,d\bar{\mu}.
\end{equation}
For large~$\ell$, the exponential kernel suppresses the
region $\bar{\mu}\gg 1/\ell$, so the leading contribution is
confined to the small-rate region where the asymptotic
form~\eqref{eq:app_small_rate_log} applies.
Fix $\bar{\mu}_0\in(0,1)$ and split the integral:
\begin{equation}
\label{eq:app_log_split}
f_0(\ell)
=
\underbrace{
\int_{0}^{\bar{\mu}_0} e^{-\ell\bar{\mu}}\,p_{\bar{\mu}}(\bar{\mu})\,d\bar{\mu}
}_{I_1(\ell)}
\;+\;
\underbrace{
\int_{\bar{\mu}_0}^{\infty} e^{-\ell\bar{\mu}}\,p_{\bar{\mu}}(\bar{\mu})\,d\bar{\mu}
}_{I_2(\ell)}.
\end{equation}

\paragraph{Bounding $I_2(\ell)$.}
On the interval $[\bar{\mu}_0,\infty)$, the exponential factor
satisfies $e^{-\ell\bar{\mu}}\le e^{-\ell\bar{\mu}_0}$ for all
$\bar{\mu}\ge \bar{\mu}_0$.
Therefore
\begin{equation}
\label{eq:app_I2_bound}
I_2(\ell)
\le
e^{-\ell\bar{\mu}_0}
\int_{\bar{\mu}_0}^{\infty} p_{\bar{\mu}}(\bar{\mu})\,d\bar{\mu}
\le
e^{-\ell\bar{\mu}_0},
\end{equation}
since $p_{\bar{\mu}}$ is a probability density
(with total mass~$1$).
The tail integral $I_2(\ell)$ is therefore
$O(e^{-\ell\bar{\mu}_0})$, exponentially small in $\ell$
and negligible compared to the logarithmically decaying
leading term derived below.

\paragraph{Evaluating $I_1(\ell)$.}
Replacing $p_{\bar{\mu}}$ in $I_1$ by its asymptotically
equivalent small-rate form from~\eqref{eq:app_small_rate_log}
and substituting $t=\ell\bar{\mu}$
(so that $d\bar{\mu}=dt/\ell$ and
$\log(1/\bar{\mu})=\log(\ell/t)$):
\begin{equation}
\label{eq:app_marginal_tsub}
I_1(\ell)
\sim
\int_0^{\bar{\mu}_0}
\frac{e^{-\ell\bar{\mu}}}{\bar{\mu}\,(\log(1/\bar{\mu}))^{1+\vartheta}}\,d\bar{\mu}
=
\int_0^{\ell\bar{\mu}_0}
\frac{e^{-t}}{t\,[\log(\ell/t)]^{1+\vartheta}}\,dt.
\end{equation}

\paragraph{Isolating the leading contribution.}
We split the integral~\eqref{eq:app_marginal_tsub} at~$t=1$.
To bound the part with $t\ge1$, split it again at $t=\sqrt{\ell}$. For sufficiently large~$\ell$,
\begin{align}
\int_1^{\sqrt{\ell}}
\frac{e^{-t}}{t[\log(\ell/t)]^{1+\vartheta}}\,dt
&\le
\frac{2^{1+\vartheta}}{(\log\ell)^{1+\vartheta}}
\int_1^\infty\frac{e^{-t}}{t}\,dt
=
O\!\left((\log\ell)^{-1-\vartheta}\right),
\nonumber\\
\int_{\sqrt{\ell}}^{\ell\bar{\mu}_0}
\frac{e^{-t}}{t[\log(\ell/t)]^{1+\vartheta}}\,dt
&\le
\frac{1}{[\log(1/\bar{\mu}_0)]^{1+\vartheta}}
\int_{\sqrt{\ell}}^\infty\frac{e^{-t}}{t}\,dt
=
O\!\left(e^{-\sqrt{\ell}}\right).
\label{eq:app_log_remainder_bounds}
\end{align}
The first bound uses
$\log(\ell/t)\ge\tfrac12\log\ell$ for
$1\le t\le\sqrt{\ell}$; the second uses
$\log(\ell/t)\ge\log(1/\bar{\mu}_0)>0$ for
$t\le\ell\bar{\mu}_0$.
Thus, the entire $t\ge1$ contribution is lower order.

For $t\in(0,1)$, we have $e^{-t}=1+O(t)$, and the $O(t)$ correction contributes $O((\log\ell)^{-1-\vartheta})$.
Therefore,
\begin{align}
\label{eq:app_marginal_mainpiece}
\int_0^1
\frac{e^{-t}}{t[\log(\ell/t)]^{1+\vartheta}}\,dt
&=
\int_0^1
\frac{1}{t[\log(\ell/t)]^{1+\vartheta}}\,dt
+
O\!\left((\log\ell)^{-1-\vartheta}\right)
\nonumber\\
&=
\frac{1}{\vartheta}(\log\ell)^{-\vartheta}
+
O\!\left((\log\ell)^{-1-\vartheta}\right),
\end{align}
where the final equality uses
$s=\log(\ell/t)$.

\paragraph{Asymptotic limit.}
Combining~\eqref{eq:app_log_split}--\eqref{eq:app_marginal_mainpiece}:
the exponentially small $I_2$ is negligible,
the $t\ge 1$ contribution is lower order, and the leading
term is~\eqref{eq:app_marginal_mainpiece}, yielding
\begin{equation}
\label{eq:app_log_envelope_decay}
f_0(\ell)\sim C_\vartheta\,(\log \ell)^{-\vartheta},
\qquad \ell\to\infty,
\end{equation}
where $C_\vartheta=1/\vartheta$ for the canonical tail model~\eqref{eq:app_marginal_tail_tau} with unit amplitude.
More generally, if $p_\infty(\tau)\sim A\,\tau^{-1}(\log\tau)^{-1-\vartheta}$ for some $A>0$, then $C_\vartheta=A/\vartheta$.
Contributions from the bulk of $p_\infty(\tau)$ are lower order and do not alter this leading prefactor.
This direct calculation agrees with case~(ii) of Proposition~\ref{prop:tauberian}, which gives the same logarithmic asymptotic from the small-rate density class~\eqref{eq:app_small_rate_log}.

\section{Tempering-controlled stationary tails below the jump-amplitude threshold}
\label{app:subthreshold_tail}

Above the jump-amplitude threshold, the far-left balance selects the
exponent $\beta$ of the stationary time-scale density
in~\eqref{eq:stationary_timescale_power_law}. In the strictly
sub-threshold range $0<\eta_J<\eta_J^*$, the characteristic equation
has no positive root in $(0,\lambda]$, since $\Phi(\lambda;\omega)<0$
(Proposition~\ref{lem:Phi_properties}). A sufficiently large negative
jump can nevertheless carry a neuron's log-rate $\zeta$ from the
populated bulk to a very long time scale $\tau=e^{-\zeta}$.

Here, the stationary-tail argument gives the probability that the log-rate lies below $-\log T$. Since $\tau=e^{-\zeta}$, this probability equals the stationary fraction $S(T)$ of neurons whose time scales exceed $T$. Thus,
\begin{equation}
\label{eq:subthreshold_survival_definition}
S(T)=\int_T^\infty p_\infty(\tau)\,d\tau
    =\int_{-\infty}^{-\log T}\rho_\infty(\zeta)\,d\zeta .
\end{equation}
Above the jump-amplitude threshold $\eta_J^*$, integrating \eqref{eq:stationary_timescale_power_law} gives $S(T)\sim(c/\beta)\,T^{-\beta}$. Below the threshold, we first determine how the same fraction decays. We then use the stationary equation to recover the time-scale density and the survival law to obtain the unweighted envelope.

We assume that the drift--diffusion--jump process~\eqref{eq:zeta_SDE_mixed} admits an invariant density $\rho_\infty$ and satisfies the positive saturating drift closure~\eqref{eq:drift_saturation}.
Write $X=-\zeta=\log\tau$ and denote its stationary law by $\pi$.
In this coordinate, the restoring drift tends to $-\kappa$ as $X\to\infty$, and $S(T)=\mathbb P_\pi(X>\log T)$ is the stationary time-scale survival function~\citep{durrett2019probability}.

\paragraph{Rate of large negative jumps.}
A negative jump of magnitude $y$ in $\zeta$ increases $X$ by $y$.
For $x>0$, let $\bar\nu_-(x)$ denote the rate per unit of coarse-grained training time of such jumps with magnitude at least $x$.
Writing $y=x+s$ in the L\'evy density~\eqref{eq:tempered_levy_measure} gives
\begin{equation}
\label{eq:subthreshold_jump_rate_factorization}
\begin{aligned}
\bar\nu_-(x)
&=\eta_J c_{\alpha_{\mathrm{jump}}}
  \int_x^\infty e^{-\lambda y}y^{-1-\alpha_{\mathrm{jump}}}\,dy\\
&=\eta_J c_{\alpha_{\mathrm{jump}}}
  e^{-\lambda x}x^{-1-\alpha_{\mathrm{jump}}}
  \int_0^\infty e^{-\lambda s}
       (1+s/x)^{-1-\alpha_{\mathrm{jump}}}\,ds .
\end{aligned}
\end{equation}
At each fixed $s\ge0$, the last integrand tends to $e^{-\lambda s}$
and is bounded by that same integrable function. By dominated
convergence~\citep{durrett2019probability}, the integral tends to
$1/\lambda$. Hence
\begin{equation}
\label{eq:subthreshold_negative_jump_tail}
\bar\nu_-(x)\sim
\frac{\eta_J c_{\alpha_{\mathrm{jump}}}}{\lambda}
e^{-\lambda x}x^{-1-\alpha_{\mathrm{jump}}},
\qquad x\to\infty .
\end{equation}

From a fixed bulk state $u$, the jump must cover the distance $x-u$
to cross level $x$. Equation~\eqref{eq:subthreshold_negative_jump_tail}
gives $\bar\nu_-(x-u)/\bar\nu_-(x)\to e^{\lambda u}$.
The proof below shows that, after accounting for the motion between
jumps, the stationary probability beyond $x$ is asymptotically
proportional to $\bar\nu_-(x)$. At $x=\log T$, the exponential factor
therefore gives a power of $T$, and the algebraic factor gives a power
of $\log T$.

\begin{proposition}[Stationary time-scale tail below the threshold]
\label{prop:subthreshold_survival}
Under the assumptions above, if $0<\eta_J<\eta_J^*$, there is a
constant $0<C<\infty$ such that
\begin{equation}
\label{eq:subthreshold_sharp_survival}
S(T)\sim C\,T^{-\lambda}(\log T)^{-1-\alpha_{\mathrm{jump}}},
\qquad T\to\infty .
\end{equation}
In particular, $S$ is regularly varying with index $-\lambda$.
\end{proposition}

To compare the two anti-collapsed branches at the level of the density, we also need
to know how the probability beyond $T$ is distributed over nearby
time scales. The stationary equation supplies this information.

\begin{corollary}[Stationary time-scale density below the threshold]
\label{cor:subthreshold_density}
Under the conditions of Proposition~\ref{prop:subthreshold_survival},
with the Gaussian intensity $\eta_G>0$ of~\eqref{eq:zeta_SDE_mixed},
the stationary time-scale density has a continuous version in the
right tail satisfying
\begin{equation}
\label{eq:subthreshold_sharp_density}
p_\infty(T)\sim \frac{\lambda S(T)}{T}
\sim \lambda C\,T^{-1-\lambda}(\log T)^{-1-\alpha_{\mathrm{jump}}},
\qquad T\to\infty .
\end{equation}
In particular, $p_\infty$ is regularly varying with index $-1-\lambda$.
\end{corollary}

We prove the survival and density laws at the end of this appendix.
The envelope follows directly from the survival law.

\paragraph{From long time scales to the envelope.}
A neuron with time scale $\tau$ contributes the weight $e^{-\ell/\tau}$ to the unweighted envelope at lag $\ell$. Neurons with time scales comparable to or longer than the lag therefore retain appreciable weight.
To express this directly in terms of $S(T)$ in Eq.~\eqref{eq:subthreshold_survival_definition}, write $e^{-\ell/\tau}=\ell\int_{1/\tau}^\infty e^{-\ell s}\,ds$.
Averaging over the stationary time scales and exchanging the two nonnegative integrals gives
\begin{equation}
\label{eq:subthreshold_envelope_tail_integral}
f_0(\ell)=\ell\int_0^\infty e^{-\ell s}S(1/s)\,ds .
\end{equation}
For $s>0$, $S(1/s)$ is the stationary fraction of neurons whose rates are below $s$.
This fraction is bounded and nondecreasing in $s$.
By~\eqref{eq:subthreshold_sharp_survival} with $T=1/s$, it is regularly varying with index $\lambda$ as $s\downarrow0$.
Part~(i) of Karamata's Tauberian theorem (Proposition~\ref{prop:tauberian}) therefore applies with exponent $1+\lambda$ and gives the following corollary.

\begin{corollary}[Unweighted envelope below the threshold]
\label{cor:subthreshold_envelope}
Under the conditions of Proposition~\ref{prop:subthreshold_survival},
\begin{equation}
\label{eq:subthreshold_sharp_envelope}
f_0(\ell)\sim\Gamma(1+\lambda)\,S(\ell)
\sim C\,\Gamma(1+\lambda)\,
     \ell^{-\lambda}(\log\ell)^{-1-\alpha_{\mathrm{jump}}},
\qquad \ell\to\infty .
\end{equation}
\end{corollary}

By the sandwich bound~\eqref{eq:sandwich}, the weighted envelope
$f(\ell)$ lies between positive multiples of $f_0(\ell)$ and
therefore has the same decay at large lag, up to bounded factors.

Below the threshold $\eta_J^*$, the tempering rate $\lambda$ plays the role of
the spectral exponent $\beta$ above the threshold. The time-scale density
has index $-1-\lambda$, while the fraction of long time scales and the
unweighted envelope both have index $-\lambda$. All three laws carry
the same logarithmic correction, with exponent $-1-\alpha_{\mathrm{jump}}$.

\paragraph{Proof of Proposition~\ref{prop:subthreshold_survival}.}

\begin{proof}[Proof]
We apply the stationary-tail theorem of \citet[Section~5.3, Theorem~4]{borovkov2001large} by observing $X=\log\tau$ just before each positive jump larger than one.
These jumps form a compound Poisson process of rate $q=\bar\nu_-(1)$, independent of the remaining driving noise~\citep{schilling2016introduction}.
A selected jump has size $Y$ with tail 
\begin{equation}
\label{eq:app_subthreshold_jumpsize_tail}
\bar Q(x)=\mathbb P(Y>x)=\bar\nu_-(x)/q
\end{equation}
for $x\ge1$.
Write $D_z$ for the displacement from $z$ under the remaining dynamics over the time to the next selected jump.
This time is exponential with rate $q$ and independent of the remaining noise. A step from $u$ thus has increment $\xi(u)=Y+D_{u+Y}$.
The theorem requires that, far to the right, this increment converge in law to a limit with negative mean; that one comparison law in the class $\mathcal S(\lambda)$, with negative mean and exponential moment below one at $\lambda$, dominate a step from every state; and that the tail of a step be asymptotically proportional to the comparison tail, uniformly on compact sets of starting states.

The observed chain preserves $\pi$. Indeed, start $X$ in stationarity and let $P$ denote the transition operator of the chain.
Since the selected jumps have constant intensity $q$, the expected sum of a bounded test function $\varphi$ over their pre-jump states in $(0,t]$ is $qt\int\varphi\,d\pi$.
Replacing each state in the sum by the next pre-jump state gives expectation $qt\int P\varphi\,d\pi$.
Only the two end terms change. Dividing by $qt$ and letting $t\to\infty$ gives $\pi P=\pi$.

\emph{Limiting increment.}
Let $D$ be the residual displacement over the same exponential time with drift held at $-\kappa$.
Then $D_z\Rightarrow D$ as $z\to\infty$.
To see this, drive the two processes with the same noise. Up to time $t$
or the first exit from $[z/2,\infty)$, their separation is at most
$t\sup_{v\ge z/2}|F(-v)-\kappa|$. For fixed $t$ this tends to zero,
and the probability of exit tends to zero because the constant-drift
displacement is bounded on each finite time interval almost surely.
Averaging over the time to the next jump gives the claimed convergence.
Conditioning on $Y$ now gives $\xi(u)\Rightarrow Y+D$ as $u\to\infty$,
with independent limiting summands.

Write $A=\mathbb{E}[e^{\lambda Y}]$ and $B=\mathbb{E}[e^{\lambda D}]$.
To compute $B$, consider the motion defining $D$: the drift of $X$ is fixed at $-\kappa$, and the Gaussian noise and all jumps other than positive jumps larger than one are retained.
Denote the cumulant generating function per unit time of this motion by $\Phi_0$.
Restoring the positive jumps larger than one gives $\Phi(\cdot\,;\omega)$ from~\eqref{eq:appendix_characteristic_function}, the function that appears in the characteristic equation.
In $A$ the exponential weight cancels the tempering of the jump density, leaving an integrable power on $y>1$, so $1<A<\infty$.
To see how the selected jumps contribute to $\Phi(\lambda;\omega)$, condition on their number in a time interval of length $t$.
Given $n$ such jumps, their independent sizes give an expected exponential weight of $A^n$.
Averaging over the Poisson count with mean $qt$ gives $e^{qt(A-1)}$.
Taking the logarithm and dividing by $t$, we obtain their contribution $q(A-1)$ to the cumulant generating function per unit time.
Removing these jumps therefore gives $\Phi_0(\lambda)=\Phi(\lambda;\omega)-q(A-1)<0$.
Conditional on the time to the next selected jump being $t$, the expected exponential weight of the residual displacement is $e^{t\Phi_0(\lambda)}$. Averaging over the waiting-time density $qe^{-qt}$ gives
\begin{equation}
\label{eq:subthreshold_chain_constants}
\begin{aligned}
B
&=\int_0^\infty qe^{-qt}e^{t\Phi_0(\lambda)}\,dt
  =q\int_0^\infty e^{-[q-\Phi_0(\lambda)]t}\,dt
  =\frac{q}{q-\Phi_0(\lambda)} \\
&=\frac{q}{qA-\Phi(\lambda;\omega)}
  =\frac{q}{qA+|\Phi(\lambda;\omega)|}.
\end{aligned}
\end{equation}
To express the denominator in terms of $\Phi$, substitute the relation for $\Phi_0$ derived above.
This gives $q-\Phi_0(\lambda) =q-[\Phi(\lambda;\omega)-q(A-1)] =qA-\Phi(\lambda;\omega)$.
Since $\Phi(\lambda;\omega)<0$ below the threshold, the denominator is $qA+|\Phi(\lambda;\omega)|$.
Since $Y$ and $D$ are independent, it follows that $\mathbb{E}[e^{\lambda(Y+D)}]=AB =qA/(qA+|\Phi(\lambda;\omega)|)<1$.
The limiting increment has finite first absolute moment, since the tempered noise has finite variance and the exponential time has a finite second moment.
The limiting increment $Y+D$ has negative mean.
Indeed, Jensen's inequality gives $\lambda\,\mathbb{E}[Y+D]\le\log\mathbb{E}[e^{\lambda(Y+D)}] =\log(AB)<0$.
An individual increment can still be positive. Since $Y$ is independent of $D$ and unbounded above, it has positive probability of exceeding $-D$, so $\mathbb{P}(Y+D>0)>0$.

\emph{A bound valid from every starting state.}
Let $X^0$ denote the motion of $X$ with the positive jumps larger than one removed.
Its generator $\mathcal L_0$ is obtained from \eqref{eq:generator_definition} by writing the dynamics in the coordinate $X=-\zeta$ and omitting those jumps.
The drift of $X^0$ at $x$ is therefore $-F(-x)$.

Since the remaining positive jumps are bounded by one, $\Phi_0(r)$ is finite for every $r>0$.
By continuity and $\Phi_0(\lambda)<0$, we can choose $r>\lambda$ such that $\Phi_0(r)<0$.
For the test function $\varphi(x)=e^{rx}$, the generator with drift fixed at $-\kappa$ gives $\Phi_0(r)\varphi(x)$.
Restoring the actual drift adds $(\kappa-F(-x))\varphi'(x)$, so
\begin{equation}
(\mathcal L_0\varphi)(x) = \bigl[\Phi_0(r)+r(\kappa-F(-x))\bigr]e^{rx}.
\end{equation}
By drift saturation, choose $R>0$ such that, for $x\ge R$, $F(-x)$ is bounded and this coefficient is at most $\Phi_0(r)/2$.
Take a smooth nondecreasing function $w$ equal to $1$ for $x\le R$ and to $e^{rx}$ for $x\ge R+1$.
Far to the right, this modification changes the generator by a bounded amount, since only jumps of magnitude greater than one can reach the modified region.
For $x\le R-1$ the generator vanishes, and on the intervening interval the drift is bounded wherever $w'$ is nonzero.
Thus, $\mathcal L_0w\le-aw+b$ for some $a,b>0$.
Dynkin's formula~\citep{schilling2016introduction}, localized to bounded intervals, gives $\mathbb{E}_z[w(X_t^0)]\le e^{-at}w(z)+b/a$ for the residual process $X^0$.
Averaging over the exponential time and multiplying by $e^{-rz}$ gives a constant $c_1$ such that
\begin{equation}
\label{eq:subthreshold_residual_moment}
\mathbb{E}[e^{rD_z}]\le c_1(1+e^{-rz}),\qquad z\in\mathbb R .
\end{equation}

To bound a full step uniformly in its starting state, fix $V>0$ and measure the residual terminal position from $\max(z,V)$.
Define $Z_V$ by its upper tail
\begin{equation}
\label{eq:subthreshold_majorant}
\mathbb P(Z_V>t) = \sup_z\mathbb P\bigl(D_z>t+[V-z]^+\bigr), \qquad [v]^{+} = \max(v,0).
\end{equation}
The supremum is nonincreasing and right-continuous, and tends to one at $-\infty$.
By~\eqref{eq:subthreshold_residual_moment}, each term is bounded by $c_1e^{-rt}(e^{-r[V-z]^+}+e^{-r\max(z,V)})\le 2c_1e^{-rt}$, so the supremum defines a probability tail.
For $z\le V/2$ this bound tends to zero at fixed $t$ as $V\to\infty$.
For $z>V/2$, convergence of $D_z$ bounds the limit superior by $\mathbb P(D>t)$ at continuity points.
Taking $z\to\infty$ at fixed $V$ gives the matching lower bound.
Hence $Z_V\Rightarrow D$, and the uniform tail bound with $r>\lambda$ gives $K_V = \mathbb{E}[e^{\lambda Z_V}]\to B$.

Choose $V$ so that $AK_V<1$ and set $\eta=Y+Z_V$ with independent summands.
Given $Y$, the terminal position from $u+Y$ is stochastically bounded above by $\max(u+Y,V)+Z_V$.
Since $\max(u+Y,V)\le\max(u,V)+Y$, the next state is stochastically bounded above by $\max(u,V)+\eta$ for every $u$.
The law of $\eta$ has a density, so it is nonlattice: it cannot be supported on an evenly spaced grid $a+h\mathbb{Z}$ with $h>0$.
Its negative part is integrable because $Z_V$ stochastically dominates $D$, and its positive part is integrable because $\mathbb{E}[e^{\lambda\eta}]=AK_V<1$.
Jensen's inequality therefore gives $\mathbb{E}[\eta]<0$.

\emph{Transition tails and application of the theorem.}
Write $p_Y$ for the density of $Y$. Conditioning on the jump size and setting $s=y-x$ gives
\begin{equation}
\label{eq:subthreshold_transition_tail}
\frac{\mathbb P(\xi(u)\ge x)}{\bar Q(x)}
=\int_{1-x}^{\infty}
 \frac{p_Y(x+s)}{\bar Q(x)}
 \mathbb P(D_{u+x+s}\ge-s)\,ds .
\end{equation}
Here $\bar Q(x)$ is the probability that a selected jump exceeds $x$, as defined in~\eqref{eq:app_subthreshold_jumpsize_tail}.
By~\eqref{eq:subthreshold_negative_jump_tail}, the density ratio is $(\lambda+o(1))e^{-\lambda s}(x/(x+s))^{1+\alpha_{\mathrm{jump}}}$, where the error depends only on $x$.
For $s<0$ in the integration range, $x/(x+s)\le1+|s|$.
If $u$ belongs to a fixed compact set, the landing position $u+x+s$ is bounded below, so \eqref{eq:subthreshold_residual_moment} bounds the residual probability by a constant times $e^{rs}$.
The integrand is therefore bounded on $s<0$ by a multiple of $(1+|s|)^{1+\alpha_{\mathrm{jump}}}e^{-(r-\lambda)|s|}$, and on $s\ge0$ by a multiple of $e^{-\lambda s}$.

Both bounds are integrable. For any choice of starting states $u=u_x$ in the compact set, $u_x+x+s\to\infty$ at fixed $s$.
Dominated convergence therefore gives the same limit for every such choice,
\begin{equation}
\label{eq:app_subthreshold_Blimit}
\int_{\mathbb R}\lambda e^{-\lambda s}\mathbb P(D\ge-s)\,ds =\mathbb{E}\!\left[\int_{-D}^{\infty}\lambda e^{-\lambda s}\,ds\right] =B.
\end{equation}
Let $G$ denote the law of $\eta=Y+Z_V$, which serves as the comparison law, and write $\bar G(x)=\mathbb P(\eta\ge x)$ for its upper tail.
The same calculation, with $Z_V$ in place of $D_{u+x+s}$, applies to the ratio $\bar G(x)/\bar Q(x)$.
Its limit is therefore $\mathbb{E}[e^{\lambda Z_V}]=K_V$, so $\bar G(x)\sim K_V\bar Q(x)$.
Combining $\bar G(x)\sim K_V\bar Q(x)$ with the limit $B$ obtained in~\eqref{eq:app_subthreshold_Blimit} for the ratio $\mathbb P(\xi(u)\ge x)/\bar Q(x)$ gives $\mathbb P(\xi(u)\ge x)/\bar G(x)\to B/K_V$ as $x\to\infty$, uniformly for starting states $u$ in any fixed compact set.
Moreover, $e^{\lambda x}\bar G(x)$ is regularly varying with index $-1-\alpha_{\mathrm{jump}}<-1$.
The comparison law consequently belongs to the class $\mathcal S(\lambda)$ required by the theorem \citep[Section~5]{borovkov2001large}.
Its exponential moments above $\lambda$ are infinite, as are those of $Y$.

Lemma~11 of \citet{borovkov2001large} now gives $\mathbb P_\pi(X\ge x)=O(\bar G(x))$.
Let $M_\lambda=\mathbb{E}_\pi[e^{\lambda X}]$ denote the stationary exponential moment of $X$ at $\lambda$.
To bound this moment, note that $e^{\lambda X}\le 1$ when $X\le 0$.
When $X>0$, we have $e^{\lambda X}=1+\int_0^X\lambda e^{\lambda x}\,dx$.
Taking expectations and exchanging the order of integration gives the following upper bound on $M_\lambda$:
\begin{equation}
M_\lambda
\le 1+\lambda\int_0^\infty
e^{\lambda x}\mathbb P_\pi(X>x)\,dx .
\end{equation}
By the bound from \citet[Lemma~11]{borovkov2001large}, the integrand is $O(e^{\lambda x}\bar G(x))=O(x^{-1-\alpha_{\mathrm{jump}}})$ as $x\to\infty$.
Since $\alpha_{\mathrm{jump}}>0$, the integral converges and $M_\lambda$ is finite.
\citet[Theorem~4]{borovkov2001large} then gives
\begin{equation}
\label{eq:subthreshold_sharp_tail}
\mathbb P_\pi(X\ge x)
\sim \frac{(B/K_V)M_\lambda}{1-AB}\,\bar G(x)
\sim \frac{B M_\lambda}{1-AB}\,\bar Q(x)
= \frac{M_\lambda}{|\Phi(\lambda;\omega)|}\,\bar\nu_-(x),
\end{equation}
where the last equality uses~\eqref{eq:subthreshold_chain_constants} and $\bar Q(x)=\bar\nu_-(x)/q$.
Finally, since $X=\log\tau$, the event $\tau>T$ is the same as $X>\log T$.
The stationary law $\pi$ has a density, so it assigns zero probability to $X=\log T$.
We may therefore use the tail estimate~\eqref{eq:subthreshold_sharp_tail} at $x=\log T$.
Substituting the jump-tail asymptotic~\eqref{eq:subthreshold_negative_jump_tail} and using $e^{-\lambda\log T}=T^{-\lambda}$ gives
\begin{equation}
\begin{aligned}
S(T)
&=\mathbb P_\pi(X\ge\log T) \\
&\sim \frac{M_\lambda}{|\Phi(\lambda;\omega)|}\,\bar\nu_-(\log T) \\
&\sim \underbrace{\frac{\eta_J c_{\alpha_{\mathrm{jump}}}M_\lambda}
   {\lambda|\Phi(\lambda;\omega)|}}_{C}\,
   T^{-\lambda}(\log T)^{-1-\alpha_{\mathrm{jump}}},
   \qquad T\to\infty .
\end{aligned}
\end{equation}
This proves the stationary time-scale survival law~\eqref{eq:subthreshold_sharp_survival}.
\end{proof}

\paragraph{Proof of Corollary~\ref{cor:subthreshold_density}.}

\begin{proof}[Proof]
Write $g(x)=\rho_\infty(-x)$ for the density of $X=-\zeta=\log\tau$ and $H(x)=\mathbb P_\pi(X>x)=S(e^x)$ for its upper tail.
For each fixed $h>0$, the interval $(x,x+h]$ has probability $H(x)-H(x+h)$.
The survival law~\eqref{eq:subthreshold_sharp_survival}, with $T=e^x$, therefore determines the large-$x$ behavior of the average density $[H(x)-H(x+h)]/h$ on this interval.
To pass from this average to the pointwise density $g(x)$, we use the stationary Fokker--Planck equation~\eqref{eq:appendix_stationary_forward_equation}, written in the coordinate $X=-\zeta$, to bound how much $g$ varies across the interval.

Write $\nu=\nu_{\alpha_{\mathrm{jump}},\lambda}$ for the symmetric jump measure in~\eqref{eq:tempered_levy_measure}.
Pairing jumps of sizes $y$ and $-y$ in~\eqref{eq:jump_operator_explicit} cancels their linear compensation terms.
For a smooth compactly supported test function $\varphi$, Taylor's formula with integral remainder gives
\begin{equation}
\varphi(x+y)+\varphi(x-y)-2\varphi(x)
=\int_{-y}^{y}(y-|v|)\,\varphi''(x-v)\,dv ,
\qquad y>0 .
\end{equation}
The weight $y-|v|$ integrates to $y^2$ on $[-y,y]$, and $\int_0^\infty y^2\,\nu(dy)<\infty$ by~\eqref{eq:tempered_levy_measure}, since $\alpha_{\mathrm{jump}}<2$ and $\lambda>0$. As $\varphi''$ is bounded, we may therefore exchange the integrals over $y$ and $v$ by Fubini's theorem~\citep{durrett2019probability}.
This gives $\mathcal J_{\alpha_{\mathrm{jump}},\lambda}[\varphi]=K*\varphi''$, where $*$ denotes convolution and
\begin{equation}
K(v)=\int_{|v|}^\infty (y-|v|)\,\nu(dy)
    =\int_{|v|}^\infty \bar\nu_-(u)\,du .
\end{equation}
The second equality uses $\bar\nu_-(u)=\int_u^\infty\nu(dy)$ for $u>0$, by symmetry.
The kernel is even and nonnegative, with $\int_{\mathbb R}K(v)\,dv=\int_0^\infty y^2\,\nu(dy)$.

Because $K$ is even, the convolution can be transferred from $\varphi''$ to $g$ when testing the stationary equation.
With $m=\eta_G g+K*g$ and drift $-F(-x)$ in the $X$ coordinate, this gives $m''=-(F(-\cdot)g)'$ in the sense of distributions.
Integrating twice toward $+\infty$ gives, almost everywhere,
\begin{equation}
\label{eq:subthreshold_density_balance}
\eta_G g(x)+(K*g)(x)=\int_x^\infty F(-u)g(u)\,du .
\end{equation}
Indeed, the integral is $O(H(x))$ because the drift is bounded in the right tail.
It therefore tends to zero, and any nonzero affine term from the two integrations would contradict $m\in L^1(\mathbb R)$.
Since both terms on the left are nonnegative and $\eta_G>0$, each is bounded by the right-hand side.
Hence, $g(x)=O(H(x))$ and $(K*g)(x)=O(H(x))$ almost everywhere in the right tail.

Fix $0<h<1/2$ and compare~\eqref{eq:subthreshold_density_balance} at $x$ and $x+h$. The drift integral changes by $O(hH(x))$ because the drift is bounded and $g(u)=O(H(u))$ on $[x,x+h]$.

For the convolution, write $v=x-u$ and split the integral at $|v|=1$. When $|v|\le1$, we have $g(u)=O(H(x-1))=O(H(x))$, since $H(x-1)/H(x)\to e^\lambda$ by~\eqref{eq:subthreshold_sharp_survival}.
The change in this part of the integral is therefore bounded by a constant times $H(x)\|K(\cdot+h)-K\|_1$.
When $|v|>1$, the interval from $v$ to $v+h$ stays at least $1/2$ away from the origin.
On this set the logarithmic derivative of $K$ is bounded, since $|K'(v)|/K(v)=\bar\nu_-(|v|)/K(v)\to\lambda$ as $|v|\to\infty$ by~\eqref{eq:subthreshold_negative_jump_tail}.
Integrating this bound from $v$ to $v+h$ gives $|K(v+h)-K(v)|\le c h K(v)$ for a constant $c$ independent of $v$ and $h$.
The change in the outer part of the convolution is thus $O(h(K*g)(x))=O(hH(x))$.

Since translations are continuous in $L^1$, these estimates give continuity of $K*g$ in the right tail. The drift integral is also continuous, so~\eqref{eq:subthreshold_density_balance} defines a continuous version of $g$ there.
Combining the bounds above, we obtain, for some $x_0$ and $c_0>0$,
\begin{equation}
\sup_{x\ge x_0}\frac{|g(x+h)-g(x)|}{H(x)}
\le c_0\bigl(h+\|K(\cdot+h)-K\|_1\bigr)
\longrightarrow 0,
\qquad h\downarrow0 .
\end{equation}

We can now recover the pointwise density from its interval averages.
For fixed $h>0$, the survival law~\eqref{eq:subthreshold_sharp_survival} gives
\begin{equation}
\frac{1}{hH(x)}\int_0^h g(x+v)\,dv
=\frac{H(x)-H(x+h)}{hH(x)}
\longrightarrow \frac{1-e^{-\lambda h}}{h},
\qquad x\to\infty .
\end{equation}
By the uniform estimate above, this normalized average differs from $g(x)/H(x)$ by an amount that tends to zero as $h\downarrow0$, uniformly for $x\ge x_0$. 
Taking $x\to\infty$ for fixed $h$ gives the limit $(1-e^{-\lambda h})/h$ for the normalized interval average.
As $h\downarrow0$, this quantity tends to $\lambda$, and the uniform estimate above makes the difference between the average and $g(x)/H(x)$ arbitrarily small for all sufficiently large $x$.
Hence, $g(x)/H(x)\to\lambda$ as $x\to\infty$.

To return to the time scale $T$, set $x=\log T$. Since $dx=dT/T$, the density transforms as $p_\infty(T)=g(\log T)/T$, while the upper tail satisfies $H(\log T)=S(T)$.
Thus
\begin{equation}
p_\infty(T)
=\frac{g(\log T)}{H(\log T)}\,\frac{S(T)}{T}
\sim \frac{\lambda S(T)}{T}
\sim \lambda C\,T^{-1-\lambda}(\log T)^{-1-\alpha_{\mathrm{jump}}},
\qquad T\to\infty .
\end{equation}
The last equivalence uses the survival law~\eqref{eq:subthreshold_sharp_survival}
and completes the proof of~\eqref{eq:subthreshold_sharp_density}.
\end{proof}

\clearpage
\section{Code availability}
\label{app:code}

Code and instructions for reproducing the experiments are available at
\begin{center}
\url{https://github.com/lorenzolivi/anticollapse}
\end{center}

\end{document}